\documentclass{article}
\usepackage[main, final]{neurips_2026}
\setcitestyle{numbers,square}

\usepackage[utf8]{inputenc} % allow utf-8 input
\usepackage[T1]{fontenc}    % use 8-bit T1 fonts
\usepackage[colorlinks=true,
linkcolor=blue,
citecolor=blue,
urlcolor=blue]{hyperref}     % hyperlinks
\usepackage{url}            % simple URL typesetting
\usepackage{booktabs}       % professional-quality tables
\usepackage{amsfonts}       % blackboard math symbols
\usepackage{nicefrac}       % compact symbols for 1/2, etc.
\usepackage{microtype}      % microtypography
\usepackage{xcolor}         % colors
\usepackage{graphicx}
\usepackage{amsmath}
\usepackage{amssymb}
\usepackage{array}
\usepackage{mathtools}
\usepackage{amsthm}
\usepackage{mathrsfs}
\usepackage{bm}
\usepackage{nccmath}   
\usepackage{float}
\usepackage{algorithm}
\usepackage{algorithmic}
\usepackage{caption}
\usepackage{subfigure}
\usepackage{subcaption} 
\usepackage{dsfont}
\usepackage{enumitem}
\usepackage{comment}
\usepackage{todonotes}

\newcommand{\R}{\mathbb{R}}

\newcommand{\E}{\mathbb{E}}

\DeclareMathOperator*{\argmin}{arg\,min}

\newcommand{\edg}[1]{\ensuremath{\! \left[ #1 \right] }}

\newtheorem{Theorem}{Theorem}[section]
\newtheorem{Lemma}[Theorem]{Lemma}

\newtheorem{Proposition}[Theorem]{Proposition}

\theoremstyle{definition}

\theoremstyle{remark}

\title{%Recursive Physics-Informed Neural Networks for PIDEs
INEUS: Iterative Neural Solver\\ for High-Dimensional PIDEs}
            
\author{
  \textbf{Jean-Loup Dupret}$^{1}$
  \qquad
  \textbf{Davide Gallon}$^{1,2}$
  \qquad
  \textbf{Patrick Cheridito}$^{1}$
  \\[1em]
  $^{1}$Department of Mathematics and RiskLab,
  ETH Zurich, Switzerland
  \\[0.3em]
  $^{2}$Applied Mathematics: Institute for Analysis and Numerics, 
  University of M\"unster, Germany
  \\[0.9em]
  \texttt{\{jdupret,patrickc\}@ethz.ch}, \
  \texttt{\{davide.gallon\}@uni-muenster.de}
}

\begin{document}
    \maketitle	

		\begin{abstract}
        In this paper, we introduce INEUS, a meshfree iterative neural solver for partial integro-differential equations (PIDEs). The method replaces the explicit evaluation of nonlocal jump integrals with single-jump sampling and reformulates PIDE solving as a sequence of recursive regression problems. Like Physics-Informed Neural Networks (PINNs), INEUS learns global solutions over the entire space-time domain, yet it offers a more efficient treatment of nonlocal terms and avoids the computationally expensive differentiation of full PIDE residuals. These features make INEUS particularly well suited for high-dimensional PDEs and PIDEs. Supported by a contraction-based convergence proof for linear PIDEs, our numerical experiments show that INEUS delivers accurate and scalable solutions for various high-dimensional linear and nonlinear examples.

		\end{abstract}
		
		%\begin{keyword}
		%	{\small } 
			
			
		%\end{keyword}
		
	%\end{frontmatter}
	
	\section{Introduction}
	\hypersetup{
		linkcolor=cyan,
		urlcolor=red,
		citecolor=black,
	}
Partial differential equations (PDEs) and their extensions, partial integro-differential equations (PIDEs), are important modeling tools in the natural sciences, engineering, economics and finance~\cite{AbergelTachet2010,GoswamiPatelShevgaonkar2016,KavallarisSuzuki2018}.  While traditional numerical methods such as finite differences~\cite{ContVoltchkova2005,KwonLee2011} and finite elements~\cite{solin2005partial,quarteroni1994numerical} are effective for low-dimensional problems, they quickly become computationally intractable as the dimension increases. This limitation, widely known as the curse of dimensionality, arises because mesh-based approaches require spatial discretization, causing computational costs to scale exponentially with the state dimension.  

\paragraph{Deep learning for PDEs.} In recent years, neural networks have emerged as effective tools for mitigating the curse of dimensionality in the numerical solution of high-dimensional PDEs~\cite{blechschmidt2021three, CaiEtAl2021}. One prominent  approach is based on  Physics-Informed Neural Networks (PINNs)~\cite{raissi2017physics,raissi2019physics,sirignano2018dgm,berg2018unified, karniadakis2021physics, wang2022and,mowlavi2023optimal, wu2023comprehensive, dupret2026deep}, which minimize PDE residuals using gradient-based optimization and automatic differentiation. A key strength of PINNs is their flexibility: they can be applied to a wide range of challenging PDEs and provide \textit{global} approximations  over the entire space-time domain, making them particularly suitable for real-time applications. A second strand of the 
literature leverages the Feynman--Kac representation or backward stochastic differential equations (BSDEs, also known as nonlinear Feynman--Kac representations) to approximate parabolic PDEs on a finite time grid using neural networks~\cite{han2017deep,han2018solving, chan2019machine, beck2021solving, 
beck2021deep, Raissi2024}. Because they rely on time-discretized stochastic dynamics, these methods introduce approximation errors that may accumulate over time steps and iterations, while their accuracy depends strongly  on how well the sample paths cover the space-time domain.  %These methodstypically produce \textit{local} approximations, in the sense that the solution is learned along sampled trajectories or at pre-specified space-time points.
Beyond these two paradigms, alternative deep learning frameworks include  the deep mixed residual method (MIM)~\cite{LyuZhangChenChen2022}, weak adversarial networks (WAN)~\cite{ZangBaoYeZhou2020}, the deep Ritz method~\cite{EYu2018}, and  the deep Nitsche method~\cite{LiaoMing2021}.

\paragraph{Deep learning for PIDEs.} Despite these advances, the deep learning literature on PIDEs remains comparatively sparse. Existing PINN-based extensions ~\cite{mishra2021physics,saleh2023learning,kartik2025physics, bansal2025application} struggle with nonlocal terms because they require the computationally expensive numerical approximation of the integral term at every sampled space-time point.  On the other hand, Feynman--Kac and BSDE-based approaches are able to handle nonlocal terms efficiently by incorporating jumps into the dynamics of the underlying auxiliary process~\cite{frey2022deep, andersson2025deep, gnoatto2025convergence, neufeld2024full, frey2025convergence, wang2023deep, lu2024temporal, ye2024fbsjnn}. 
However, in contrast to PINNs which are trained directly on the original PIDE, Feynman--Kac and BSDE-based methods instead solve a time-discretized approximation of the PIDE and are trained only on data points along a given time grid. Even when arbitrary restarts are used to improve domain coverage, BSDE methods still rely on simulated time-discretized trajectories.\ As a result, approximation errors tend to accumulate over successive time steps, their accuracy depends on how well simulated\hspace{-0.2mm}  paths cover\hspace{-0.2mm} the spatial domain  and a finer time grid increases  computational cost. Recently, \citet{georgoulis2026deep} also propose a time-discretized approach for PIDEs, but combine a
variational neural loss with an implicit-explicit minimizing movement
time-stepping scheme.
\vspace{-0.5mm} 

%However, these methods are restricted to parabolic PIDEs and, because they rely on simulated discretized stochastic processes, they approximate the solution only on a finite time grid, yielding local solutions that become more costly as the number of grid points increases. Consequently, they are less suited to approximating a global solution over the entire PIDE domain.\vspace{-0.5mm}%Finally,  \cite{neufeld2025multilevel}  propose a MLP extension for semilinear PIDEs.

\paragraph{Contributions.} To bridge this gap, we introduce INEUS, a meshfree iterative neural solver for nonlinear PIDEs, inspired by \cite{cheridito2025deep}, with the following key features: \vspace{-1mm}
\begin{itemize}[leftmargin=1em]
\item It combines the global approximation capabilities of PINNs with the efficient treatment
of nonlocal terms provided by Feynman--Kac and BSDE-based methods. 
\item
By using an expectation-free recursive target based on single-point jump sampling, it avoids 
explicit numerical integration of nonlocal terms and remains efficient in high dimensions.
\item
Its recursive formulation reduces the differentiation burden of residual-based PINNs by avoiding backpropagation through the full integro-differential operator and the explicit evaluation of gradients, Hessians, and integral terms altogether.
\end{itemize}
We further provide a contraction-based convergence proof for linear PIDEs, showing that 
repeated minimization of the INEUS loss function approximates both the standard  Feynman--Kac operator and a relaxed contractive variant that enhances algorithmic stability via Polyak averaging. We illustrate the accuracy and scalability of INEUS on several linear and nonlinear numerical examples, and compare its performance with PINNs and deep BSDE methods for PDEs and PIDEs. Proofs of the theoretical results and additional numerical  experiments are given in the Appendix.

	\section{INEUS}
	 We describe INEUS for a representative class of parabolic PIDEs, noting that it extends directly to more general PIDEs. Specifically, we consider equations of the form
  %Although INEUS applies to general PIDEs, we present the method, for concreteness, in the representative case of a parabolic PIDE of the form
  \begin{align}
    &\partial_t u(t,x) 
    + \mathcal{F}\bigl(t,x,u(t,x),\nabla_x u(t,x),\nabla^2_x u(t,x)\bigr)
    + \mathcal{I}(t,x,u) = 0, 
    &&(t,x)\in [0,T) \times D, \label{intgen}
    \\
    &u(T,x) = \varphi(x), 
    &&x\in D, \label{termgen}
    \\
    &u(t,x) = g(t,x), 
    &&(t,x)\in [0,T) \times D^c, \label{boundgen}
\end{align}
for a finite time horizon $T > 0$ and an open subset\footnote{In the special case $D = \R^d$, the complement $D^c$ is empty, so condition \eqref{boundgen} becomes vacuous.} $D \subseteq \R^d$. We write $D_T:= [0,T) \times D$ and $\bar{D}_T =[0,T]\times \bar{D}$. The function $\mathcal{F} \colon D_T \times \R \times \R^d \times \R^{d\times d} \to \R$
specifies the local differential operator,
while  $\mathcal{I}$ is a nonlocal operator given by 
	\begin{equation*}
    \mathcal{I}(t,x,u)
    =
    \lambda(t,x) \int_E
    \ell\!\left(
        u(t,x+\gamma(t,x,e)) - u(t,x)
    \right)
   \nu(de),
\end{equation*} 
for a probability \vspace{-0.75mm}measure $\nu$ on $E:= \R^l \setminus \{0\}$, together with suitable functions $\lambda : D_T\to \R_+$, $\gamma : D_T \times E  \to \R^d$ and $\ell\colon \R \to \R$. Using random variables $E_1, E_2, \ldots  \overset{\rm{i.i.d.}}{\sim}   \nu$, one can write
	\begin{equation}
		\mathcal{I}(t,x,u) = \lambda(t,x) \,  \mathbb{E}_{\nu}\left[ \ell\!\left(u(t,x+\gamma(t,x,E_1)) - u(t,x) \right)\right] . \label{operator2}
	\end{equation}
	Let us denote by $\mathcal{U}$ the set of all functions in $C^{1,2}(D_T)\cap C(\bar{D}_T)$ such that the jump expectation in \eqref{operator2} is finite for all $(t,x) \in D_T$ 
	and by $\mathcal{A}$ the operator on $\mathcal{U}$ given by the left-hand side of \eqref{intgen}:  \begin{equation}  \label{HJBB}
		\begin{aligned}
			\mathcal{A}[u](t,x) &:= \partial_t u(t,x) + \mathcal{F}(t,x, u(t,x), \nabla_x u(t,x), \nabla^2_x u(t,x))+ \mathcal{I}(t,x,u)   \, .
		\end{aligned}
	\end{equation}
	For a given scaling parameter $\xi \in \mathbb{R}\setminus\{0\}$, we introduce the corresponding expectation-free operator $\mathcal{G}_\xi: D_T \times E \times \mathcal{U} \to \mathbb{R}$ \vspace{-1mm} by\footnote{with the convention that whenever $x+\gamma(t,x,e) \in D^c$, we set $u(t,x+\gamma(t,x,e)) = g(t,x+\gamma(t,x,e))$.} 	 \begin{align*}
	 \mathcal{G}_\xi(t,x, e, u) \hspace{-0.1mm}:=& \hspace{0.2mm}u(t,x) + \xi \Big[\partial_t u(t,x)  + 
     \\[-1mm]
     & \mathcal{F}(t,x, u(t,x), \nabla_x u(t,x), \nabla^2_x u(t,x)) \hspace{-0.1mm} + \hspace{-0.1mm}\lambda(t,x) \, \ell\! \left( u(t,x+\gamma(t,x,e)) \hspace{-0.1mm} - \hspace{-0.1mm} u(t,x) \right)  \hspace{-0.1mm}\Big] . 
		\end{align*}%For ease of exposition, we write \(\mathcal{G}_\xi(t,x,e,u)\) instead of \(\mathcal{G}_\xi[u](t,x,e)\) throughout the paper. %, as well as,
	%	\begin{equation}
		%	\mathcal{G}_\xi(t,x,z,u) := \mathcal{G}_\xi(t,x,z,u) - u(t,x) \, . \label{key2}
		%	\end{equation}
In contrast to $\mathcal{A}$, the map $\mathcal{G}_\xi$ evaluates the jump term at a single realization $e \in E$, thereby avoiding the computation of the jump expectations $\mathcal{I}$  at every sampled space-time point $(t,x)$. However, averaging over $E_1\sim\nu$ recovers
the PIDE operator:
\begin{equation*}
\mathbb E_\nu\!\left[
\mathcal G_\xi(t,x,E_1,u)
\right]
=
u(t,x)+\xi\mathcal A[u](t,x).
\end{equation*}
The following proposition then shows that, when $u$ solves the PIDE,
the INEUS target $\mathcal{G}_\xi$ yields an exact population regression that characterizes $u$.

	\begin{Proposition} 	\label{Proposition: L2}
		Let $u\in \mathcal{U}$ be a solution of the PIDE \eqref{intgen}--\eqref{boundgen} %\todo{Have we already proved uniqueness of the solution?}
        and let $Y$ be a $D_T$-valued 
		random variable independent of $E_1$ such that $\E \,[ \mathcal{G}_\xi(Y, E_1, u)^2] 
		< \infty$. 
       Then $u(Y) = v^*(Y)$ almost surely for the Borel measurable function $v^*:D_T \to \R$ minimizing  the mean squared error \vspace{-1mm}
\begin{equation*} \mathbb{E} \edg{ \big( v(Y) - \mathcal{G}_{\xi}(Y, E_1, u) \big)^2 } \end{equation*}
over all Borel measurable functions $v :  D_T \to \R$. 

  % Then the minimizer $v^*$ of the mean squared error  \vspace{-2mm}\begin{equation} \label{eq:MSE} \mathbb{E} \edg{ \big( v(Y) - G_{\xi}(Y, E_1, u) \big)^2 } \vspace{-1mm}, \end{equation} over all Borel measurable functions $v \colon D_T \to \R$ satisfies $u(Y) = v^*(Y)$ almost surely.
	\end{Proposition} 
  Motivated by Proposition \ref{Proposition: L2}, we approximate the PIDE solution $u$  with a neural network $u_{\theta} \colon [0,T] \times \R^d \to \R$ and 
iteratively update its parameters by solving the \vspace{-1mm}regression	
\begin{equation}
		\theta^{(k+1)}  \in \argmin_{\theta}   \mathbb{E}  \Big[ \big( u_\theta(Y)  - 
		\mathcal{G}_\xi(Y , E_1, u_{\theta^{(k)}}) \big)^2 \Big] ,  \label{notyet3}  
	\end{equation} 
	for a $D_T$-valued random variable $Y$ with distribution $\mu$ and independent of $E_1$. 
We use hard-constrained PINNs (hPINNs) \cite{lu2021physics}, which enforce the terminal and boundary conditions \eqref{termgen}--\eqref{boundgen} directly by specifying the network architecture of $u_\theta$ as 
	\begin{equation} \label{hPINN}
		u_\theta(t,x) = A(t,x) + B(t,x)v_\theta(t,x) \, ,
	\end{equation}
	for a function $A(t,x)$ encoding the  terminal and boundary conditions, a  smooth distance function $B(t,x)$  vanishing for $t=T$ and $x\in D^c$, and another neural network $v_\theta:D_T \to \R$  trained  only in the interior $D_T$.
    %where $A(t,x)$ is a function satisfying the terminal and boundary conditions, $B(t,x)$ is a smooth function vanishing on the constraint set $\{t = T\} \cup \{x \in D^c\}$, and $v_\theta : D_T \to \mathbb{R}$ is a neural network trained only in the interior $D_T$.
    %For example, in the case of a terminal condition only, one may take \(B(t,x)=T-t\) and \(A(t,x)=\Phi(x)\). For more complex problems, $A$ and $B$ depend on the exact boundary conditions imposed on the solution. 
    In practice, we sample independent realizations $(t_m, x_m, e_m)_{m=1}^{M}$ of $(Y,E_1)  \sim \mu \otimes \nu$ on $D_T \times E$ and solve  	$$ \theta^{(k+1)}  \in \argmin_{\theta} \frac{1}{M}  \sum_{m=1}^M \big(u_\theta(t_m, x_m)  - \mathcal{G}_\xi(t_m, x_m, e_m, u_{\theta^{(k)}} \hspace{-0.23mm} ) \big)^2 \, .$$ This allows us to mitigate the curse of dimensionality by avoiding any space-time discretization of the PIDE \eqref{intgen}--\eqref{boundgen}, resulting in the meshfree INEUS Algorithm  \ref{Algo1}.
This recursive formulation offers three main computational advantages over standard PINN training. \vspace{-1mm}
\paragraph{Cheaper gradients.} Standard\vspace{-0.2mm} PINNs minimize an empirical objective of the form {\small$ (1/M)\sum_{m=1}^M A[u_\theta]^2(t_m, x_m)$} using a gradient method with respect to $\theta$. Since $\mathcal{A}
$ already contains second-order spatial derivatives, this produces third-order mixed derivatives during backpropagation, which are both expensive to compute and numerically sensitive. In our formulation, the target $\mathcal{G}_\xi(Y, E_1, u_{\theta^{(k)}}\hspace{-0.5mm})$ is evaluated at the frozen weights $\theta^{(k)}$ and treated as a constant during the backward pass. Differentiating \eqref{notyet3} with respect to $\theta$ therefore requires only $\nabla_\theta u_\theta$; see Algorithm~\ref{Algo1} and the explicit gradient update \eqref{eq:algo1}. Notably, this advantage persists even without nonlocal terms, so also for standard PDEs; \hspace{-0.1mm}see Figures \ref{Fig: pde_linear}--\ref{Fig: pde_nonlinear} below.\vspace{-0.6mm}
\paragraph{No numerical integration of the jump term.} The expectation-free target $\mathcal{G}_\xi$ evaluates the jump integrand at a single draw $E_1 \sim \nu$, rather than numerically approximating the full integral term $\mathcal{I}$ at each sampled point $(t,x)$. In Appendix~\ref{app:jump_induced_var}, we examine which choice of the scaling parameter $\xi$ optimally mitigates the jump-induced variance  and discuss the
trade-off between accuracy and computational cost when averaging multiple jump samples in $\mathcal{G}_\xi$. Comparing INEUS with PINNs using the same number of jump samples per
space-time point also helps distinguish the  benefits of the
regression-based INEUS formulation from those of its single-jump sampling.
\paragraph{Preservation of population minimizers under jump sampling.} \hspace{-3mm} Appendix~\ref{app:jump_induced_var} shows that approximating the nonlocal term by Monte Carlo jump sampling introduces an additional jump variance term into the PINN population loss, which can shift its minimizer. This is not the case for INEUS since this additional jump variance term is independent of the function being optimized and evaluated at the \vspace{-0.4mm}frozen target $\mathcal{G}_\xi$.

The next lemma provides a useful characterization of  \eqref{notyet3} in terms of its  underlying operator.
	\begin{Lemma} \label{lemma 1}
		Let $\mu$ be a probability measure on $D_T$ with full support and $u_{\theta^{(k)}} \in \mathcal{U}$ be given at the $k$-th iteration step such that $\E \,[ \mathcal{G}_\xi(Y, E_1, u_{\theta^{(k)}})^2] 
		< \infty$. Let $u^* : \bar{D}_T \to \R$ be a Borel measurable function that minimizes the objective
		\begin{equation} \label{objective_functional} 
		\mathbb{E} \Big[ \big( v(Y) - \mathcal{G}_\xi(Y, E_1, u_{\theta^{(k)}}) \big)^2 \Big],
		\end{equation} 
      over all Borel functions $v :  D_T \to \R$,  where $(Y,E_1) \sim \mu \otimes \nu$. Then, $u^*$ satisfies the update rule 
		\begin{equation} \label{ope1b}
			u^*(t,x) = \overline{\mathcal{T}}_{\hspace{-0.65mm}\xi} \, u_{\theta^{(k)}} (t,x):= u_{\theta^{(k)}}(t,x) + \xi \mathcal{A}[u_{\theta^{(k)}}](t,x) \, , \quad \mu\text{-almost everywhere} \text{ on } D_T .
		\end{equation} 
        Moreover, if $u^* \in \mathcal{U}$ is realized by a hPINN of the form \eqref{hPINN}  and is a fixed point of the operator $\overline{\mathcal{T}}_{\hspace{-0.65mm}\xi}$, then $u^*$ is a classical solution of the PIDE problem \eqref{intgen}--\eqref{boundgen}.
	\end{Lemma}
	Lemma \ref{lemma 1} shows that the update rule \eqref{notyet3} based on an hPINN $u_\theta$ is realized by the operator $\overline{\mathcal{T}}_{\hspace{-0.65mm}\xi}$, whose fixed points are solutions of the PIDE \eqref{intgen}–\eqref{boundgen}. However, \(\overline{\mathcal{T}}_{\hspace{-0.65mm}\xi}\) is, in general, not a contraction.\ 
    Proposition \ref{prop:contrac} below shows that a suitable modification of the operator \(\overline{\mathcal{T}}_{\hspace{-0.65mm}\xi}\) defines a contraction mapping on the space of bounded continuous functions and that this modified operator can be recovered through repeated minimization of \eqref{notyet3}. \vspace{-1.7mm}

	\section{Convergence analysis}\label{sec:convergence}
    The convergence analysis exploits the Feynman–Kac representation and the contractivity of the associated semigroup, which are naturally available for linear PIDEs. We therefore restrict to this setting, noting that the algorithm itself applies unchanged to the nonlinear problems of Section~\ref{Section: numeric}. Appendix~\ref{app: generalization} also presents an extension of the linear theory to solution functions of polynomial growth and discusses nonlinear settings. We hence consider  linear PIDEs in $D=\R^d$ of the form\footnote{Boundary conditions \eqref{boundgen} and general open domains $D \subseteq \R^d$ can  both be treated easily by introducing the stopping time $\tau := T \wedge \inf \{s\geq t :X_s \notin D \}$ in the Feynman--Kac formula, see Th.\ 4.2 in \cite{yong1999stochastic}.\ We here omit this extension for readability.}

\begin{align}
    \partial_t u(t,x)
    + \mathcal L[u](t,x)
    + \mathcal I[u](t,x)
    - c(t,x)u(t,x)
    + f(t,x)
    &= 0,
    && (t,x)\in D_T,
    \label{intdiff}
    \\
    u(T,x) &= \varphi(x),
    && \ x\in D.
    \label{termdiff}
\end{align}
Here, we denote by $\mathcal C$ the space of bounded, uniformly continuous real-valued functions on $\bar D_T$. We assume that
$f,c,\varphi\in\mathcal C$ and that there exists a constant $c_0>0$ such that
    \begin{equation} \label{eq:assumption_c}
        c(t,x)\ge c_0>0, \qquad (t,x)\in\bar{D}_T \, .
    \end{equation} The non-local operator $\mathcal{I}$ in \eqref{operator2} is such that $\ell(z) =z$ and the differential operator $\mathcal{L}$ is of the form
	\begin{equation*}
		\mathcal{L}[u](t,x) = b^\top\hspace{-0.4mm}(t,x) \nabla_x u(t,x) + \frac{1}{2} \text{Tr}\left[\sigma \sigma^\top\hspace{-0.4mm}(t,x) \nabla^2_x u(t,x) \right] ,
	\end{equation*}
	with coefficients $b: D_T \to \R^d$, $\sigma: D_T \to \R^{d\times q}$, $\gamma:D_T \times E \to \R^d$, $\lambda: D_T \to \R_+$ ensuring that \eqref{sde} has a unique solution exhibiting uniform stochastic continuity and that the associated Feynman--Kac
operator $\mathcal T_h:\mathcal{C}\to\mathcal{C}$ introduced below
is well defined for every $h\geq 0$.
	\begin{Theorem}[Feynman--Kac] \label{thm:fk}
		Let $u \in \mathcal{U}\cap\mathcal{C}$ be a classical solution of the PIDE \eqref{intdiff}--\eqref{termdiff}, then for every $(t,x) \in \bar{D}_T$,
		\begin{equation}
		\hspace{-1.2mm}	u(t,x) = \hspace{-0.25mm}\mathbb{E}_x\left[ \int_t^T \hspace{-1mm}e^{-\int_t^s c(r, X^{t,x}_r) dr} f(s, X^{t,x}_s) ds + e^{-\int_t^T c(r,X^{t,x}_r) dr} \varphi(X^{t,x}_T) \right]\hspace{-0.32mm} , \hspace{-0.5mm}\label{Feynman}
		\end{equation}
		for a stochastic process $(X^{t,x}_s)_{t\leq s \leq T}$ satisfying
		\begin{equation}
			dX^{t,x}_s = b(s,X^{t,x}_s) ds +\sigma(s, X^{t,x}_s) dW_s + \int_E \gamma(s,X^{t,x}_{s-},  e) \, dN(ds, de) , \quad X^{t,x}_t =x \, , \label{sde}
		\end{equation}
		where $W$ is a $q$-dimensional Brownian motion and $N$ is an integer-valued random measure  on \([0,T] \times E\) with predictable compensator \( \lambda(s,X^{t,x}_{s-})\nu(de)ds\).\vspace{-0.5mm}
	\end{Theorem}

	\subsection{The Feynman--Kac semigroup and its generator} 
The Feynman--Kac representation \eqref{Feynman} naturally defines a two-parameter family of affine operators. To obtain a one-parameter semigroup structure,  we fix an initial point $y=(t,x) \in \bar{D}_T$ and introduce the space-time process $Y^y_s := (t+s , X^{t,x}_{t+s})$ for $0 \leq s \leq T-t$. For $h\geq 0$, we set $\tau_h(t) := h \wedge (T-t)$ and  introduce the function spaces \[
\mathcal C_\varphi
:=
\left\{
\psi\in \mathcal{C}
:
\psi(T,x)=\varphi(x)
\text{ for all }x\in D
\right\}, \] \[
\mathcal C_0
:=
\left\{
\psi\in \mathcal{C}
:
\psi(T,x)=0
\text{ for all }x\in D
\right\}.
\] 
The set $\mathcal C_\varphi$ is a closed affine subspace of $ \mathcal{C}$ and is therefore complete under the supremum norm $\|\cdot\|_\infty$, whereas $\mathcal C_0$ is a Banach space under $\|\cdot\|_\infty$. For $\psi \in \mathcal{C}$, define the Feynman--Kac operator
	\begin{equation}
		\mathcal{T}_{h}\psi(y) :=  \mathbb{E}_y\left[\int_0^{\tau_h(t)} e^{-\int_0^s c(Y^y_r)dr} f(Y^y_s) ds + e^{-\int_0^{\tau_h(t)} c(Y^y_r) dr}\psi(Y^y_{\tau_h(t)}) \right] \hspace{-0.2mm}   , \label{semigroupbb}
	\end{equation}
	%with  $u(T,x) = \varphi(x)$ when $\psi = u$. Although $(\mathcal{T}_h)$ is now a F-K semigroup acting on functions of space and time, we note that the presence of the source term $f\not\equiv 0$ makes $\mathcal{T}_h\bar{F}$  affine.  Hence, we can rewrite the semigroup $\mathcal{T}_h$ in \eqref{semigroupbb} as
   where\hspace{-0.3mm} $\mathbb E_y$\hspace{-0.3mm} denotes expectation conditional on $Y_0^y\hspace{-0.3mm}=y$.\ The affine operator $\mathcal T_h$ admits the decomposition
	\begin{equation} \label{eq: decomposition}
		\mathcal{T}_h \psi = \mathcal{R}_h + \mathcal{S}_h \psi \, , 
	\end{equation}
where $\mathcal R_h$ accounts for the source term $f$,
\begin{equation*}
    \mathcal{R}_h(y) := \mathbb{E}_y\left[\int_0^{\tau_h(t)} e^{-\int_0^s c(Y^y_r) dr} f(Y^y_s) ds \right]  \vspace{-0.5mm},
\end{equation*} 
and $\mathcal S_h$ is the linear Feynman--Kac operator
	\begin{equation}
    \mathcal{S}_h \psi(y) := \mathbb{E}_y\left[ e^{-\int_0^{\tau_h(t)} c(Y^y_r) dr} \, \psi(Y^y_{\tau_h(t)})\right], \quad \psi \in \mathcal{C}.
    \label{semigroupb}
	\end{equation}

    \begin{Proposition}\label{prop:contrac}
    Let $(\mathcal S_h)_{h\geq 0}$ be the family of linear operators defined
    in \eqref{semigroupb}. Then, the following statements \vspace{-1mm}hold.
    \begin{enumerate}[align=left,leftmargin=*, labelsep=-0.2em, label=\textnormal{(\roman*)}]
    \item The family \label{prop:contrac:item1} $(\mathcal S_h)_{h\geq 0}$ is a strongly continuous, monotone  semigroup 
    on $\mathcal{C}$.
    \item \label{prop:contrac:item2} Under Assumption \eqref{eq:assumption_c}, for every $h>0$,
    the operator $\mathcal S_h$ maps $\mathcal C_0$ into itself
    and is a strict contraction on this space:
     \begin{equation*} 
    	\|\mathcal{S}_h \psi \|_\infty 
    	\;\le\; 
    	e^{-c_0 h}\,\|\psi\|_\infty, \qquad \psi \in \mathcal{C}_0.
    \end{equation*}
    Moreover, $\mathcal T_h$ maps $\mathcal C_\varphi$ into itself
    and is a strict contraction on this space, with
    \begin{equation*}
    \|\mathcal{T}_h \psi_1 - \mathcal{T}_h \psi_2\|_\infty 
    \;\le\; 
    e^{-c_0 h}\,\|\psi_1 - \psi_2\|_\infty, \qquad \psi_1,\psi_2\in\mathcal C_\varphi.
    \end{equation*}
    \item \label{prop:contrac:item3}
    Let $\mathcal A_0$ denote the infinitesimal generator of $(\mathcal S_h)_{h\geq0}$, with domain
\[
D(\mathcal A_0)
:=
\left\{
\psi\in\mathcal C :
\exists\,\Phi\in\mathcal C,\ 
\lim_{h\downarrow0}
\left\|
\frac{\mathcal S_h\psi-\psi}{h}-\Phi
\right\|_\infty=0
\right\}.
\]
For every $\psi\in D(\mathcal A_0)\cap\mathcal U$ such that the function
$\partial_t\psi+\mathcal L[\psi]+\mathcal I[\psi]-c\psi$
defined on $D_T$ admits  a bounded, uniformly continuous
extension to $\bar D_T$, then $\mathcal A_0\psi$ coincides with this extension:
\begin{equation*}
\mathcal A_0\psi(t,x)
=
\partial_t\psi(t,x)
+\mathcal L[\psi](t,x)
+\mathcal I[\psi](t,x)
-c(t,x)\psi(t,x),
\qquad (t,x)\in D_T.
\end{equation*}
    %where $\mathcal L$ denotes the diffusion operator and $\mathcal I$ the jump operator associated with the time–space process $Y$.
    \end{enumerate}
    \end{Proposition}

	In particular, if $u\in D(\mathcal{A}_0) \cap \mathcal{U}$ is a classical solution of the PIDE \eqref{intdiff}, then we obtain from  \eqref{HJBB}, $$\mathcal{A}[u](t,x) =  \mathcal{A}_0u(t,x) + f(t,x) = 0, \qquad (t,x) \in D_T \, .$$  
	\subsection{A contractive simulation-free operator}
	Recall that update rule \eqref{ope1b} is realized by the operator $\overline{\mathcal{T}}_{\hspace{-0.65mm}\xi}$ acting on $\mathcal{U}$, where
	$$\overline{\mathcal{T}}_{\hspace{-0.65mm}\xi} \, \psi  = \psi +  \xi	\mathcal{A} [\psi]\, .$$ 
   On the other hand, one can see from  \eqref{eq: decomposition} and the pointwise relation
$\mathcal R_h(t,x)/h\to f(t,x)$  that 

\[
\overline{\mathcal{T}}_{\hspace{-0.65mm}\xi} \, \psi(t,x)
=\psi(t,x)+\xi\lim_{h\downarrow0}
\frac{\mathcal T_h\psi(t,x)-\psi(t,x)}{h},
\qquad(t,x)\in D_T,
\]
for all $\psi \in \mathcal{D}(\mathcal{A}_0) \cap \mathcal{U}$.  By Lemma \ref{lemma 1}, the exact solution $u$ is a  fixed point of $\overline{\mathcal{T}}_{\hspace{-0.65mm}\xi}$. However, in general, $\overline{\mathcal{T}}_{\hspace{-0.65mm}\xi}$ is not contractive. We therefore introduce  a \textit{relaxed operator} based on the Feynman--Kac affine family $(\mathcal{T}_h)_{h \geq 0}$.
    \begin{Proposition} \label{Prop:modified}
       For $h>0$ and $\alpha>0$, define
       $\mathcal{T}_{h,\alpha}\colon \mathcal C_\varphi
       \to \mathcal C_\varphi$ by
    \begin{equation} \label{mod_operator}
		\mathcal{T}_{h,\alpha}\psi := \psi	+ \frac{\alpha}{h} \, (\mathcal{T}_{h}\psi - \psi) \, = \Big(1-\frac{\alpha}{h}\Big) \psi + \frac{\alpha}{h}\mathcal{T}_{h}\psi \, .
	\end{equation}
    \begin{enumerate}[align=left,leftmargin=*, labelsep=-0.2em, label=(\roman*)]
    \item \label{prop:mod:item2} Assume that \eqref{eq:assumption_c} holds. If 
    $0<\alpha <2h/(1+e^{-c_0 h})$,
    then \(\mathcal{T}_{h,\alpha}\) is a strict contraction on \(\mathcal C_\varphi\), with Lipschitz constant
    \[
    \left|1-\frac{\alpha}{h}\right|
    +
    \frac{\alpha}{h}e^{-c_0 h} \in (0,1) \, .
    \]
    \item\label{prop:mod:item3} The solution $u$ of the PIDE \eqref{intdiff}--\eqref{termdiff} is the unique fixed point of $\mathcal{T}_{h,\alpha}$ in $\mathcal{C}_\varphi$.
    \end{enumerate}

    \end{Proposition}
First, $\mathcal{T}_{h,\alpha}$ can be viewed as a finite-time-step approximation of $\overline{\mathcal{T}}_{\hspace{-0.65mm}\xi}$: when $\alpha=\xi$, it replaces
the infinitesimal increment $\mathcal A[\psi]$ by
$(\mathcal T_h\psi-\psi)/h$. More generally,  the \textit{relaxation} parameter $\alpha$ in \eqref{mod_operator} need not coincide with the \textit{scaling} parameter $\xi$ used in $\overline{\mathcal{T}}_{\hspace{-0.65mm}\xi}$. This additional degree of freedom will play an important role in Theorem~\ref{Prop: conv}, where it provides greater flexibility in controlling the stability\vspace{-0.5mm} of the numerical scheme.
%note first that $\alpha$ can be different from $\xi$ as we will show in Proposition \ref{Prop: conv}. A
Second, although $\mathcal{T}_{h,\alpha}$ is contractive for  $0<\alpha <2h/(1+e^{-c_0 h})$, a direct implementation of this operator would still require the evaluation of the Feynman--Kac operator $\mathcal T_h$, which in turn involves simulating the dynamics of
$Y^y$ and approximating the conditional expectation by Monte Carlo averages over many sample paths. Such a procedure tends to be computationally expensive, introduces time-discretization
errors and yields only a local approximation for fixed initial conditions $y$.\ We avoid these difficulties by using the following semigroup expansion. 
%For ease of exposition, Propositions \ref{ass:series} and \ref{Prop: conv} are stated for $f\in \mathcal{C}_0$. Appendix x details a more general construction  that  removes this assumption.
\begin{Proposition}\label{ass:series}
Fix $n\geq 1$ and let $\psi\in\mathcal U\cap\mathcal{C}$. Assume the residual
\[
\mathcal A\psi
=
\partial_t\psi
+\mathcal L[\psi]
+\mathcal I[\psi]
-c\psi+f \, ,\] 
satisfies
$\mathcal{A}\psi\in D(\mathcal A_0^{n})\cap\mathcal C_0$. %\footnote{A typical sufficient condition is  $\psi,f\in C_b^{n,2n}(\bar D_T) \cap \mathcal{U}$ with  smooth coefficients $b, \sigma, \gamma,c, \lambda$ having bounded derivatives up to order \(n\), and that the jump term admits the required integrability so that \(\mathcal A_0^k\psi\in \bar C\)  and \(\mathcal A_0^k f\in \bar C\) for \(k\le n\).}.
 Then, for every $h\geq 0$, there exist remainder terms $R_{n+1}(h) \in \mathcal{C}$ such that the Feynman--Kac operator \eqref{semigroupbb} admits the expansion
\begin{align} \label{eq_mod_op}
	\mathcal{T}_{h}\psi %&= \psi +\sum_{j=1}^{\infty}\frac{h^j}{j!} \left(	\mathcal{A}_0^j \psi + \mathcal{A}_0^{j-1} f \right)\\
	&= \psi +\sum_{j=1}^{n}\frac{h^j}{j!} \mathcal{A}_0^{j-1} (\mathcal{A} \psi) + R_{n+1}(h) \, ,
\end{align}
where
\[
\|R_{n+1}(h)\|_\infty
\leq
\frac{h^{n+1}}{(n+1)!}
\big\|\mathcal A_0^{n}(\mathcal A\psi)\big\|_\infty
=
\mathcal O(h^{n+1})
\qquad\text{as }h\downarrow0.
\]
Moreover, for $h>0$ and $\psi \in \mathcal{C}_\varphi$, the relaxed operator \eqref{mod_operator}  satisfies
\begin{align}\label{fin_approx}
\mathcal{T}_{h,\alpha}\psi
&=
\psi+\alpha\left[
\sum_{j=1}^{n}\frac{h^{j-1}}{j!}\,
\mathcal A_0^{j-1}(\mathcal A\psi)
+\widetilde R_{n}(h)
\right],
\end{align}
where $\widetilde R_{n}(h):=R_{n+1}(h)/h$ satisfies
\[
\|\widetilde R_{n}(h)\|_\infty
\leq
\frac{h^{n}}{(n+1)!}
\big\|\mathcal A_0^{n}(\mathcal A\psi)\big\|_\infty
=
\mathcal O(h^{n})
\qquad\text{as }h\downarrow0.
\]
\end{Proposition}
	\subsection{An efficient training algorithm} \label{section:efficient algo}
Importantly, since $\mathcal{A}_0$ is a second-order partial integro-differential operator, evaluating $\mathcal{A}_0^j$  requires computing  partial derivatives of order up to $2j$,  which quickly becomes computationally prohibitive. The following theorem shows that repeated exact minimization of the INEUS objective \eqref{notyet3} generates an approximation of both the truncated expansions  \eqref{eq_mod_op} and \eqref{fin_approx} of the Feynman--Kac operator $\mathcal{T}_h$ and  its relaxed operator $\mathcal{T}_{h,\alpha}$, while avoiding the explicit computation of the  costly operators $\mathcal{A}_0^j$. 
%A slight extension of Lemma \ref{lemma 1} for the infinitesimal generator $\mathcal{A}$ directly shows that if the neural network is expressive enough and is an exact minimizer of \eqref{loss}, the following update rule at epoch $k$ with $\theta^{(k)}_0 := \theta^{(k)}$,
%	$$ u_{\theta^{(k)}_1} = \argmin_\theta \mathscr{L}^{(k)}(\theta)$$ effectively becomes 
%	$$u_{\theta_1^{(k)}} = (I + \xi \mathcal{A}_0)u_{\theta^{(k)}} + \xi f \, .$$
	
\begin{Theorem}
\label{Prop: conv}
Fix an epoch $k$ and an integer $n\geq 1$. Let $h>0$,
set $\xi=h/n$, and choose $0<\alpha<2h/(1+e^{-c_0 h})$.
Assume that
\begin{equation} \label{eq:assump}
u_0:=u_{\theta^{(k)}}
\in \mathcal U\cap\mathcal C_\varphi,
\qquad
\mathcal Au_0
\in D(\mathcal A_0^{n})\cap\mathcal C_0.
\end{equation}
Define the $n$-step inner recursion, for $i=0,\dots,n-1$, by
\begin{equation}\label{eq:inner_min1}
u_{i+1}:=u_{\theta_{i+1}},
\qquad
\theta_{i+1}\in\arg\min_\theta
\mathbb E\Big[
\big(u_\theta(Y)-\mathcal G_\xi(Y,E_1,u_i)\big)^2
\Big],
\end{equation}
where $(Y,E_1)\sim\mu\otimes\nu$ and $\mu$ is a probability
measure with full support on $D_T$. At every
inner step $i$, assume that  $\E \,[ \mathcal{G}_\xi(Y, E_1, u_i)^2] < \infty$ and that the corresponding minimum over
all square-integrable measurable functions of $Y$ is attained
by an admissible hPINN
$u_{i+1}\in\mathcal U\cap\mathcal C_\varphi$. Define the next outer iterate by the Polyak average
\begin{equation}\label{eq:relaxed-update-clean}
u^{(k+1)}
:=
\left(1-\frac{\alpha}{h}\right)u_{\theta^{(k)}}
+
\frac{\alpha}{h}  u_{n}.
\end{equation}
Then, on $\bar D_T$,
\begin{equation}\label{eq:conv-relaxed}
u^{(k+1)}
=
u_{\theta^{(k)}}
+\alpha\sum_{j=1}^{n}\frac{h^{j-1}}{j!}\,
\mathcal A_0^{j-1}\bigl(\mathcal Au_{\theta^{(k)}}\bigr)
+
\frac{\alpha}{h}E_{n}(h),
\end{equation}
where the remainder $E_{n}(h) \in \mathcal{C}$  is given by \eqref{eq:Enh} and satisfies
\begin{equation}\label{eq:global-binomial-bound}
\|E_{n}(h)\|_\infty
\leq\frac{h^2}{2n}
\sum_{j=2}^{n}\frac{h^{j-2}}{(j-2)!}
\|\mathcal A_0^{j-1}(\mathcal Au_0)\|_\infty \, .
\end{equation}
In particular, the relaxed operator $\mathcal{T}_{h,\alpha}$ satisfies  as $h \downarrow 0$,
\begin{equation}\label{eq:conv-relaxed-Ttilde}
\|u^{(k+1)}-\mathcal{T}_{h,\alpha} u_{\theta^{(k)}}\|_{\infty}
\le
\frac{\alpha}{h}\|E_{n}(h)\|_{\infty}
+
\alpha\|\widetilde{R}_{n}(h)\|_{\infty}
=
\mathcal O\!\left(\frac{\alpha h}{n}\right)
+
\mathcal O(\alpha h^{\,n}) \, .
\end{equation}
Moreover,\hspace{-0.2mm} for $\alpha = h$,\hspace{-0.2mm} one has $u^{(k+1)}=u_{n}$\hspace{-0.2mm} and \(\mathcal{T}_{h,\alpha}=\mathcal T_h\)\hspace{-0.01mm}.\hspace{1mm}Hence, as $h \downarrow 0$,
\begin{equation}\label{eq:conv-inner-Th}
\big\| u^{(k+1)}-\mathcal T_hu_{\theta^{(k)}}\big\|_{\infty}
\le
\|E_{n}(h)\|_{\infty}+\|R_{n+1}(h)\|_{\infty}
=
\mathcal O\!\left(\frac{h^2}{n}\right)+\mathcal O(h^{n+1}).
\end{equation}

\end{Theorem}

In practice, each inner minimization \eqref{eq:inner_min1} is approximated by several steps of mini-batch gradient descent \citep{du2019gradient,allen2019convergence}, yielding the  efficient  and robust INEUS training procedure described in Algorithm \ref{Algo1}. Importantly, considering the relaxed operator $\mathcal{T}_{h,\alpha}$ instead of the Feynman--Kac operator $\mathcal{T}_h$ provides an additional tunable parameter $\alpha$, which increases the flexibility and stability of the scheme.  In particular, since $\alpha = h$ gives back the original Feynman--Kac operator,  $\mathcal{T}_{h,\alpha}$ contains $\mathcal{T}_h$ as a special case. Since INEUS is an iterative scheme, it is known to be less stable for nonlinear function approximation than residual-based methods such as standard PINNs, see \cite{baird1995residual, cheridito2025deep}.  But for $0<\alpha<h$, the relaxed operator $\mathcal{T}_{h,\alpha}$ then
helps mitigate this training instability by damping changes in the
network output through the Polyak averaging
step~\eqref{eq:relaxed-update-clean}, as illustrated in
Figure~\ref{Fig: FKvsmod} of Appendix~\ref{Appendix: additional}. %Importantly, the relaxed recursion in Proposition \ref{Prop: conv} \ref{item:propconv:2} can be applied only at the end of  a selected number of epochs $k$. I

Moreover, choosing the scaling parameter \(\xi = h/n\) mitigates the jump-induced variance of the single-jump target \(\mathcal{G}_\xi\), as quantified  in Proposition~\ref{prop:variance_estim}. %\todo{To change this since $\xi=h/(n-1)$}This instability can also arise in our setting from the single-jump point operator $\mathcal{G}_\xi$, which introduces additional variance, although this effect is mitigated by choosing a smaller scaling parameter $\xi$, as shown in 
Also, the operator \(\mathcal{G}_\xi\) requires the evaluation of the second-order differential operator \(\mathcal L[u_\theta]\), which in principle involves computing \(\partial_t u_\theta\), \(\nabla_x u_\theta\), and \(\nabla_x^2 u_\theta\). However, Proposition \ref{Proposition: efficent} reduces the cost of evaluating $\mathcal{G}_\xi$ by replacing this explicit computation of the gradients and Hessians of \(u_\theta\) with a single directional second derivative. Finally, the hPINN
parameterization $u_\theta$ can be modified to impose the
higher-order terminal compatibility assumption \eqref{eq:assump}. However, we did not observe a performance gain
from this modification in our experiments and therefore retained
the original and more expressive hPINN parameterization~\eqref{hPINN}.

\section{Numerical results}  \label{Section: numeric}
The proposed  INEUS method for nonlinear  PIDEs is  summarized in Algorithm \ref{Algo1} below.  In all examples, we use the DGM architecture of \cite{sirignano2018dgm} as it has been shown to improve PINN performance empirically. %In our numerical experiments, we use the DGM architecture of \cite{sirignano2018dgm} as it has been shown toempirically improve PINN performance, combined with a hPINN construction. 
Further details of the network design and hyperparameters are given in Appendix \ref{app:algo}. Additional experiments for both PDEs and PIDEs are  reported in Appendix \ref{Appendix: additional}.
	\paragraph{Linear PIDE with quadratic terminal condition.}
	We first assess INEUS on a benchmark problem with closed-form solution. We consider the  linear PIDE in $d$ dimensions on $D_T = [0,T) \times \mathbb{R}^d$,
{\small 	\begin{equation}
	\begin{aligned}
	\hspace{-0.1mm}	\partial_t u(t,x) \hspace{-0.2mm} + \hspace{-0.2mm} b x^\top \nabla_x u(t,x) \hspace{-0.2mm} &+ \hspace{-0.2mm} \frac{1}{2}\text{Tr}[\Sigma \Sigma^\top \nabla_x^2 u(t,x)]  \hspace{-0.2mm}- \hspace{-0.2mm}c u(t,x) 
+\lambda \hspace{-0.2mm}\int_{\R^d} \hspace{-1.1mm}\big(u(t,x+e)- u(t,x)\big) \varphi_{\Sigma_J}\hspace{-0.2mm}(e) de= 0 \, ,
	\end{aligned} \label{linear_PDE}
	\end{equation}}subject to the terminal condition $u(T,x) = \|x\|^2$, where $\Sigma \in \R^{d\times q}$ and $\varphi_{\Sigma_J}$ is the $\mathcal{N}(0,\Sigma_J)$-density with $\Sigma_J \in \R^{d\times d}$. % It is  direct to show, see \cite{cheridito2025deep}, that the solution satisfies
	%$$u(t,x) = e^{(2\xi -c) (T-t)} \|x\|^2 + b(t) \, ,$$
	%where $b(t)$ solves the one-dimensional ODE
%	$$b'(t) - c b(t)  = -\big(\text{Tr}[\Sigma \Sigma^\top] + \lambda_t \text{Tr}[\Sigma_J] \big)  e^{(2\xi -c) (T-t)} \,. $$
The corresponding solution is explicitly given by 
	\begin{equation} \label{analytical_linear}
	    u(t,x) = e^{(2b -c) (T-t)} \|x\|^2 + e^{-c(T-t)} \big(\text{Tr}[\Sigma \Sigma^\top] + \lambda \text{Tr}[\Sigma_J] \big)  \frac{e^{2b(T-t)}-1}{2b} \, .
	\end{equation}
Figure \ref{fig: linear_LQR_u} displays the INEUS approximation for this problem in dimension \(d=100\). In this high-dimensional setting with jumps, standard PINNs become computationally intractable, whereas INEUS still provides an accurate approximation of the solution.
\begin{figure}[h]
	\centering
	\begin{minipage}{.5\textwidth}
		\includegraphics[width=0.95\columnwidth, height=4.1cm]{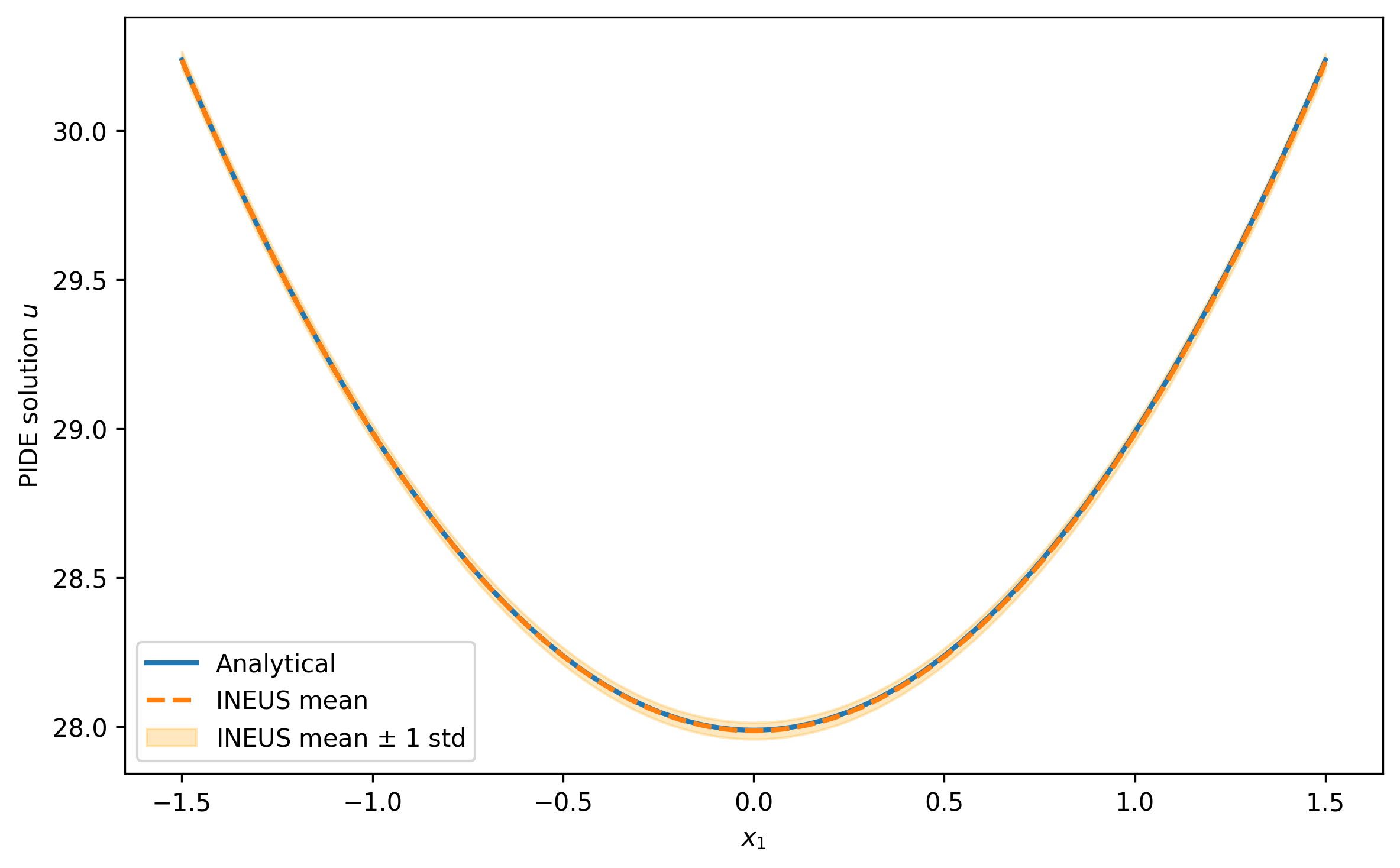}
		\label{fig:test10} \vspace{6mm}
	\end{minipage} \hspace{-4mm}
	\begin{minipage}{.5\textwidth} 	\vspace{-6mm}
\includegraphics[width=0.96\columnwidth, height=4.1cm]{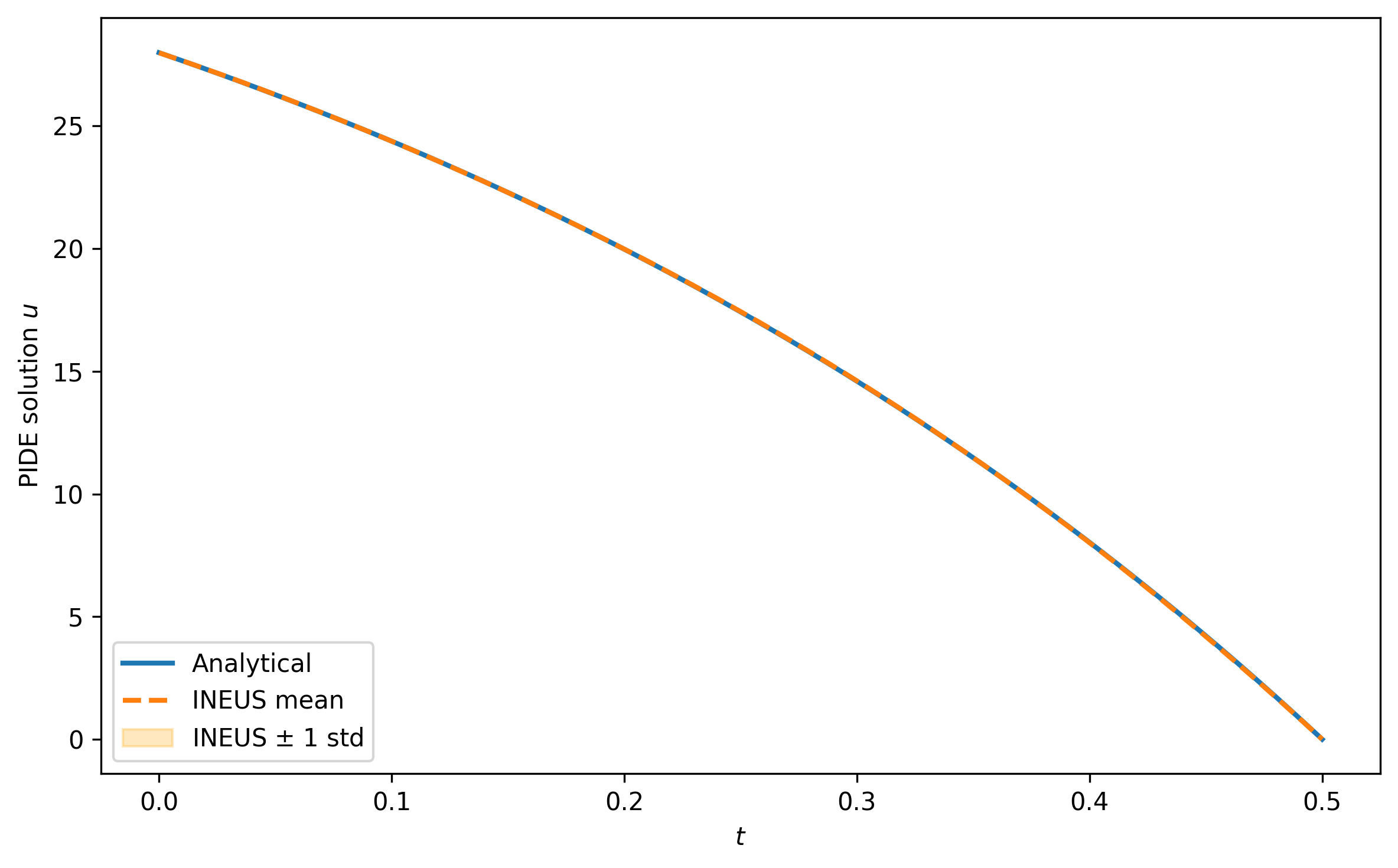}
		\label{fig:test20}
	\end{minipage}\vspace*{-6mm}
\caption{\small PIDE solution $u(0,x)$  for $x=(x_1,0, \dots, 0)$ with $x_1  \in [-1.5, 1.5]$ (left) and  $u(t,x)$ for  $t \in [0,0.5]$ and $x= \mathbf{0}_{100}$ (right) for a 100-dimensional linear PIDE  of the form \eqref{linear_PDE}. Orange dotted lines: numerical results of INEUS with $\pm 1$ standard deviation given by orange shaded area. Blue lines: analytical solution \vspace{-4mm}\eqref{analytical_linear}.}
\label{fig: linear_LQR_u}
\end{figure}  
\begin{figure}[h]
	\centering
    \vspace{-2mm}
	\subfigure{ \hspace{-1mm}
		\includegraphics[width=0.485\columnwidth, height=4.1cm]{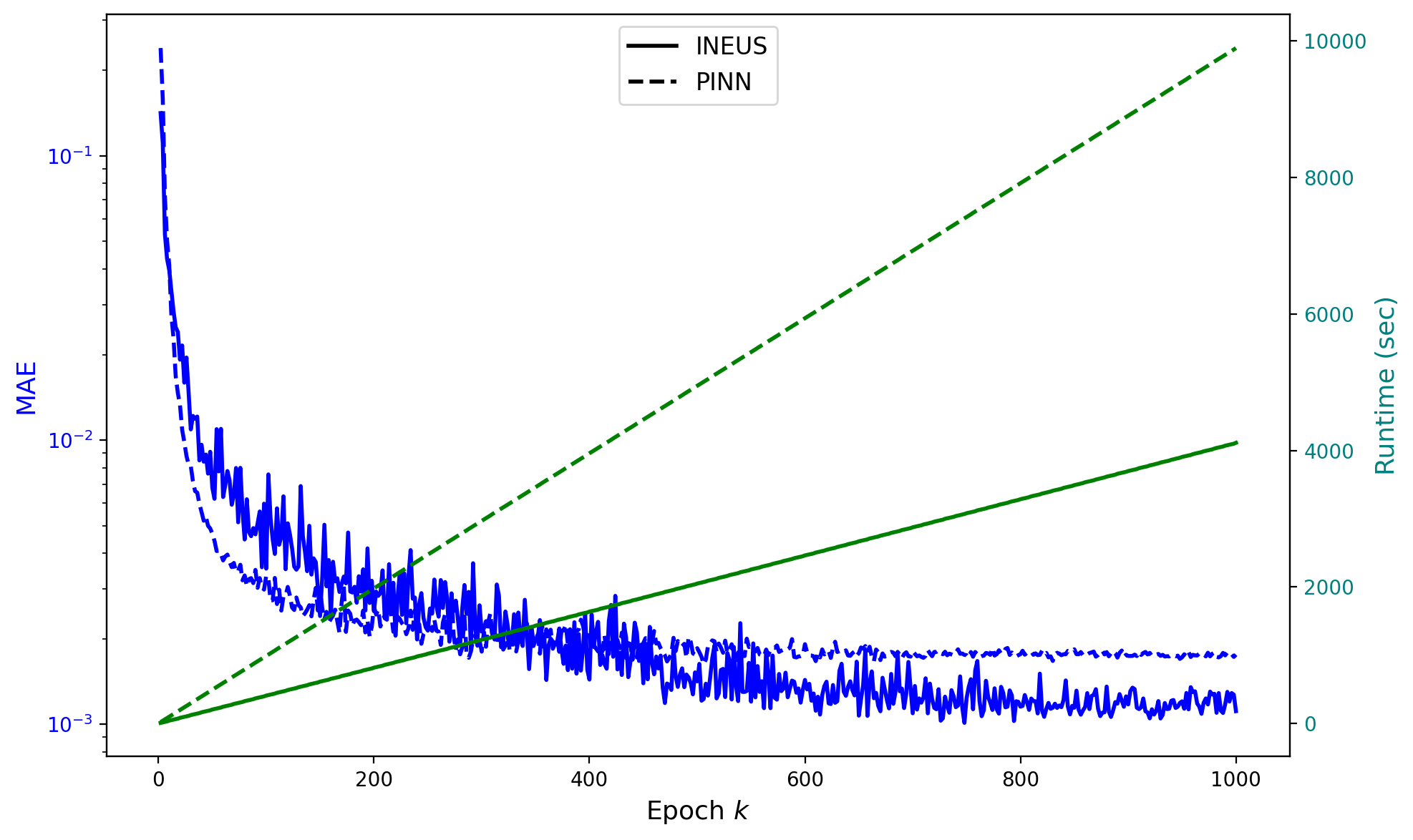}
		\label{fig:sub44}
	}
	 \hspace{-3mm}
	\subfigure{
\includegraphics[width=0.49\columnwidth, height=4.1cm]{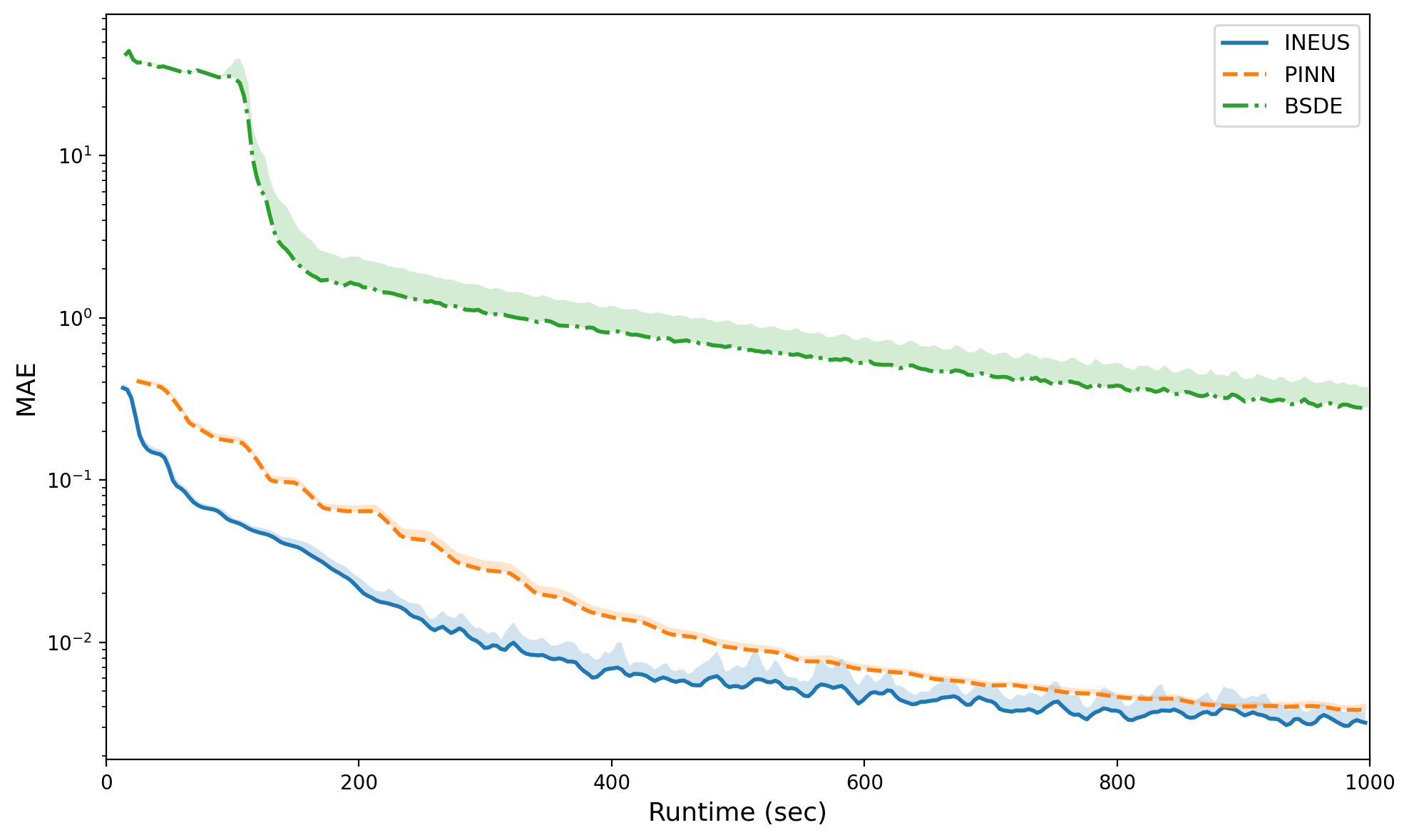}
	\label{fig:sub33}
	} \vspace{-1.95mm}
	\caption{{\small Numerical solutions of a 10-dimensional linear PIDE of the form \eqref{linear_PDE}. Left panel: comparison of ${\rm MAE}$ (blue) and runtime in seconds (green) of INEUS (solid lines) and  PINN (dashed lines) as functions of the epoch $k$. Right panel: comparison of INEUS, PINN and deep BSDE methods as functions of runtime (seconds)\vspace{-2mm}.}} 
	\label{fig1} 
\end{figure} \\
Figure \ref{fig1}  compares INEUS  with competing methods for $d=10$, where PINNs are now computationally feasible.
The \vspace{-0.1mm}left plot reports the runtime and mean absolute error (MAE) of $u_{\theta^{(k)}}$ as\hspace{-0.1mm} functions of the number of epochs $k$, 
where $${\small {\rm MAE} =  \frac{1}{\mathcal{M}} \sum_{i=1}^{\mathcal{M}} |u_{\theta^{(k)}}\hspace{-0.37mm}(t_i, x_i) - u(t_i, x_i) | }$$ 
on a test set of $\mathcal{M}=2000$ uniformly sampled points from $[0,0.5] \times [-1.5,1.5]^d$.\hspace{-0.1mm} Overall, \vspace{-0.1mm}INEUS exhibits  faster convergence in terms of epochs and a substantially lower computational cost. This gain\vspace{-0.1mm} comes from avoiding both the numerical evaluation of third-order \vspace{-0.1mm}derivatives and the numerical \vspace{-0.1mm}integration of the jump term. \vspace{-0.1mm}However, we see that the residual-based structure of PINN makes it slightly more stable than INEUS \cite{baird1995residual}. The right panel shows that INEUS also outperforms the deep BSDE method in terms of $\rm{MAE}$ as a function of runtime. To obtain a meaningful comparison over the whole space-time domain, we ran the jump-adapted \vspace{-0.1mm}deep BSDE  \cite{andersson2025deep} starting from many \vspace{-0.1mm}randomly sampled initial points \((t_0,x_0)\in D_T\). However, because that approach depends on simulated discretized \vspace{-0.1mm}stochastic
processes, its performance deteriorates when used to approximate a global solution, \vspace{-0.1mm}both due to discretization errors that accumulate and to its limited ability to cover the full domain. \vspace{-0.1mm}These observations are further supported by the following nonlinear examples and by the additional experiments on higher-dimensional PIDEs and PDEs in Appendix~\ref{app:baselines} and \ref{Appendix: additional}.\vspace{-0.5mm}

\paragraph{Nonlinear Hamilton--Jacobi--Bellman PIDE.} \label{section: nonlinear}
 \hspace{-0.5mm} Let us consider the nonlinear PIDE on $[0,T) \times \R^d$,
{\small \begin{equation} \label{nonlinear}
 	\partial_t u(t,x) + \text{Tr}[\nabla^2_x u(t,x)] - \eta \|  \nabla_x u(t,x) \|^2 - \frac{\lambda}{\eta} \hspace{-0.5mm}\int_{\R^d} \hspace{-0.75mm}\left(e^{-\eta(u(t,x+e) - u(t,x))} -1 \right) \nu(de) + f(t,x) = 0,
 \end{equation}}with terminal condition $u(T,x) = \varphi(x)$.\ %Define the function $\Psi$ as the solution of the linear PIDE
% \begin{equation} \label{lin_PIDE}
% 	\partial_t v(t,x) + \text{Tr}[\Sigma \Sigma^\top\nabla^2_x v(t,x)] + \lambda \int_{\R^d} (v(x+z) - v(x)) \, \nu(dz) - \eta f(t,x) v(t,x) = 0
% \end{equation}
% with terminal condition $v(T,x) = e^{-\eta g(x)}$. 
By the Feynman--Kac formula \eqref{Feynman}, the solution is given by
{\small \begin{equation} \label{repres1}
	 u(t,x) = - \frac{1}{\eta} \log \Bigg( \mathbb{E}\left[\exp\left(-\int_t^T \eta f(s,X_s) ds \right) e^{-\eta \varphi(X_T)} \, \Big| \, X_t =x \right] \Bigg).
\end{equation}}where {\small $
 	dX_s = \sqrt{2} \,   dW_s  + \int_{\R^d} e N(ds, de)$, $t \leq s \leq T, \ X_t = x, $}
and $N$ is a Poisson  measure with intensity $\lambda \nu(de) ds$, $\lambda\geq0$. This probabilistic representation provides a Monte Carlo (MC) estimator for $u(t,x)$, thereby serving as a reference solution for assessing the  INEUS accuracy on the nonlinear PIDE \eqref{nonlinear}. Here we study a $100$-dimensional case with $\varphi(x) = ||x||^2$ and Gaussian jumps $\nu = \mathcal{N}(\mu_J, \Sigma_J)$.
  \begin{figure}[h]
	\centering
    \vspace{-1mm}
	\begin{minipage}{.5\textwidth}
		\includegraphics[width=0.95\columnwidth, height=4cm]{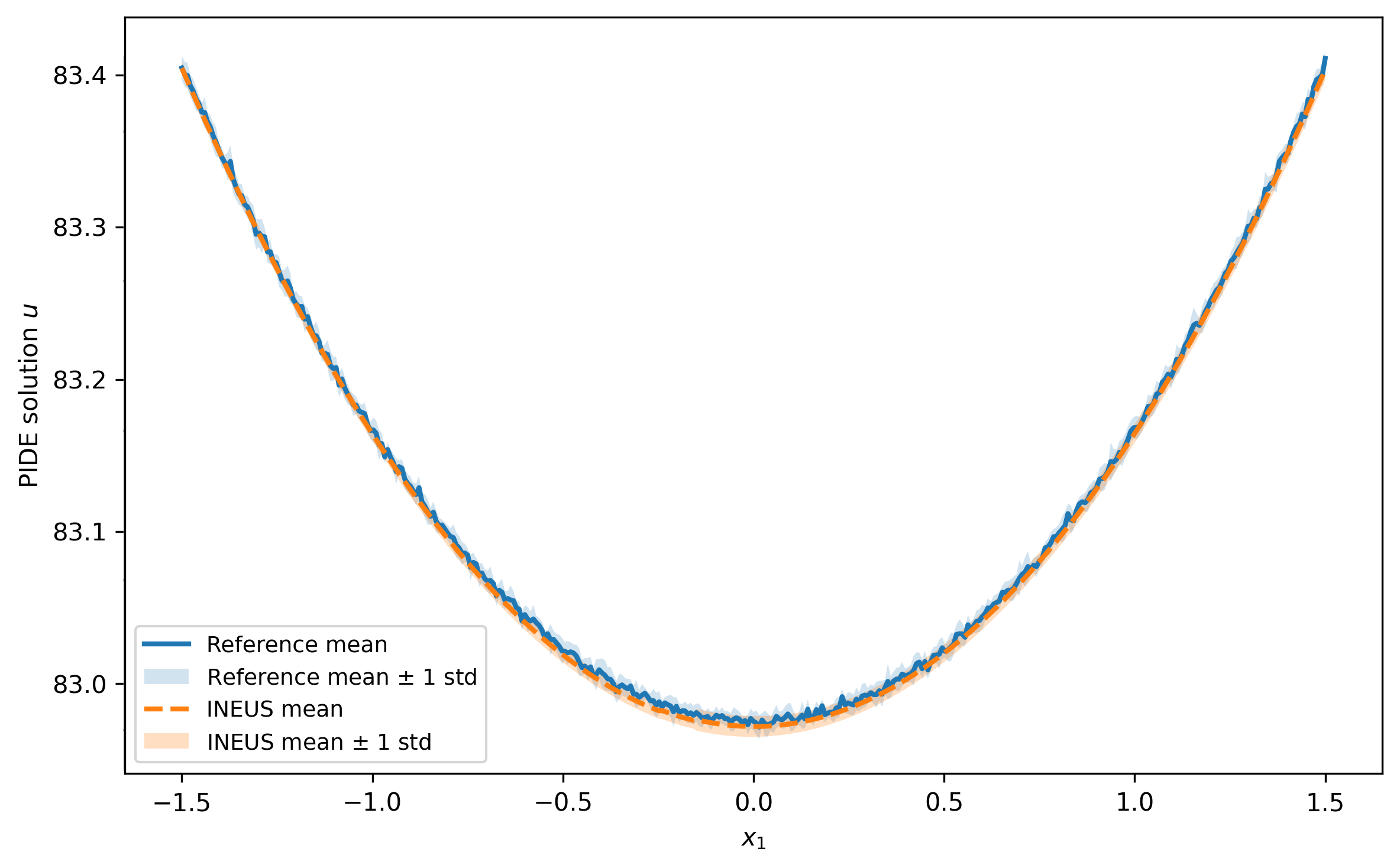}
		\label{fig:nonlinear_t0} \vspace{4mm}
	\end{minipage} \hspace{-4mm}
	\begin{minipage}{.5\textwidth} 	\vspace{-5mm}
\includegraphics[width=0.96\columnwidth, height=4cm]{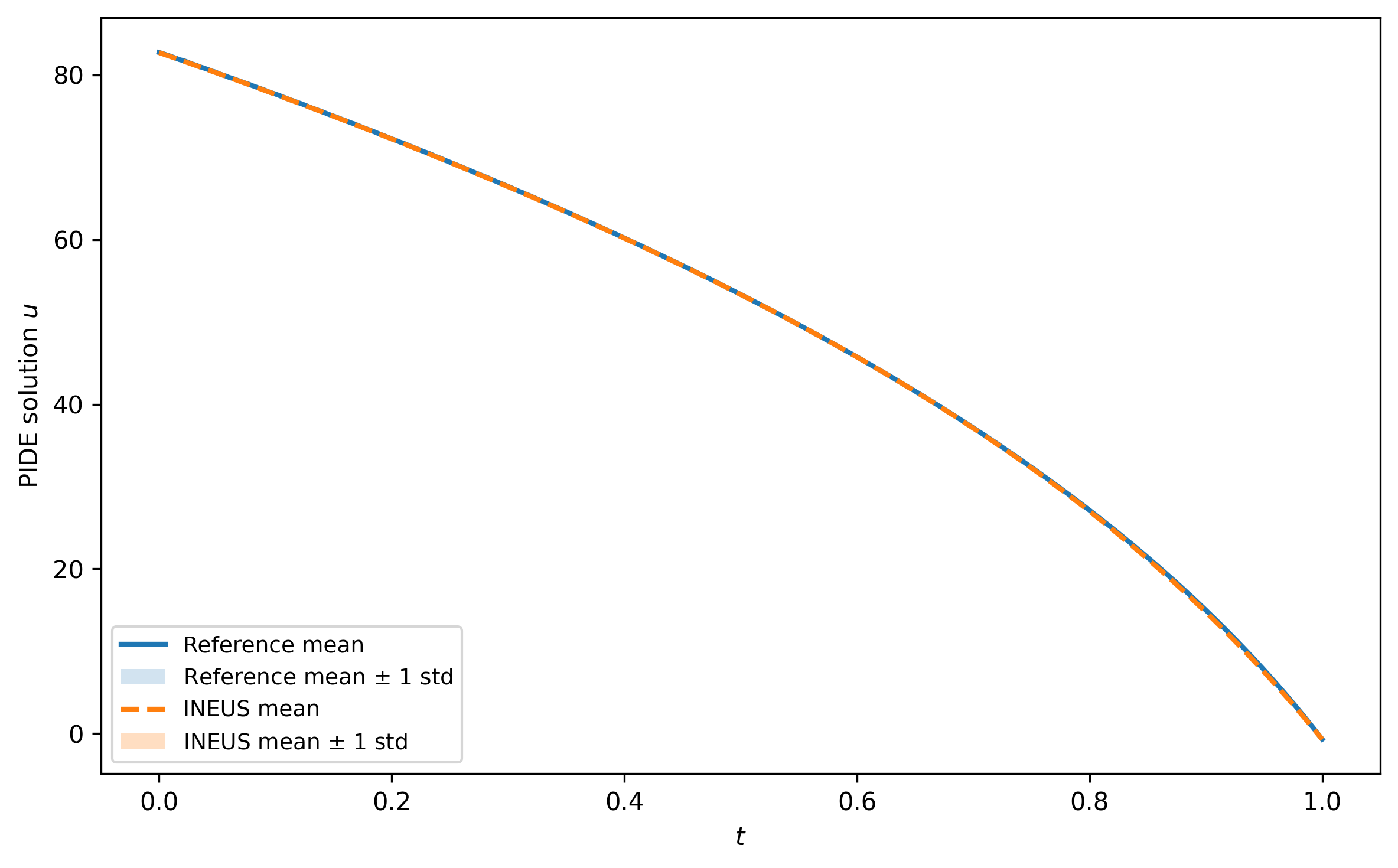}
		\label{fig:nonlinear_x0}
	\end{minipage}\vspace*{-4.5mm}
\caption{\small PIDE solution $u(0,x)$  for $x=(x_1,0, \dots, 0)$ with $x_1  \in [-1.5, 1.5]$ (left) and  $u(t,x)$ for  $t \in [0,1]$ and $x= \mathbf{0}_{100}$ (right) for a 100-dimensional nonlinear PIDE of the form \eqref{nonlinear}. Orange dotted lines: numerical results of INEUS with $\pm 1$ standard deviation given by orange shaded area. Blue lines: reference MC solution based on \eqref{repres1}\vspace{-6mm}.}
\label{fig: nonlinear} 
\end{figure} 
\begin{figure}[h!]
	\centering
	\subfigure{ \hspace{-1mm}
		\includegraphics[width=0.485\columnwidth, height=3.95cm]{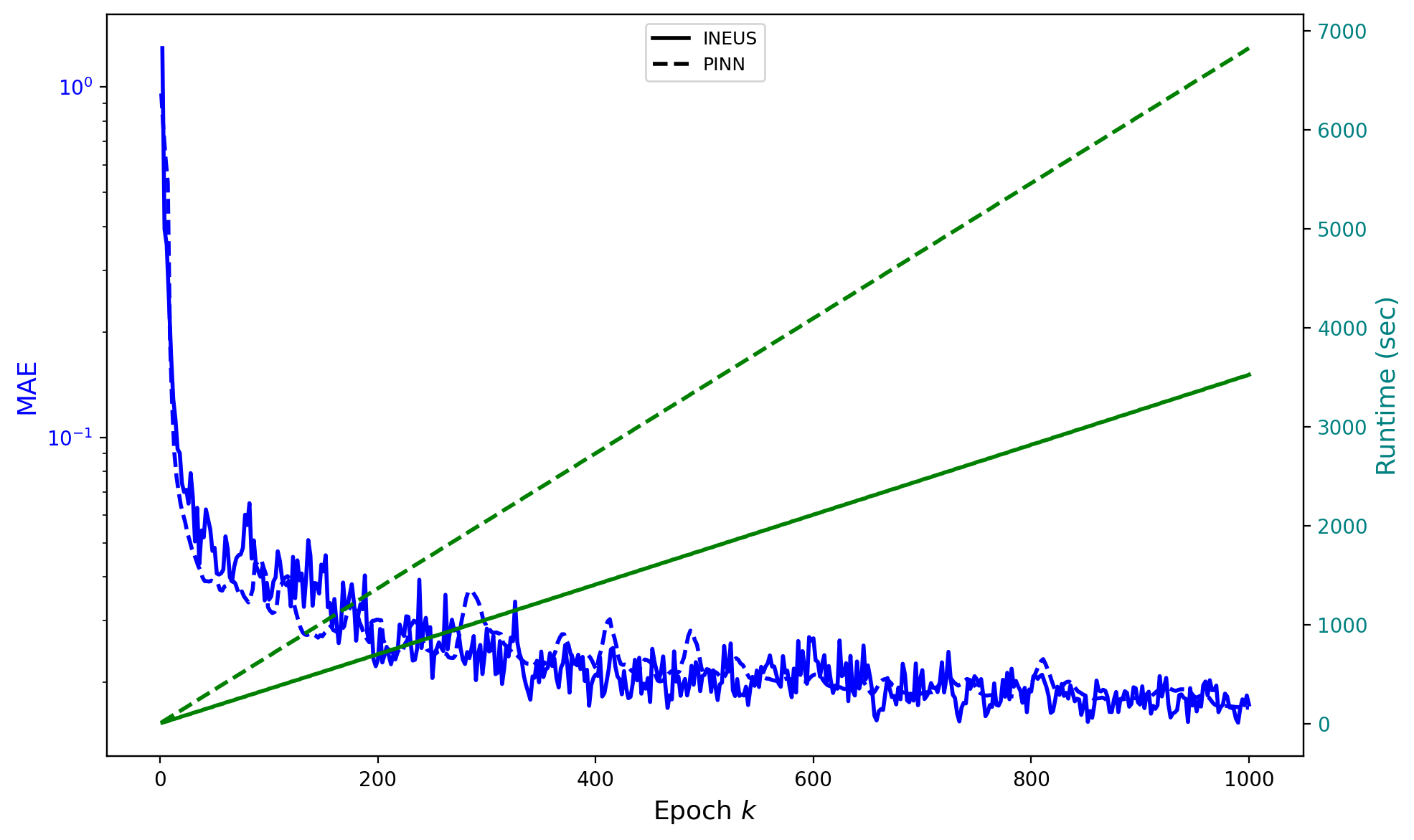}
		\label{fig:nonlinearvspinn}
	}
	 \hspace{-4mm}
	\subfigure{	\includegraphics[width=0.49\columnwidth, height=4cm]{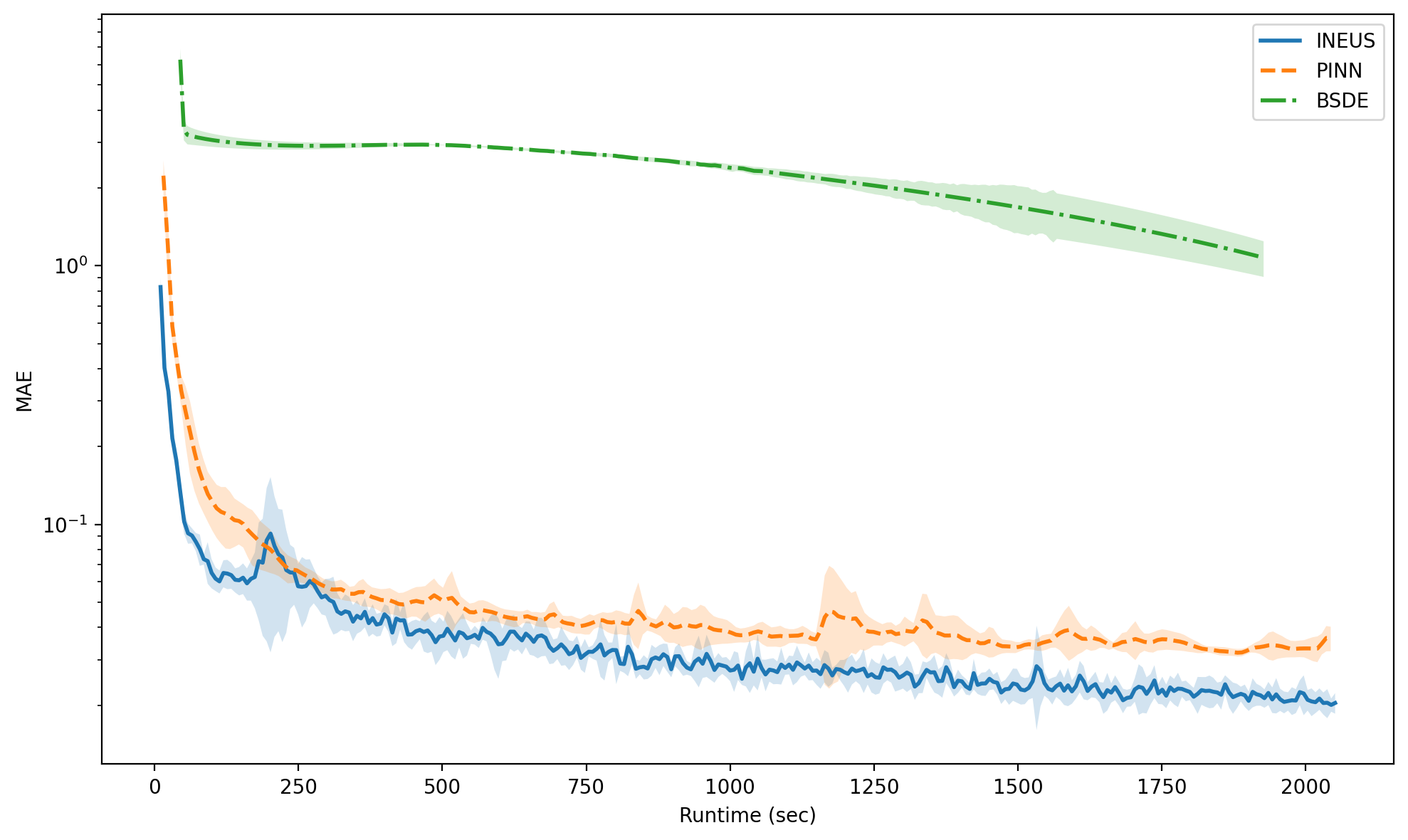}
	\label{fig:nonlinearvsall}
	} \vspace{-2.5mm}
	\caption{{\small Numerical solutions of a 10-dimensional nonlinear PIDE of the form \eqref{nonlinear}. Left panel: comparison of ${\rm MAE}$ (blue) and runtime in seconds (green) of INEUS (solid lines) and  PINN (dashed lines) as functions of the epoch $k$. Right panel: comparison of INEUS, PINN and deep BSDE methods as functions of runtime (seconds)\vspace{-3.4mm}.}} 
	\label{fig:nonlinearvs} 
\end{figure} 
\paragraph{Nonlinear Black--Scholes PIDE.}\hspace{-1mm}In our last example, we add jumps to the $d$-dimensional nonlinear Black--Scholes PDE with default risk from \cite{han2018solving}, yielding the following PIDE \vspace{-0.5mm}on $D_T = [0,T) \times \mathbb{R}^d_+$:
\begin{equation} \label{eq:default_risk}
	\partial_t u + \beta \, x \cdot \nabla_x u + \frac{1}{2}\hspace{-0.5mm} \operatorname{Tr}\!\big[\sigma^2 \operatorname{diag}(x)^2 \nabla^2_x u\big]
	- \big[(1-\delta) Q(u) + R\big] u + \lambda
    \mathbb{E}_{\nu}\!\left[
      u \bigl(t,\,x\odot e^{\mathbf{E}}\bigr)\hspace{-0.75mm}-\hspace{-0.5mm} u(t,x)
    \right]\hspace{-0.5mm} = 0,
\end{equation} with terminal condition $u(T,x) = \min_{i=1,\ldots,d} x_i$ and $Q(u)$ the piecewise-linear default intensity \eqref{eq:intensity}. Here, $\sigma,\beta,\delta,R,\lambda \in\mathbb{R}_+$, and $\odot$ denotes component-wise multiplication. The jump vector \(
\mathbf{E}=(E_1,\ldots,E_d)\) has distribution \(\nu\) and is specified by $
E_i
=
\mu_J
+
\sigma_J
\big(
\sqrt{\rho_J}\,Z_0
+
\sqrt{1-\rho_J}\,Z_i
\big), $
where \(Z_0,Z_1,\ldots,Z_d\) are independent standard normal random variables.
The upper row of Figure~\ref{fig:min_option} considers the diffusion-only case ($\lambda = 0$), for which a reference value $u(0,x_0) = 57.300$ is available from \cite{han2018solving}. The lower row of Figure~\ref{fig:min_option} reports the corresponding experiment with jumps, for which the reference value $u(0,x_0) = 55.810$ is computed
via the 
forward Picard scheme~\cite{bender2007forward}, 
based on the Feynman--Kac representation with $4\times10^5$ antithetic 
paths and $30$ Picard iterations. It can be seen that, as training progresses,  INEUS substantially reduces the variability and standard deviation of the prediction, while converging toward the correct solution.
Additional experiments for both the nonlinear Black--Scholes PIDE~\eqref{eq:default_risk} and a linear Black--Scholes variant~\eqref{Black_Scholes} (with  reference solution) are provided in Appendix~\ref{Appendix: additional}.

\begin{figure}[h]
    \centering

    % Upper row: PDE without jumps
    \begin{minipage}{.48\textwidth}
        \centering
        \includegraphics[width=0.85\linewidth]{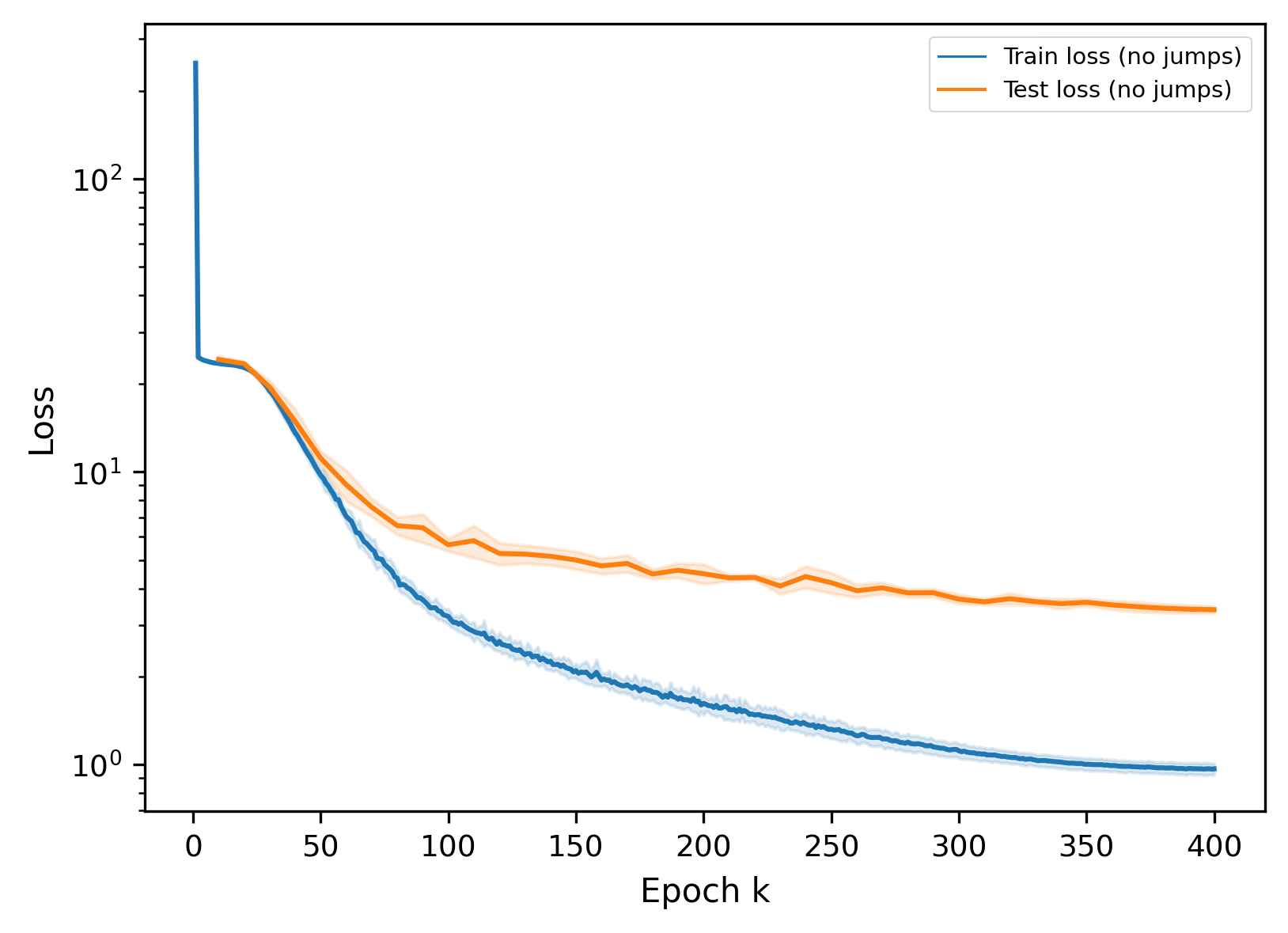}
    \end{minipage}\hfill
    \begin{minipage}{.48\textwidth}
        \centering
        \includegraphics[width=0.85\linewidth]{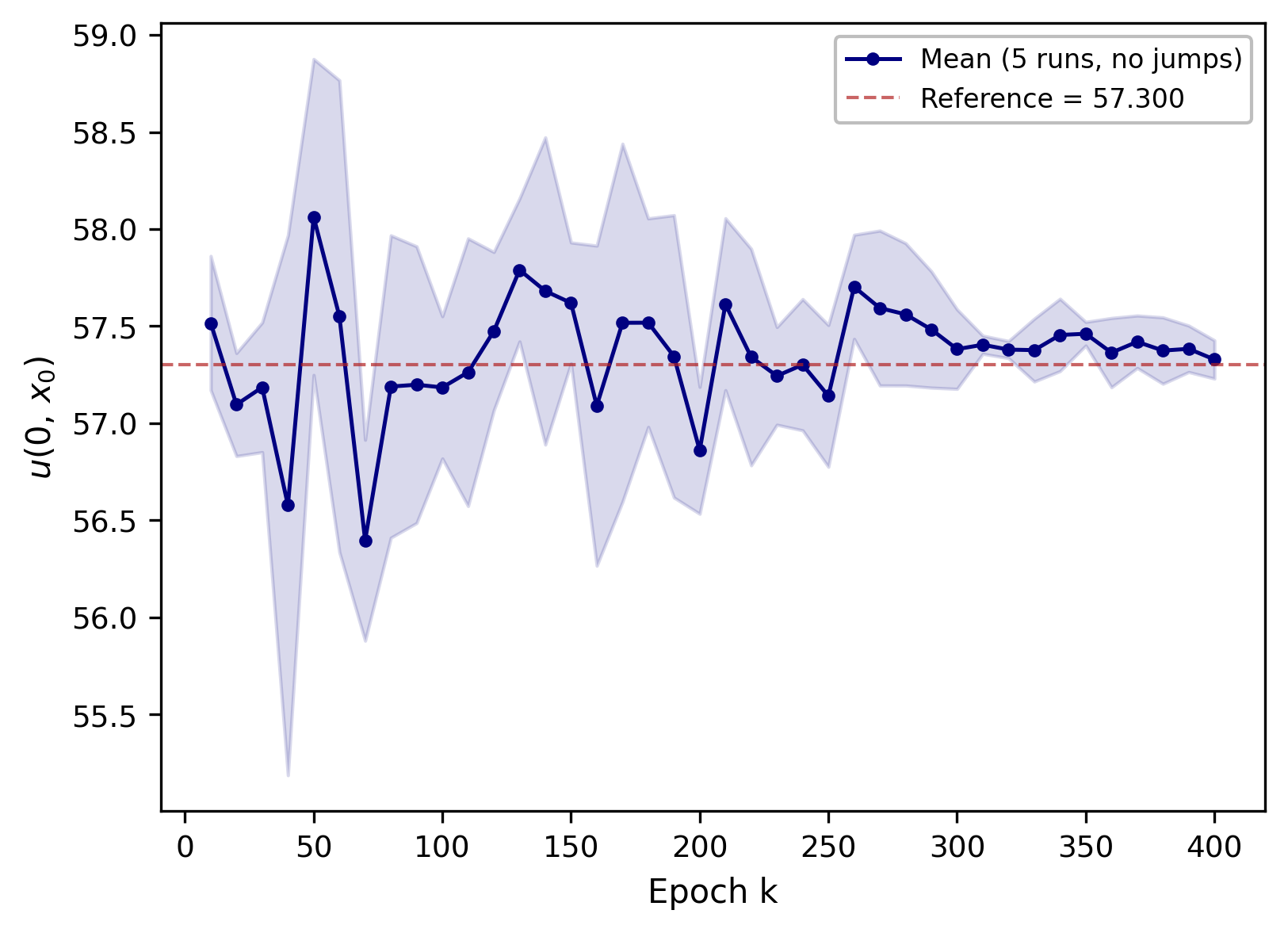}
    \end{minipage}

    \vspace{2mm}

    % Lower row: PIDE with jumps
    \begin{minipage}{.48\textwidth}
        \centering
        \includegraphics[width=0.85\linewidth]{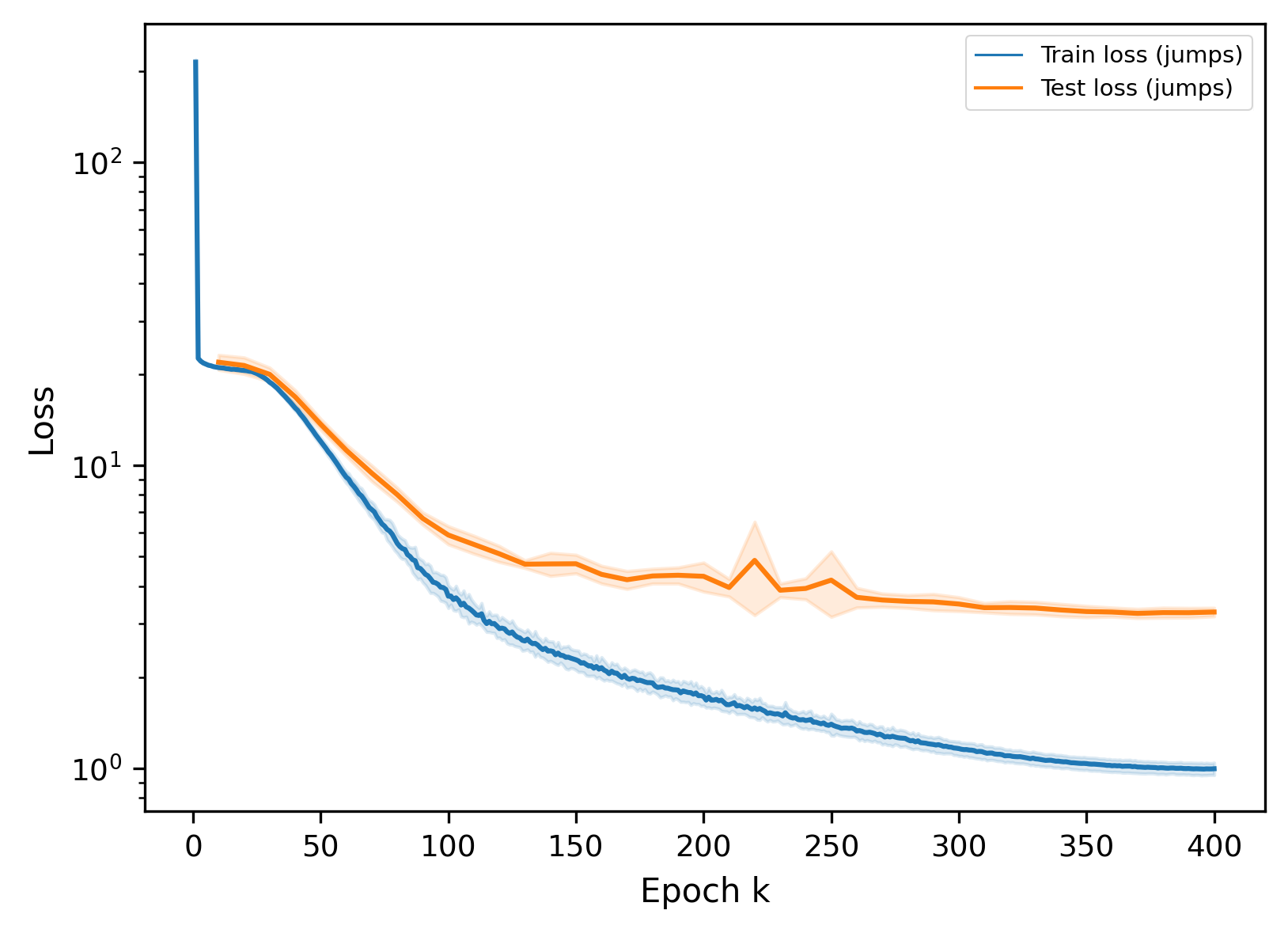}
    \end{minipage}\hfill
    \begin{minipage}{.48\textwidth}
        \centering
        \includegraphics[width=0.85\linewidth]{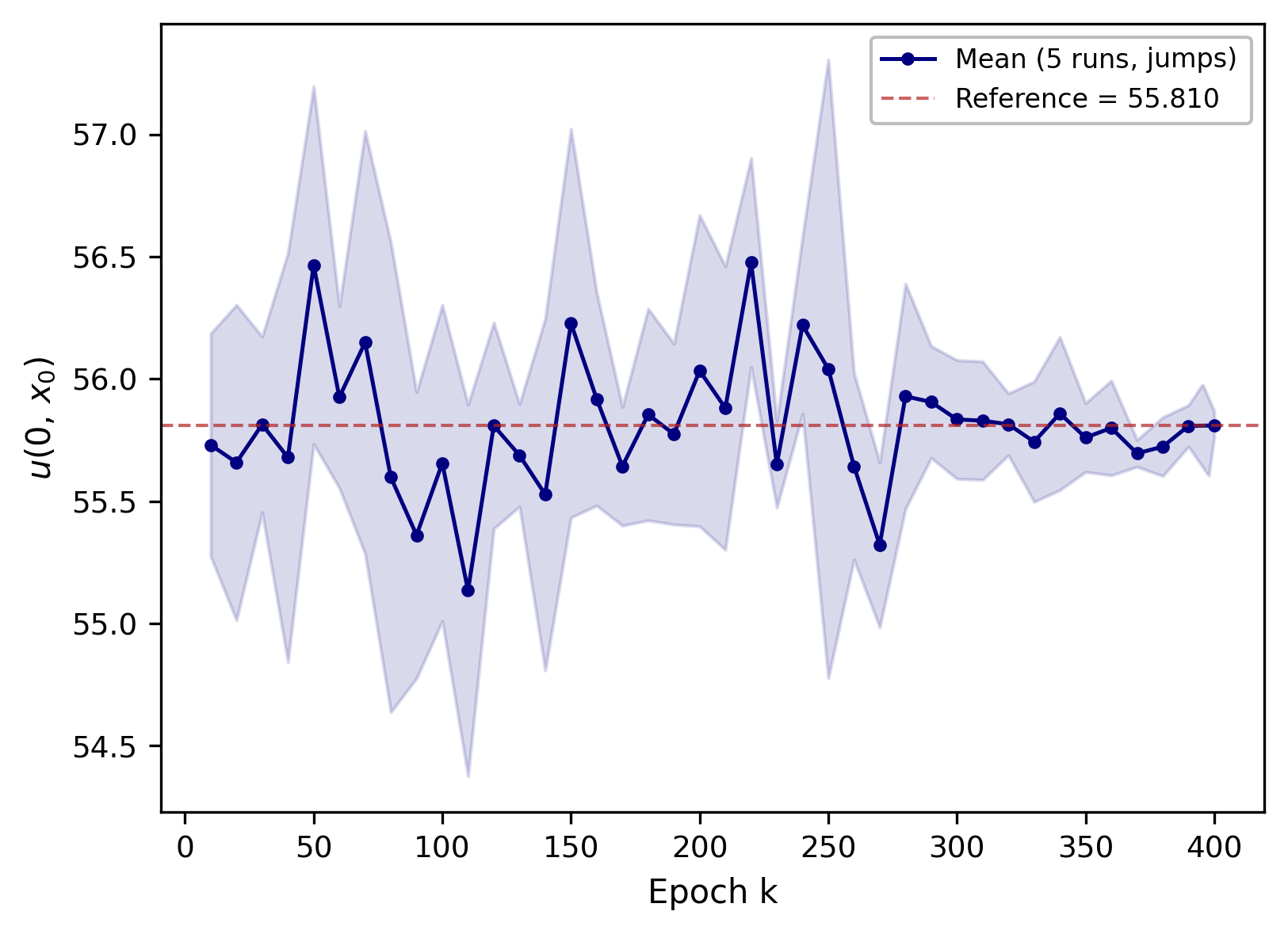}
    \end{minipage}

    \vspace*{-1mm}
    \caption{\small INEUS for the 100-dimensional nonlinear Black--Scholes equation \eqref{eq:default_risk}, without jumps  (upper row) and with jumps (lower row). The left panels show the training and test losses over $400$ epochs, while the right panels show the prediction trajectory of $u_\theta(0,x_0)$ with $x_0=(100,\ldots, 100)$. The no-jump estimate is compared with the reference value $57.300$ of the deep-BSDE method, while the jump estimate is compared with the  reference value $55.810$ of the Picard iteration.\vspace{-1mm}}
    \label{fig:min_option}
\end{figure} 

 \section{Conclusions, limitations and future work}

We have proposed INEUS, a recursive learning framework for high-dimensional nonlinear PIDEs with jumps. INEUS combines hard-constrained PINNs with an expectation-free  target, thereby avoiding both explicit jump integration and the expensive higher-order derivatives required by standard residual-based PINN training. This yields a meshfree global solver that remains effective in high dimension and is naturally adapted to jump processes. For linear PIDEs, we established a contraction-based theoretical foundation by linking the recursive updates to both the Feynman--Kac semigroup and a relaxed version.  This shows that INEUS  approximates a class of fixed-point operators with provable convergence properties, while the relaxed formulation provides an additional mechanism to improve training stability. Our numerical results  show that INEUS achieves accurate global approximations at a lower computational cost than standard residual-based PINNs and deep BSDE approaches for high-dimensional PDEs and PIDEs. 

Several directions remain open. Further theoretical work is needed to extend the convergence analysis to nonlinear PIDEs and infinite-activity Lévy operators. These extensions  would require developing a  substantially different mathematical machinery from the one used in the paper  to handle nonlinear evolution and to  control the sampling, approximation and optimization errors arising in the practical algorithm. % Finite mini-batch optimization, neural-network approximation, and periodic regression back into the hPINN class also introduce optimization and approximation errors that are not yet quantified.
On the algorithmic side, adaptive relaxation strategies and improved sampling schemes may further enhance stability and efficiency.
\section*{Acknowledgments and disclosures}
The authors declare that they have no conflicts of interest and thank the reviewers for their helpful comments. The work was supported by the Swiss National Science Foundation  under Grant No.\ 10003723. The code accompanying this paper is available at \url{https://github.com/jdupret97/INEUS}.
 %mention that our scheme also applies directly to PDEs and can be used to solve stochastic control problems. Limitation: stability issues especially for nonlinear PDEs with non-smooth terminal conditions. But $\mathcal{T}_{h,\alpha}$ helps mitigate this.
 
	{\small
		\bibliographystyle{abbrvnat}      % mathematics and physical sciences
		\bibliography{bibfpinn}   % name your BibTeX data base
		
	}

%%%%%%%%%%%%%%%%%%%%%%%%%%%%%%%%%%%%%%%%%%%%%%%%%%%%%%%%%%%%%%%%%%%%%%%%%%%%%%%
%%%%%%%%%%%%%%%%%%%%%%%%%%%%%%%%%%%%%%%%%%%%%%%%%%%%%%%%%%%%%%%%%%%%%%%%%%%%%%%
% APPENDIX
%%%%%%%%%%%%%%%%%%%%%%%%%%%%%%%%%%%%%%%%%%%%%%%%%%%%%%%%%%%%%%%%%%%%%%%%%%%%%%%
%%%%%%%%%%%%%%%%%%%%%%%%%%%%%%%%%%%%%%%%%%%%%%%%%%%%%%%%%%%%%%%%%%%%%%%%%%%%%%%
\newpage
\appendix
\onecolumn
\section{Algorithm and network architecture} \label{app:algo}
\subsection{INEUS algorithm}
We now describe the practical implementation of the full scheme in Theorem~\ref{Prop: conv}, which approximates the relaxed contracting operator $\mathcal{T}_{h,\alpha}$ (or $\mathcal{T}_h$ when $\alpha = h$). The algorithm is organized into outer epochs indexed by $k=1,2,\ldots$, inner recursion steps indexed by $i=0,\ldots,n-1$, and gradient updates indexed by $j=0,\ldots,N-1$. To simplify notation, we suppress the dependence of the weights on the inner-step index $i$ and write $\theta^{(k)}_j$ instead of the more explicit $\theta^{(k)}_{i,j}$.
	\begin{algorithm}[h!]
		\caption{Iterative Neural Solver for PIDEs (INEUS)}
		\label{Algo1}
		\begin{algorithmic}[1]
			\REQUIRE learning rate $\eta>0$, relaxation parameter $\alpha>0$, scale $h>0$, number of inner steps $n\ge 1$, number of gradient steps $N$, batch size $M$.
\STATE Initialize admissible weights $\theta^{(0)}$ for the hPINN $u_{\theta}$ satisfying boundary conditions \eqref{termgen}--\eqref{boundgen}.
\STATE Set $\xi \gets h/n$ and \(k\gets 0\).
			
			\REPEAT

				\STATE Set $\theta_{0}^{(k)} \gets \theta^{(k)}$.  \hfill\(\triangleright\) Initialize the frozen target network
			\FOR{$i=0$ to $n-1$}

			\FOR{$j=0$ to $N-1$}
			 \STATE Sample $\{(y_m,e_m)\}_{m=1}^M$ i.i.d.\ from $\mu\otimes \nu$ on $D_T\times E$.\vspace{0.5mm}
			 \STATE Form the empirical loss \vspace{-2mm}
           \begin{equation*}
             \widehat{\mathscr L}^{(k)}(\theta \, ; \xi)
            :=
            \frac1M\sum_{m=1}^M
            \Big(
            u_\theta(y_m)-\mathcal{G}_\xi(y_m,e_m,u_{\theta^{(k)}_{0}})
            \Big)^2 . \vspace{-5mm}\end{equation*} 
           \STATE Update
            \[
            \theta_{j+1}^{(k)}
            \gets
            \theta_{j}^{(k)}
            -
            \eta \nabla_\theta \widehat{\mathscr L}^{(k)}(\theta_{j}^{(k)} \, ; \xi),
            \]
            or equivalently,
         \begin{equation} \label{eq:algo1}
            \theta_{j+1}^{(k)}
            \gets
            \theta_{j}^{(k)}
            -
            \frac{2\eta}{M}\sum_{m=1}^{M}
            \Big(
            u_{\theta_{j}^{(k)}}(y_m)
            -
            \mathcal{G}_\xi(y_m,e_m,u_{\theta_{0}^{(k)}})
            \Big) \, 
            \nabla_\theta u_{\theta_{j}^{(k)}}(y_m).         \end{equation}
        \ENDFOR \vspace{0.5mm}
			
			\STATE Set $\theta_{0}^{(k)} \gets \theta_{N}^{(k)}$. \hfill\(\triangleright\) Freeze target for next inner step
			\ENDFOR
			
			\STATE Apply the Polyak update
			\[
			u^{(k+1)}
			\gets
			\left(1-\frac{\alpha}{h}\right)u_{\theta^{(k)}}
			+
			\frac{\alpha}{h}\,u_{\theta^{(k)}_N}.
			\]
			  \STATE Choose parameters \(\theta^{(k+1)}\) such that \(u_{\theta^{(k+1)}} \approx u^{(k+1)}\),
    by regression onto the hPINN class, see Equation~\eqref{eq:distillation}.
			%\STATE Choose $\theta^{(k+1)}$ such that $u_{\theta^{(k+1)}} = u^{(k+1)}$.
			\STATE $k \gets k+1$.
			\UNTIL{some convergence criterion is satisfied. %$\max_{(t,x) \in B} \Big| V_{\theta^{(k)}}\hspace{-0.4mm}(t,x)- V_{\theta^{(k-1)}}\hspace{-0.4mm}(t,x) \Big|  \leq \varepsilon$. 
			}
			\STATE \textbf{return} $u_{\theta^{(k)}}$ and set $k_* \leftarrow k$.
		\end{algorithmic} 
	\end{algorithm} 
    \\
We emphasize that after the Polyak update, the relaxed function $u^{(k+1)}$ is generally not itself represented exactly by the chosen hPINN parametrization. We therefore determine new parameters $\theta^{(k+1)}$ by projecting this updated target back onto the hPINN class, that is by choosing  $u_{\theta^{(k+1)}}$ to approximate $u^{(k+1)}$ in a least-squares sense over sampled points $(t,x)$. Since performing this regression at every epoch can be computationally costly, we instead carry it out only once every block of $K$ epochs. Unrolling the Polyak recursion over this block, starting from epoch $k$, yields the target
\begin{equation}\label{eq: target}u^{(k + K)} = \Big(1-\frac{\alpha}{h}\Big)^K u_{\theta^{(k)}} + \sum_{j=1}^K \frac{\alpha}{h}\Big(1-\frac{\alpha}{h}\Big)^{K-j} u_{\theta_N^{(k+j-1)}} .\end{equation}
Thus, at the end of the $K$-block, the new hPINN $u_{\theta^{(k+K)}}$ is obtained by regressing this weighted combination  $u^{(k+K)}$ of the block-start network and the intermediate candidate networks back onto the admissible hPINN class, that is,
\begin{equation} \label{eq:distillation}
\min_{\theta}
\;
\mathbb{E}_{(t,x)\sim \mu}
\Big[
u_\theta(t,x)
-
\Big(
(1-\alpha/h)^K u_{\theta^{(k)}}(t,x)
+
\sum_{j=1}^K
\frac{\alpha}{h}(1-\alpha/h)^{K-j}
u_{\theta_N^{(k+j-1)}}(t,x)
\Big)
\Big]^2 ,
\end{equation}
where $\mu$ is the same space-time sampling distribution used for training (see \ref{dgm} below). For $K=1$, this reduces exactly to Algorithm~\ref{Algo1}. The student network has the same hPINN architecture as the main model and is initialized with the last candidate network in the block. In the numerical experiments of Section~\ref{Section: numeric}, we use $K=10$. The distillation is performed for $200$ Adam steps with learning rate $10^{-3}$, after which the distilled network replaces the current model and training resumes.
\subsection{Hyperparameters} 
\paragraph{DGM network.}
\label{dgm}
\begin{figure}[h!] 
\hspace{-0.2cm}	\includegraphics[scale=0.36]{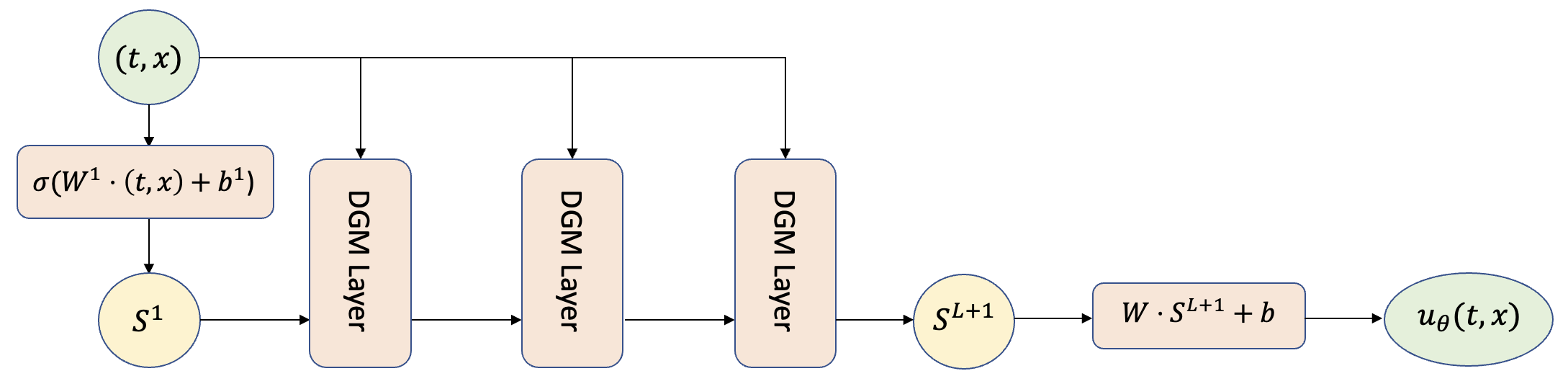} 
\caption{\small DGM architecture of the neural network $u_\theta$ with $L=3$ (\textit{i.e.}\ 4 hidden layers).}
	\label{U_S process} 
\end{figure}
Each DGM layer of the network $u_\theta$ in Figure \ref{U_S process} is of the following form
$$
\begin{aligned}
	S^1 & =\sigma \left(W^1 \cdot (t,x)+b^1\right), \\
	Z^{\ell} & =\sigma \left(U^{z, \ell} \cdot (t,x)+W^{z, \ell} \cdot S^{\ell}+b^{z, \ell}\right), \quad \ell=1, \ldots, L, \\
	G^{\ell} & =\sigma\left(U^{g, \ell} \cdot (t,x)+W^{g, \ell} \cdot S^\ell+b^{g, \ell}\right), \quad \ell=1, \ldots, L, \\
	R^{\ell} & =\sigma \left(U^{r, \ell} \cdot (t,x)+W^{r, \ell} \cdot S^{\ell}+b^{r, \ell}\right), \quad \ell=1, \ldots, L, \\
	H^{\ell} & =\sigma \left(U^{h, \ell} \cdot (t,x)+W^{h, \ell} \cdot \left(S^{\ell} \odot R^{\ell}\right)+b^{h, \ell}\right), \quad \ell=1, \ldots, L, \\
	S^{\ell+1} & =\left(1-G^{\ell}\right) \odot H^{\ell}+Z^{\ell} \odot S^{\ell}, \quad \ell=1, \ldots, L, \\
	u_\theta(t, x) & = A(t,x) + B(t,x) \left(W \cdot S^{L+1}+b \right), 
\end{aligned}
$$
where the number of hidden layers is $L+1$, $\cdot$ denotes matrix multiplication, $\odot$ element-wise multiplication, $A:[0,T]\times \R^d \to \R$ encodes the exact terminal and boundary conditions and $B :[0,T]\times \R^d \to \R$ is a smooth distance function vanishing for $t=T$ and $x\in D^c$, see \eqref{hPINN}. The DGM parameters of the value network are
\begin{align*}
		\theta = \Big\{ \hspace{-0.5mm}W^1, b^1, \left(U^{z, \ell}, W^{z, \ell}, b^{z, \ell}\right)_{\ell=1}^L,&\left(U^{g, \ell}, W^{g, \ell}, b^{g, \ell}\right)_{\ell=1}^L,
		\\
	&\left(U^{r, \ell}, W^{r, \ell}, b^{r, \ell}\right)_{\ell=1}^L, \left(U^{h, \ell}, W^{h, \ell}, b^{h, \ell}\right)_{\ell=1}^L , W, b\Big\} .
\end{align*}The number of units in each layer is $n_{\rm{hid}}$ and  $\sigma: \mathbb{R}^{n_{\rm{hid}}} \rightarrow \mathbb{R}^{n_{\rm{hid}}}$ is a twice-differentiable element-wise nonlinearity. Throughout the numerical examples of the paper, we use $L=3$ with $n_{\rm{hid}} = 50$ neurons in 
each of the DGM layers, and we take $\tanh$ for the 
activation function $\sigma$. \vspace{-1mm}

\paragraph{Training setup.} In Algorithm \ref{Algo1},  INEUS is implemented with batch size $M = 2048$, number of inner steps $n = 20$, number of gradient steps $N = 32$, and a maximum number of epochs $k_* = 1000$ as the stopping criterion\footnote{We use a fixed epoch budget
to facilitate comparison and reproducibility, but any other terminal criterion could be used equivalently, for instance based on the variation of the loss function $\widehat{\mathscr{L}}^{(k)}$ or in successive network outputs $u_{\theta^{(k)}}$.}. The network parameters are updated using Adam \citep{kingma2014adam} with a constant learning rate $\eta = 0.0005$.  Table~\ref{tab:hyperparameter-sensitivity} examines the
sensitivity of INEUS to the learning rate $\eta$, the number of
gradient steps $N$ and the batch size $M$ for the linear PIDE~\ref{linear_PDE}
in $d=50$ dimensions. For each configuration, we report
the mean and standard deviation of the mean absolute error
(MAE) over 10 runs, together with the computation time. Overall, we observe that INEUS is fairly robust with respect to these hyperparameters, while emphasizing that their effects are strongly interdependent, making it difficult to isolate the influence of a single parameter. \vspace{-2mm}
\begin{table}[H]
    \centering
    \small
    \setlength{\tabcolsep}{4pt}
  \caption{Sensitivity to the learning rate (top left),
    number of gradient steps (top right), and number of samples
    (bottom). \\ }
    \begin{minipage}[t]{0.48\linewidth}
        \centering
        \begin{tabular}{@{}rrrr@{}}
            \toprule
            \shortstack{Learning\\rate $\eta$}
            & \shortstack{Mean\\MAE}
            & \shortstack{Std.\ dev.\\MAE}
            & \shortstack{Time\\(sec.)} \\
            \midrule
            0.00005 & 0.0246 & 0.0025 & 6\,364 \\
            0.0001  & 0.0123 & 0.0019 & 6\,394 \\
            0.0005  & 0.0045 & 0.0006 & 6\,570 \\
            0.001   & 0.0091 & 0.0021 & 6\,633 \\
            \bottomrule
        \end{tabular}
    \end{minipage}
    \hfill
    \begin{minipage}[t]{0.48\linewidth}
        \centering
        \begin{tabular}{@{}rrrr@{}}
            \toprule
            \shortstack{Gradient\\steps $N$}
            & \shortstack{Mean\\MAE}
            & \shortstack{Std.\ dev.\\MAE}
            & \shortstack{Time\\(sec.)} \\
            \midrule
             8 & 0.0116 & 0.0041 &  2\,646 \\
            16 & 0.0076 & 0.0011 &  3\,977 \\
            32 & 0.0045 & 0.0006 &  6\,570 \\
            64 & 0.0053 & 0.0012 & 11\,768 \\
            \bottomrule
        \end{tabular}
    \end{minipage}

    \par\bigskip

    \begin{tabular}{@{}rrrr@{}}
        \toprule
        Batch size $M$ & Mean MAE & Std.\ dev.\ MAE & Time (sec.) \\
        \midrule
           256 & 0.0677 & 0.01143 & 6\,139 \\
        1\,024 & 0.0081 & 0.0014  & 6\,402 \\
        2\,048 & 0.0045 & 0.0006  & 6\,570 \\
        4\,096 & 0.0039 & 0.0004  & 6\,986 \\
        8\,192 & 0.0038 & 0.0002  & 7\,338 \\
        \bottomrule
    \end{tabular}

    \label{tab:hyperparameter-sensitivity}
\end{table}
 
 \paragraph{Choice of $\alpha$ and $h$.}
The relaxation parameter $\alpha$ and scaling parameter $h$ should be chosen within the contraction range $ 0<\alpha<2h/(1+e^{-c_0 h})$. For fixed $h$, the contraction modulus
$|1-\frac{\alpha}{h}|+\frac{\alpha}{h} e^{-c_0 h}
$ is minimized at $\alpha=h$ and equal to $e^{-c_0 h}$. In practice, $\alpha$ should therefore be chosen close to, but not necessarily equal to, $h$: $\alpha = h$ gives the strongest contraction, while $\alpha < h$ provides a more damped update that can improve numerical stability when recursive training  introduces variability. The choice of $h$ involves a separate trade-off:  larger $h$ reduces $e^{-c_0 h}$ and strengthens the contraction, but Theorem \ref{Prop: conv} shows that the update approximation error grows as $
\mathcal{O}\!\left(\alpha h / n\right)+\mathcal{O}(\alpha h^{n}).
$
Hence,  $h$ should balance stronger contraction against the accuracy of the recursive approximation. That is why, in most experiments, we set $\alpha = 0.4$ and $h = 0.5$, which provides a good trade-off between convergence speed, numerical stability and accuracy of the iterative updates in INEUS,  although problem-specific fine-tuning may be beneficial. Alternatively,  one may also consider an epoch-dependent relaxation factor $\alpha_k >0$ in \eqref{eq:relaxed-update-clean} to adapt the size of the outer update throughout training. Choosing $\alpha_k<h$ indeed dampens the update and improves stability by avoiding overly large changes in the networks $\| u_{\theta^{(k+1)}}-  u_{\theta^{(k)}} \|$, while choosing $\alpha_k > h$ may accelerate convergence in flatter regions of the optimization landscape. These choices are justified by the following sensitivity analysis on the linear PIDE \eqref{linear_PDE} in $50$ dimensions, although their effects are again strongly interdependent. \vspace{-3mm}
\begin{table}[H]
    \centering
    \small
    \setlength{\tabcolsep}{4pt}
    \caption{Sensitivity of INEUS to the relaxation parameter
    (top left), scale parameter (top right), and number of
    inner steps (bottom). \\ }
    \begin{minipage}[t]{0.48\linewidth}
        \centering
        \begin{tabular}{@{}rrrr@{}}
            \toprule
            \shortstack{Relaxation\\parameter $\alpha$}
            & \shortstack{Mean\\MAE}
            & \shortstack{Std.\ dev.\\MAE}
            & \shortstack{Time\\(sec.)} \\
            \midrule
            0.05 & 0.0029 & 0.0007 & 6\,408 \\
            0.1 & 0.0034 & 0.0008 & 6\,623 \\
            0.25 & 0.0029 & 0.0019 & 6\,599 \\
            0.4 & 0.0045 & 0.0006 & 6\,570 \\
            0.5   & 0.0061 & 0.0018 & 6\,299 \\
            \bottomrule
        \end{tabular}
    \end{minipage}
    \hfill
    \begin{minipage}[t]{0.48\linewidth}
        \centering
        \begin{tabular}{@{}rrrr@{}}
            \toprule
            \shortstack{Scale\\parameter $h$}
            & \shortstack{Mean\\MAE}
            & \shortstack{Std.\ dev.\\MAE}
            & \shortstack{Time\\(sec.)} \\
            \midrule
            0.05 & 0.0081 & 0.0017 & 6\,325 \\
            0.1 & 0.0059 & 0.0003 & 6\,559 \\
            0.25  & 0.0056 & 0.0010 & 6\,566 \\
            0.5    & 0.0045 & 0.0006 & 6\,570 \\
            1    & 0.0066 & 0.0042 & 6\,551 \\
            \bottomrule
        \end{tabular}
    \end{minipage}

    \par\bigskip

    \begin{tabular}{@{}rrrr@{}}
        \toprule
        Inner steps $n$ & Mean MAE & Std.\ dev.\ MAE & Time (sec.) \\
        \midrule
         2 & 0.1408 & 0.0633 &  1\,366 \\
         5 & 0.0149 & 0.0021 &  2\,063 \\
        10 & 0.0131 & 0.0085 &  3\,223 \\
        20 & 0.0045 & 0.0006 &  6\,570 \\
        50 & 0.0081 & 0.0024 & 11\,829 \\
        \bottomrule
    \end{tabular}

    \label{tab:relaxation-scale-inner-steps}
\end{table}

\paragraph{Sampling measure.} Another crucial component is the choice of the sampling measure $\mu$, which plays an important role in the quality of the approximation. We therefore use the residual-adaptive distribution (RAD) introduced by \citet{wu2023comprehensive}, which has been shown to improve the performance of PINN-style methods. We refer to \citet{wu2023comprehensive} for an ablation study comparing RAD with standard uniform and Gaussian sampling strategies. Only for the Black--Scholes PIDEs \eqref{eq:default_risk}--\eqref{Black_Scholes}, we use path-based sampling, which generates collocation points along geometric Brownian motion trajectories, to accelerate convergence.
\subsection{Baselines}  \label{app:baselines}
In the main text, we compared INEUS with a standard PINN/DGM baseline \cite{sirignano2018dgm} and a classical deep BSDE method with jumps \cite{andersson2025deep}. For all methods, the reported $\mathrm{MAE}$ is computed on the same fixed test set of $\mathcal{M}=2000$ independently sampled space-time points. The PINN baseline uses the same network architecture, sampling domain, optimizer, and training hyperparameters as INEUS. The only difference is that it is trained by minimizing the standard squared PIDE residual directly and the nonlocal jump operator $\mathcal{I}[u]$ is approximated by Monte Carlo using 64 jump samples. One reported PINN or INEUS epoch then corresponds to $nN=640$ optimizer updates. For the deep BSDE baseline, we use the classical local formulation for jump-diffusions. Since this method is intrinsically local and parameterizes a single initial value $u(t_0,x_0)$, it does not directly produce the global value function required for comparison with INEUS. We therefore sample multiple initial training points $(t_0,x_0)\sim\mu$, train independent local BSDE models at those points, and then fit a supervised global surrogate $u_\theta(t,x)$ to the resulting  local solution labels. The global test MAE is computed by evaluating this surrogate on the fixed test set. Each local BSDE model is trained for 10,000 optimizer steps with batch size 1024, Adam learning rate $3\times 10^{-4}$, gradient clipping norm 1.0, 25 time steps, and 64 Monte Carlo samples for the jump expectation. All methods are implemented in TensorFlow/Keras and run with GPU acceleration on an NVIDIA RTX 4090.

To avoid repeatedly solving the above independent deep BSDE problem for each initial condition $(t_0,x_0)$ and subsequently fitting a global surrogate, we also consider a global BSDE-based alternative, namely the FBSNN method \cite{Raissi2024} and its extension to PIDEs, FBSJNN \cite{ye2024fbsjnn}. These methods directly parameterize the PIDE solution through $
Y_t^{t,x}=u_\theta(t,x)$,
thereby allowing the initial pair $(t_0,x_0)$ to be sampled during training and a single network to approximate the solution over the full space-time domain.  The first table below compares INEUS with this global FBSJNN extension of the deep BSDE method for the 50-dimensional linear PIDE \eqref{linear_PDE}. To ensure a fair comparison, accuracy is reported as a function of runtime (in seconds). Performance is evaluated using the three following metrics: (i) the global MAE, computed on a fixed space--time test set of $2{,}000$ observations; (ii) the trajectory MAE, computed over the states visited along $256$ simulated trajectories on a fixed time grid; and (iii) the absolute local error at the initial condition $(t_0,x_0) = (0, \bm{0}_{50})$. Across these metrics and runtimes, INEUS consistently and substantially outperforms FBSJNN. In particular, the FBSJNN method does not appear to converge when learning the solution over the full space-time domain. \vspace{-3mm}
\begin{table}[H]
\centering
\caption{Comparison of INEUS and the global FBSJNN baseline on the 50-dimensional linear PIDE \eqref{linear_PDE} at matched wall-clock runtimes.\\}
\label{tab:fbsjnn_comparison}
\begin{tabular}{lccc}
\toprule
Method & Global MAE & Trajectory MAE & Local error at $(0, \bm{0}_{50})$ \\
\midrule
INEUS (100 sec.)                  & 0.031 & 0.067 & 0.068 \\
FBSJNN (100 sec.)       & 0.081 & 0.083 & 0.377 \\
\midrule
INEUS (500 sec.)                  & 0.007 & 0.021 & 0.004 \\
FBSJNN (500 sec.)       & 0.116 & 0.077 & 0.139 \\
\midrule
INEUS (1000 sec.)                 & 0.005 & 0.016 & 0.001 \\
FBSJNN (1000 sec.)      & 0.096 & 0.071 & 0.016 \\
\bottomrule
\end{tabular}
\end{table}

We finally provide local comparisons with temporal difference learning (TDL) \citep{lu2024temporal} and deep splitting extensions \citep{frey2022deep} for PIDEs, using the publicly available implementations released by the respective authors. Extending these methods to global solvers over the full space-time domain would require substantial modifications to their existing implementations  and we therefore restrict the comparison to the local error at the initial condition $(t_0,x_0)=  (0, \bm{0}_{50})$. Nevertheless, the results show that INEUS achieves competitive accuracy at the initial point $(0, x_0)$, despite simultaneously learning the solution over the full space-time domain. \vspace{-3mm}

\begin{table}[H]
\centering
\caption{Local comparison of INEUS with TDL (left) and deep splitting (right) on the 50-dimensional linear PIDE \eqref{linear_PDE}. Errors are reported at the initial condition $(0,\bm{0}_{50})$ for matched wall-clock runtimes. \\ }
\label{tab:local_baselines}
\begin{minipage}[t]{0.47\textwidth}
\small
\centering
\hspace{-2mm}
\begin{tabular}{lc}
\toprule
Method & \hspace{-4.5mm}Local error at $(0,\bm{0}_{50})$ \\
\midrule
INEUS (100 sec.)        & 0.068 \\
Local TDL (100 sec.)    & 0.054 \\
\midrule
INEUS (500 sec.)        & 0.004 \\
Local TDL (500 sec.)    & 0.014 \\
\midrule
INEUS (1000 sec.)       & 0.0015 \\
Local TDL (1000 sec.)   & 0.011 \\
\bottomrule
\end{tabular}
\end{minipage}
\hfill
\begin{minipage}[t]{0.47\textwidth}
\small
\hspace{-7mm}
\begin{tabular}{lc}
\toprule
Method & \hspace{-4.5mm}Local error  at $(0,\bm{0}_{50})$ \\
\midrule
INEUS (100 sec.)                    & 0.068 \\
Local deep splitting (100 sec.)     & 0.042 \\
\midrule
INEUS (500 sec.)                    & 0.004 \\
Local deep splitting (500 sec.)     & 0.003 \\
\midrule
INEUS (1000 sec.)                   & 0.0015 \\
Local deep splitting (1000 sec.)    & 0.0023 \\
\bottomrule
\end{tabular}
\end{minipage}

\end{table}

To conclude, these comparisons highlight a structural advantage of INEUS for global approximation. Like PINNs, INEUS is explicitly designed to approximate the solution over the entire space-time domain. Indeed,
\begin{itemize}[leftmargin=1em]
    \item INEUS and PINNs are trained directly on the original PIDE and aim to provide solutions for all space-time points $(t,x)$ in the domain.
    \item Time-discretized BSDE-based methods instead solve a discretized approximation of the equation and are trained on data points only along a given time grid. Even when arbitrary restarts are used to improve domain coverage, BSDE-based methods still rely on simulated time-discretized trajectories. As a result, approximation errors tend to propagate/accumulate over successive time steps and their accuracy depends on how well the simulated paths cover the spatial domain. 
\end{itemize}

% \paragraph{Baselines.} When comparing INEUS with PINNs \cite{sirignano2018dgm} and the deep BSDE method with jumps \cite{andersson2025deep}, $\rm{MAE}$ is computed on a fixed test set of \(\mathcal{M}=2000\) sample points. The PINN architecture and hyperparameters are exactly the same as those used for INEUS (with $(n-1)\times N = 608$ optimizer updates per epoch), except that 64 samples are used for the Monte Carlo estimation of the jump term $\mathcal{I}[u]$. The deep BSDE method with jumps is an intrinsically local solver that parameterizes a single initial value $u(t_0,x_0)$. To compute a global test MAE, the method is run from many randomly sampled initial (training) points using multiple independent local models. Each deep BSDE model is  trained for 10,000 epochs with batch size 1024, Adam learning rate $3 \times 10^{-4}$, gradient clipping norm 1.0, 25 time steps and 64 Monte Carlo samples for the jump expectation. All three methods are implemented using TensorFlow and Keras with GPU acceleration on an NVIDIA RTX 4090. %The code will be made available at x.
\section{Proof of theoretical results} \label{app:Proof}
	\subsection{Proof of Proposition \ref{Proposition: L2}} \label{proof: propL2}
    \begin{proof}
		Since $Y$ is independent of $E_1$, one obtains from the definition of $\mathcal{G}_\xi$ that
		\begin{align}
			\mathbb{E}[\mathcal{G}_\xi(Y,  E_1, u) \mid Y= (t,x)] = u(t,x)   &+ \xi \hspace{0.1mm}\Big[ \partial_t u(t,x) + \mathcal{F}(t,x,u(t,x),\nabla_xu(t,x), \nabla^2_xu(t,x)) \hspace{-0.5mm} \nonumber
			\\
			&  \hspace{10mm}+ \lambda(t,x) \, \mathbb{E}_{\nu} \left[ \ell\!\left(u(t,x+\gamma(t,x,E_1)) - u(t,x) \right) \right]  \Big] \nonumber
			\\[0.1cm] \label{Bell}
			&\hspace{-17mm}= u(t,x) + \xi \, \mathcal{A}[u](t,x)\\[0.1cm]
			&\hspace{-17mm}= u(t,x), \notag
		\end{align} 
		where the last equality follows from Equation \eqref{HJBB}. On the other hand, it is well known that 
		\begin{equation}
			\mathbb{E}[\mathcal{G}_\xi(Y,  E_1, u) \mid Y] = v^*(Y) \label{L2opt}
		\end{equation}
		for the Borel measurable function $v^* \colon  D_T \to \R$ minimizing the mean squared error
		\[ v \mapsto
	 \mathbb{E} \edg{ \big( v(Y) - \mathcal{G}_\xi(Y, E_1, u) \big)^2 } ,
		\]
	over all Borel measurable functions $v :  D_T \to \R$,	see e.g. Theorem 4.1.15 of \citet{durrett}. This concludes the proof. \vspace{-1mm}
	\end{proof}

    \subsection{Proof of Lemma \ref{lemma 1}}
    \begin{proof}
		Since $u^*$ is an exact minimizer of \eqref{objective_functional},   Equations \eqref{Bell} and \eqref{L2opt} in \ref{proof: propL2} imply that, for $\mu$-a.e. $(t,x) \in D_T$, 
        \[
u^*(t,x)
= \E\!\left[\mathcal{G}_\xi\!\left((t,x),E_1,u_{\theta^{(k)}}\right)\middle|(t,x)\right]
= u_{\theta^{(k)}}(t,x)+\xi \mathcal{A}[u_{\theta^{(k)}}](t,x).
\]
Hence the update rule \eqref{ope1b} holds $\mu$-a.e.\ on $D_T$. 
Now let $u^*$ be a fixed point of $\overline{\mathcal{T}}_{\hspace{-0.65mm}\xi}$. By definition of \eqref{ope1b}, $\mu\text{-a.e. on } D_T$,
\[
u^*=\overline{\mathcal{T}}_{\hspace{-0.65mm}\xi} \, u^*
\quad \Longleftrightarrow \quad
u^* = u^* + \xi \mathcal A[u^*],
\]
so that
\[
\xi \mathcal A[u^*]=0
\qquad \mu\text{-a.e. on } D_T.
\]
Since $\xi\neq 0$, it follows that
\[
\mathcal A[u^*]=0
\qquad \mu\text{-a.e. on } D_T.
\]
Because $\mu$ has full support and $(t,x)\mapsto \mathcal A[u^*](t,x)$ is continuous, we conclude that
\[
\mathcal A[u^*](t,x)=0
\qquad \text{for all } (t,x)\in D_T.
\]
Finally, if $u^*$ is represented by the hPINN architecture \eqref{hPINN}, then the terminal and boundary conditions are satisfied by construction, namely
\[
u^*(T,x)=\varphi(x), \qquad x\in D,
\]
and
\[
u^*(t,x)=g(t,x), \qquad (t,x)\in [0,T) \times D^c.
\]
Therefore, $u^*$ satisfies \eqref{intgen} together with \eqref{termgen}--\eqref{boundgen}, and is thus a classical solution of the PIDE problem \eqref{intgen}--\eqref{boundgen}.
%		The fixed point property of $\overline{\mathcal{T}}$ is equivalent to $\xi \mathcal{A}[u] = 0$.
%		Since $\xi \neq 0$, this implies $\mathcal{A}[u](t,x) = 0$ for almost all $(t,x) \in D_T$. Since $\mu$ as full support and $(t,x) \mapsto \mathcal{A}[\cdot](t,x)$ is continuous, $\mathcal{A}u=0$ everywhere in $D_T$. Together with  $u$ being realized by a hPINN architecture \eqref{hPINN} (implying $u(T,x)=F(x)$ on $D$ and $u(t,x) = g(x)$ on $D_T^{\mathrm{ext}}$ and hence satisifying the terminal and boundary conditions \eqref{termgen}--\eqref{boundgen} by construction), this makes $u$ a classical solution to the PIDE problem \eqref{intgen}-\eqref{boundgen}. 
	\end{proof}

\subsection{Proof of Proposition \ref{prop:contrac}}

\begin{proof}
    We establish \ref{prop:contrac:item1}, \ref{prop:contrac:item2}, and \ref{prop:contrac:item3} in sequence.
    
   \textbf{Proof of \ref{prop:contrac:item1}}. The operator $\mathcal S_h$ is linear by linearity of expectation. Since $c\geq c_0>0$,
\[
\|\mathcal S_h\psi\|_\infty\leq\|\psi\|_\infty,
\qquad \psi\in\mathcal{C}.
\]
In particular, $\mathcal S_h$ is bounded.
    The identity $\mathcal{S}_0 = \mathbb{I}$ is clear 
    from~\eqref{semigroupb}.  For $h_1,h_2\geq0$, the truncated time increments satisfy
    \[
    \tau_{h_1+h_2}(t)
    =
    \tau_{h_1}(t)
    +
    \tau_{h_2}\bigl(t+\tau_{h_1}(t)\bigr).
    \]
    Indeed, both sides are equal to
    \(
    (h_1+h_2)\wedge(T-t).
    \)
    Using this identity, the Markov property of $Y^y$, and the tower
    property of conditional expectation, we obtain
{\small \begin{align*} 
 	\mathcal S_{h_1+h_2}\psi(y)
 	&=
 	\mathbb E_y\left[
 	\exp\left(
 	-\int_0^{\tau_{h_1+h_2}(t)}
 	c(Y_r^y)\,dr
 	\right)
 	\psi\bigl(Y_{\tau_{h_1+h_2}(t)}^y\bigr)
 	\right]
 	\\
 	&=
 	\mathbb E_y\Bigg[
 	\exp\left(
 	-\int_0^{\tau_{h_1}(t)}
 	c(Y_r^y)\,dr
 	\right) \cdot 
 	\mathbb E\left[
 	\exp\left(
 	-\int_{\tau_{h_1}(t)}^{\tau_{h_1+h_2}(t)}
 	c(Y_r^y)\,dr
 	\right)
 	\psi\bigl(Y_{\tau_{h_1+h_2}(t)}^y\bigr)
 	\,\middle|\,
 	\mathcal F_{\tau_{h_1}(t)}
 	\right]
 	\Bigg]
 	\\
 	&=
 	\mathbb E_y\left[
 	\exp\left(
 	-\int_0^{\tau_{h_1}(t)}
 	c(Y_r^y)\,dr
 	\right)
 	\mathcal S_{h_2}\psi
 	\bigl(Y_{\tau_{h_1}(t)}^y\bigr)
 	\right]
 	\\
 	&=
 	\mathcal S_{h_1}
 	\bigl(\mathcal S_{h_2}\psi\bigr)(y).
 \end{align*}}Hence, we have the semigroup property
\[
\mathcal S_{h_1+h_2}
=
\mathcal S_{h_1}\circ\mathcal S_{h_2}.
\] Monotonicity follows from the positivity of the discount factor.
Indeed, if $\psi_1\leq\psi_2$ pointwise, then
\[
\exp\left(
-\int_0^{\tau_h(t)}c(Y_r^y)\,dr
\right)
\psi_1\bigl(Y_{\tau_h(t)}^y\bigr)
\leq
\exp\left(
-\int_0^{\tau_h(t)}c(Y_r^y)\,dr
\right)
\psi_2\bigl(Y_{\tau_h(t)}^y\bigr),
\]
and therefore
\[
\mathcal S_h\psi_1
\leq
\mathcal S_h\psi_2.
\]
 To verify strong
continuity, let $\psi\in \mathcal{C}$. For
$y=(t,x)\in\bar D_T$, we have{\small
\begin{align*}
	\left|
	\mathcal S_h\psi(y)-\psi(y)
	\right|
	\leq{}&
	\mathbb E_y\left[
	\left|
	\exp\left(
	-\int_0^{\tau_h(t)}
	c(Y_r^y)\,dr
	\right)-1
	\right|
	\left|
	\psi\bigl(Y_{\tau_h(t)}^y\bigr)
	\right|
	\right] +
	\mathbb E_y\left[
	\left|
	\psi\bigl(Y_{\tau_h(t)}^y\bigr)-\psi(y)
	\right|
	\right].
\end{align*}}Since $c$ is bounded and $\tau_h(t)\leq h$,
\[
\left|
\exp\left(
-\int_0^{\tau_h(t)}
c(Y_r^y)\,dr
\right)-1
\right|
\leq
\|c\|_\infty h.
\]
Consequently, the first term converges to zero uniformly in $y$.
The second term converges uniformly to zero by the uniform stochastic
continuity of the space-time process $Y^y$ and the uniform continuity of
$\psi$. Hence,
\[
\lim_{h\downarrow0}
\left\|
\mathcal S_h\psi-\psi
\right\|_\infty
=
0.
\]
Thus, $(\mathcal S_h)_{h\geq0}$ is a strongly continuous, monotone
semigroup on $\mathcal{C}$, which shows \ref{prop:contrac:item1}.
    
\textbf{Proof of \ref{prop:contrac:item2}.} For every $\psi\in\mathcal{C}$, we have
\[
\mathcal S_h\psi(T,x)=\psi(T,x),
\]
because $\tau_h(T)=0$. In particular, $\mathcal S_h$
maps $\mathcal C_0$ into $\mathcal C_0$. 
Moreover, $\mathcal R_h(T,x)=0$, and hence
$\mathcal T_h$ maps $\mathcal C_\varphi$ into $\mathcal C_\varphi$.

Let $\psi\in\mathcal C_0$ and fix
$y=(t,x)\in\bar D_T$. We distinguish two cases.

If $h\leq T-t$, then $\tau_h(t)=h$, and Assumption
\eqref{eq:assumption_c} yields
\begin{align*}
	\left|
	\mathcal S_h\psi(y)
	\right|
\leq
	\mathbb E_y\left[
	\exp\left(
	-\int_0^h c(Y_r^y)\,dr
	\right)
	\left|
	\psi(Y_h^y)
	\right|
	\right] \leq
	e^{-c_0h}\|\psi\|_\infty.
\end{align*}

If $h>T-t$, then $\tau_h(t)=T-t$ and
\(
Y_{\tau_h(t)}^y
=
\bigl(T,X_T^{t,x}\bigr).
\)
Since $\psi\in\mathcal C_0$,
\(
\psi\bigl(T,X_T^{t,x}\bigr)=0,
\)
and therefore
\[
\mathcal S_h\psi(y)=0.
\]
Combining the two cases and taking the supremum over
$y\in\bar D_T$ gives
\[
\|\mathcal S_h\psi\|_\infty
\leq
e^{-c_0h}\|\psi\|_\infty.
\]
Since $e^{-c_0h}<1$ for every $h>0$, the restriction of
$\mathcal S_h$ to $\mathcal C_0$ is a strict contraction, showing \ref{prop:contrac:item2}.

Now let $\psi_1,\psi_2\in\mathcal C_\varphi$. Since they satisfy the
same terminal condition,
\[
\psi_1-\psi_2\in\mathcal C_0.
\]
Moreover, the affine term $\mathcal R_h$ cancels, so that
\[
\mathcal T_h\psi_1-\mathcal T_h\psi_2
=
\mathcal S_h(\psi_1-\psi_2).
\]
It follows that
\[
\|\mathcal T_h\psi_1-\mathcal T_h\psi_2\|_\infty
\leq
e^{-c_0h}
\|\psi_1-\psi_2\|_\infty.
\]
    
\textbf{Proof of \ref{prop:contrac:item3}}.    First, let $\psi\in D(\mathcal A_0)$. Since
$\mathcal S_h\psi(T,x)=\psi(T,x)$ for every $h\geq0$,
the definition of the generator gives
\[
\mathcal A_0\psi(T,x)
=
\lim_{h\downarrow0}
\frac{\mathcal S_h\psi(T,x)-\psi(T,x)}{h}
=0,
\qquad x\in D.
\]

Now let $\psi\in D(\mathcal A_0)\cap\mathcal U$
satisfy the additional regularity condition in the
statement, and denote by $\Phi\in\mathcal{C}$ the
continuous extension of
\[
\partial_t\psi+\mathcal L[\psi]
+\mathcal I[\psi]-c\psi
\]
from $D_T$ to $\bar D_T$.

Fix $y=(t,x)\in D_T$ and choose $0<h<T-t$.
For $0\leq s\leq h$, define
\[
U_s:=e^{-\int_0^s c(Y_r^y)\,dr},
\qquad
Z_s:=U_s\psi(Y_s^y).
\]
It\^o's formula for jump-diffusions, together with
the product rule and
$dU_s=-c(Y_s^y)U_s\,ds$, yields
\begin{equation}\label{eq:ito}
Z_s
=
\psi(y)
+
\int_0^s U_r\Phi(Y_r^y)\,dr
+
N_s,
\qquad 0\leq s\leq h,
\end{equation}
where $N$ is a local martingale with $N_0=0$.

Because $0<U_s\leq1$ and $\psi,\Phi$ are bounded,
\eqref{eq:ito} implies
\[
|N_s|
\leq
2\|\psi\|_\infty+h\|\Phi\|_\infty,
\qquad 0\leq s\leq h.
\]
Thus, $N$ is a bounded local martingale on $[0,h]$.
By localization and dominated convergence, it is
a true martingale and $\mathbb E_y[N_h]=0$.
Taking expectations in \eqref{eq:ito} gives
\begin{equation*}
\mathcal S_h\psi(y)-\psi(y)
=
\int_0^h
\mathbb E_y\left[
e^{-\int_0^s c(Y_r^y)\,dr}\Phi(Y_s^y)
\right]ds
=
\int_0^h\mathcal S_s\Phi(y)\,ds,
\end{equation*}
where the last equality uses $s\leq h<T-t$.

Since $\Phi\in\mathcal{C}$, strong continuity from
\ref{prop:contrac:item1} implies
\[
\left|
\frac{\mathcal S_h\psi(y)-\psi(y)}{h}
-\Phi(y)
\right|
\leq
\frac1h\int_0^h
\|\mathcal S_s\Phi-\Phi\|_\infty\,ds
\longrightarrow0
\]
as $h\downarrow0$.
On the other hand, $\psi\in D(\mathcal A_0)$ ensures
that the difference quotient converges in
$\mathcal{C}$, and hence pointwise, to
$\mathcal A_0\psi$. Uniqueness of the pointwise
limit therefore gives
\[
\mathcal A_0\psi(t,x)
=
\partial_t\psi(t,x)
+\mathcal L[\psi](t,x)
+\mathcal I[\psi](t,x)
-c(t,x)\psi(t,x),
\qquad (t,x)\in D_T.
\]
Finally, both $\mathcal A_0\psi$ and $\Phi$ belong
to $\mathcal{C}$ and agree on $D_T$. By continuity,
they also agree at the terminal boundary.
In particular, the extension necessarily satisfies
$\Phi(T,\cdot)=0$.
    which proves \ref{prop:contrac:item3}.
\end{proof}

\subsection{Proof of Proposition \ref{Prop:modified}}
   \begin{proof} We first verify that $\mathcal{T}_{h,\alpha}$ is well defined on $\mathcal C_\varphi$. By construction, $\mathcal T_h$ preserves $\mathcal C_\varphi$. In particular, for every $\psi\in\mathcal C_\varphi$, \[ \mathcal T_h\psi(T,x) = \psi(T,x) = \varphi(x), \qquad x\in D, \] because $\tau_h(T)=0$. Therefore, \begin{align*} \mathcal{T}_{h,\alpha}\psi(T,x) = \left( 1-\frac{\alpha}{h} \right)\psi(T,x) + \frac{\alpha}{h}\mathcal T_h\psi(T,x) = \left( 1-\frac{\alpha}{h} \right)\varphi(x) + \frac{\alpha}{h}\varphi(x) = \varphi(x). \end{align*} Hence, \( \mathcal{T}_{h,\alpha}\psi \in \mathcal C_\varphi. \)
  	
\textbf{ Proof of  \ref{prop:mod:item2}}.  Let $\psi_1,\psi_2\in\mathcal C_\varphi$. Since both functions satisfy the same terminal condition, \[ \psi_1-\psi_2\in\mathcal C_0. \] By Proposition~\ref{prop:contrac} \ref{prop:contrac:item2}, \[ \left\| \mathcal T_h\psi_1-\mathcal T_h\psi_2 \right\|_\infty \leq e^{-c_0 h} \left\| \psi_1-\psi_2 \right\|_\infty . \] 
	Therefore, using \eqref{mod_operator},
	\begin{align*}
		\hspace{0mm}	\mathcal{T}_{h,\alpha}\psi_1 -  \mathcal{T}_{h,\alpha}\psi_2&= \left(\psi_1 + \frac{\alpha}{h} \, (\mathcal{T}_{h}\psi_1- \psi_1 ) \right) - \left(\psi_2 + \frac{\alpha}{h} \, (\mathcal{T}_{h}\psi_2 - \psi_2 ) \right) 
		\\
		&= (\psi_1 - \psi_2 ) +\frac{\alpha}{h} \left(\mathcal{T}_{h} \psi_1  - \mathcal{T}_{h} \psi_2  - (\psi_1 - \psi_2)\right)
		\\
		&= (1-\frac{\alpha}{h})(\psi_1 - \psi_2) + \frac{\alpha}{h}\left(\mathcal{T}_{h} \psi_1  - \mathcal{T}_{h} \psi_2 \right) .
	\end{align*}
	Taking the supremum norm gives
	\begin{align*}
		\| 	\mathcal{T}_{h,\alpha}\psi_1 - \mathcal{T}_{h,\alpha}\psi_2 \|_\infty &\leq  \left| 1 -\frac{\alpha}{h} \right| \|\psi_1 - \psi_2 \|_\infty +\frac{\alpha}{h} 	\| \mathcal{T}_{h} \psi_1 - \mathcal{T}_{h} \psi_2 \|_\infty 
		\\
		& \leq \left(  \left| 1 -\frac{\alpha}{h} \right| +\frac{\alpha}{h} e^{-c_0 h} \right) \| \psi_1 - \psi_2 \|_\infty  \,.
	\end{align*}
Thus, $\mathcal{T}_{h,\alpha}$ is Lipschitz with constant \[ L_h := \left| 1-\frac{\alpha}{h} \right| + \frac{\alpha}{h}e^{-c_0 h}. \] It remains to verify that $L_h\in(0,1)$. We distinguish two cases. If $0<\alpha\leq h$, then \[ L_h = 1-\frac{\alpha}{h} + \frac{\alpha}{h}e^{-c_0 h} = 1-\frac{\alpha}{h}(1-e^{-c_0 h}). \] Since $0<e^{-c_0 h}<1$ and $0<\alpha/h\leq1$, it follows that \( 0<L_h<1. \) If $\alpha>h$, then \[ L_h = \frac{\alpha}{h}-1 + \frac{\alpha}{h}e^{-c_0 h} = \frac{\alpha}{h}(1+e^{-c_0 h})-1. \] In this case, $L_h>0$, while \[ L_h<1 \] is equivalent to \[ \frac{\alpha}{h}(1+e^{-c_0 h})<2, \quad \ \text{or} \ \quad \alpha < \frac{2h}{1+e^{-c_0 h}}. \] Consequently,
\[ 0<\alpha<\frac{2h}{1+e^{-c_0 h}} \quad \text{ implies } \quad  L_h\in(0,1), \] and $\mathcal{T}_{h,\alpha}$ is a strict contraction on $\mathcal C_\varphi$.
This proves \ref{prop:mod:item2}.

\medskip

   \textbf{Proof of \ref{prop:mod:item3}}. By the Feynman--Kac  Theorem \ref{thm:fk}, together with the Markov property and the tower property of conditional
expectation, the solution $u\in\mathcal U\cap\mathcal{C}$
of \eqref{intdiff}--\eqref{termdiff} satisfies
\[
u=\mathcal T_hu.
\]
Substituting this identity into \eqref{mod_operator} gives
\[
\mathcal{T}_{h,\alpha} u
=
u+\frac{\alpha}{h}(\mathcal T_hu-u)
=
u,
\]
thus \(u\) is a fixed point of \(\mathcal{T}_{h,\alpha}\). To prove uniqueness, let $\psi\in\mathcal C_\varphi$ be any fixed point of $\mathcal{T}_{h,\alpha}$. Then \[ \psi = \psi + \frac{\alpha}{h} \left( \mathcal T_h\psi-\psi \right). \] Since $\alpha/h>0$, this implies \[ \mathcal T_h\psi=\psi. \] Therefore, the fixed points of $\mathcal{T}_{h,\alpha}$ coincide with those of $\mathcal T_h$. By Proposition~\ref{prop:contrac} \ref{prop:contrac:item2}, the operator $\mathcal T_h$ is a strict contraction on the complete metric space $\mathcal C_\varphi$. It therefore has at most one fixed point by Banach's fixed-point theorem. Since $u$ is such a fixed point, we conclude that \[ \psi=u. \] Hence, $u$ is the unique fixed point of $\mathcal{T}_{h,\alpha}$ in $\mathcal C_\varphi$.\end{proof}
\subsection{Proof of Proposition \ref{ass:series}}
\begin{proof}
Fix $h>0$. We first derive an integral representation of
$\mathcal T_h\psi-\psi$.
For every $y=(t,x)\in\bar D_T$, the discounted
It\^o--Dynkin formula gives
\[
\mathcal S_h\psi(y)-\psi(y)
=
\mathbb E_y\left[
\int_0^{\tau_h(t)}
e^{-\int_0^s c(Y_r^y)\,dr}
\bigl(\mathcal A\psi-f\bigr)(Y_s^y)\,ds
\right].
\]
%Here, for $t<T$, one may first localize and stop strictly before $T$ (the boundedness of $\psi$, $\mathcal A\psi$, and $f$, together with $c\geq0$, justifies passage to the limit). At $t=T$, both sides vanish.
Using $\mathcal T_h\psi=\mathcal S_h\psi+\mathcal R_h$,
we therefore obtain
\[
\mathcal T_h\psi(y)-\psi(y)
=
\mathbb E_y\left[
\int_0^{\tau_h(t)}
e^{-\int_0^s c(Y_r^y)\,dr}
(\mathcal A\psi)(Y_s^y)\,ds
\right].
\]

Since $\mathcal A\psi\in\mathcal C_0$, for every
$y=(t,x)\in\bar D_T$ and $s\geq T-t$, we have
\[
\mathcal S_s(\mathcal A\psi)(y)
=
\mathbb E_y\left[
e^{-\int_0^{T-t}c(Y_r^y)\,dr}
(\mathcal A\psi)(T,X_T^{t,x})
\right]
=0.
\]
Therefore,
\begin{align*}
\mathcal T_h\psi(y)-\psi(y)
&=
\int_0^{\tau_h(t)}
\mathbb E_y\left[
e^{-\int_0^s c(Y_r^y)\,dr}
(\mathcal A\psi)(Y_s^y)
\right]ds
\\
&=
\int_0^{\tau_h(t)}
\mathcal S_s(\mathcal A\psi)(y)\,ds
=
\int_0^h\mathcal S_s(\mathcal A\psi)(y)\,ds,
\end{align*}
where Fubini's theorem applies because $\mathcal A\psi$
is bounded and the discount factor is at most one.
Hence,
\[
\mathcal T_h\psi-\psi
=
\int_0^h\mathcal S_s(\mathcal A\psi)\,ds
\qquad\text{on }\bar D_T.
\]

Since $\mathcal A\psi\in D(\mathcal A_0^{n})$,
the standard Taylor expansion for the strongly continuous
semigroup $(\mathcal S_s)_{s\geq0}$ gives
\[
\mathcal S_s(\mathcal A\psi)
=
\sum_{j=0}^{n-1}
\frac{s^j}{j!}\,
\mathcal A_0^j(\mathcal A\psi)
+
\frac{1}{(n-1)!}
\int_0^s
(s-r)^{n-1}
\mathcal S_r\mathcal A_0^{n}(\mathcal A\psi)
\,dr.
\]
Integrating from $0$ to $h$ gives
\begin{align*}
\mathcal T_h\psi
={}&
\psi+
\sum_{j=0}^{n-1}
\frac{1}{j!}
\left(\int_0^h s^j\,ds\right)
\mathcal A_0^j(\mathcal A\psi) +
\frac{1}{(n-1)!}
\int_0^h
\int_0^s
(s-r)^{n-1}
\mathcal S_r\mathcal A_0^{n}(\mathcal A\psi)
\,dr\,ds
\\
={}&
\psi+
\sum_{j=1}^{n}
\frac{h^j}{j!}\,
\mathcal A_0^{j-1}(\mathcal A\psi)
+
R_{n+1}(h),
\end{align*}
where, by Fubini's theorem,
\begin{align*}
R_{n+1}(h)
&:=
\frac{1}{(n-1)!}
\int_0^h
\int_0^s
(s-r)^{n-1}
\mathcal S_r\mathcal A_0^{n}(\mathcal A\psi)
\,dr\,ds
\\
&=
\frac{1}{n!}
\int_0^h
(h-r)^{n}
\mathcal S_r\mathcal A_0^{n}(\mathcal A\psi)
\,dr.
\end{align*}
By strong continuity, the integrands are continuous
with values in $\mathcal{C}$, which justifies the
Bochner integrals and the preceding use of Fubini's theorem.
In particular,
\[
R_{n+1}(h)\in\mathcal{C}.
\]
Moreover, $\mathcal S_r$ is nonexpansive on $\mathcal{C}$:
\[
\|\mathcal S_r\Phi\|_\infty
\leq
\|\Phi\|_\infty,
\qquad \Phi\in\mathcal{C}.
\]
Consequently,
\begin{align*}
\|R_{n+1}(h)\|_\infty
&\leq
\frac{1}{n!}
\int_0^h
(h-r)^{n}
\big\|\mathcal S_r
\mathcal A_0^{n}(\mathcal A\psi)\big\|_\infty
\,dr
\\
&\leq
\frac{
\big\|\mathcal A_0^{n}(\mathcal A\psi)\big\|_\infty
}{n!}
\int_0^h(h-r)^{n}\,dr
\\
&=
\frac{h^{n+1}}{(n+1)!}
\big\|\mathcal A_0^{n}(\mathcal A\psi)\big\|_\infty.
\end{align*}
Thus,
\[
\|R_{n+1}(h)\|_\infty=\mathcal O(h^{n+1})
\qquad\text{as }h\downarrow0,
\]
which proves \eqref{eq_mod_op} and its remainder estimate.

Dividing the preceding expansion by $h$ gives
\[
\frac{\mathcal T_h\psi-\psi}{h}
=
\sum_{j=1}^{n}
\frac{h^{j-1}}{j!}\,
\mathcal A_0^{j-1}(\mathcal A\psi)
+
\widetilde R_{n}(h),
\]
where
\[
\widetilde R_{n}(h)
:=
\frac{R_{n+1}(h)}{h}
\in\mathcal{C}.
\]
The preceding remainder estimate yields
\[
\|\widetilde R_{n}(h)\|_\infty
\leq
\frac{h^{n}}{(n+1)!}
\big\|\mathcal A_0^{n}(\mathcal A\psi)\big\|_\infty
=
\mathcal O(h^{n})
\qquad\text{as }h\downarrow0.
\]
Finally, for $\psi\in\mathcal C_\varphi$, using
the definition \eqref{mod_operator}, we obtain for every
$\alpha>0$,
\begin{align*}
\mathcal{T}_{h,\alpha}\psi
&=
\psi+\frac{\alpha}{h}
\left(\mathcal T_h\psi-\psi\right)
\\
&=
\psi+\alpha\left[
\sum_{j=1}^{n}
\frac{h^{j-1}}{j!}\,
\mathcal A_0^{j-1}(\mathcal A\psi)
+\widetilde R_{n}(h)
\right] .
\end{align*}
This proves \eqref{fin_approx}.
\end{proof}

We also note that, for a classical solution
$u\in\mathcal U\cap\mathcal C_\varphi$ of
\eqref{intdiff}--\eqref{termdiff}, the function
$\mathcal Au$ vanishes on $D_T$ and admits the zero
extension to $\bar D_T$. Consequently,
\[
\mathcal Au
\in D(\mathcal A_0^{n})\cap\mathcal C_0,
\qquad
\mathcal A_0^{j}(\mathcal Au)=0,
\qquad j=0,1,\ldots,n.
\]
Hence, all the explicit correction terms in
\eqref{fin_approx}, as well as the integral remainder
$R_{n+1}(h)$, vanish when $\psi=u$. This gives the exact
fixed-point identities
\[
\mathcal T_hu=u
\qquad\text{and}\qquad
\mathcal{T}_{h,\alpha}u=u.
\]
Under the condition
\[
0<\alpha<\frac{2h}{1+e^{-c_0 h}},
\]
this fixed point is unique by
Proposition~\ref{Prop:modified}.
\subsection{Proof of Theorem \ref{Prop: conv}}
\begin{proof}
 Set $
u_0:=u_{\theta^{(k)}}$. We first justify the inner recursion and the continuous
extensions of the functions $\mathcal Au_i$.

By Lemma~\ref{lemma 1} and the square-integrability assumption,
for each $i=0,\ldots,n-1$, the regression loss
\[
\mathscr L^{(k,i)}(v;\xi)
:=
\mathbb E\left[
\Big(v(Y)-\mathcal G_\xi(Y,E_1,u_i)\Big)^2
\right]
\]
is minimized over all square-integrable measurable functions
of $Y$ by the conditional expectation
$\mathbb E[\mathcal G_\xi(Y,E_1,u_i)\mid Y]$.
By the exact representability assumption, the hPINN minimizer
therefore satisfies
\begin{align}\label{eqinhom-euler-sup}
u_{i+1}(y)
&=
\mathbb E\!\left[\mathcal G_\xi(y,E_1,u_i)\right]
=
u_i(y)+\xi\mathcal Au_i(y),
\qquad
\mu\text{-a.e. on }D_T.
\end{align}
Under the  coefficient assumptions of Theorem \ref{thm:fk}, the differential and jump expressions defining $\mathcal Au_i$ are continuous
on $D_T$. The full support of $\mu$ therefore implies that
this identity holds on all of $D_T$.
Consequently, $\mathcal Au_i$ admits the continuous extension
\[
\mathcal Au_i
=
\frac{u_{i+1}-u_i}{\xi}
\qquad\text{on }\bar D_T.
\]
Since $u_i,u_{i+1}\in\mathcal U\cap\mathcal C_\varphi$
and the jump operator is linear, this also gives
\[
\mathcal Au_i\in\mathcal U\cap\mathcal C_0,
\qquad i=0,\ldots,n-1.
\]
Hence,
\[
u_{i+1}=u_i+\xi\mathcal Au_i
\qquad\text{on }\bar D_T,
\qquad i=0,\ldots,n-1.
\]

We next prove by induction that
\[
\mathcal Au_i
\in D(\mathcal A_0^{n-i})\cap\mathcal C_0,
\qquad i=0,\ldots,n-1.
\]
The claim holds for $i=0$ by assumption.
Suppose that it holds for some $i=0,\ldots,n-2$.
By the linearity of the differential and jump terms,
\eqref{eqinhom-euler-sup} gives
\[
\partial_t(\mathcal Au_i)
+\mathcal L[\mathcal Au_i]
+\mathcal I[\mathcal Au_i]
-c\,\mathcal Au_i
=
\frac{\mathcal Au_{i+1}-\mathcal Au_i}{\xi}
\qquad\text{on }D_T.
\]
The right-hand side belongs to $\mathcal{C}$ by the
preceding extension argument.
Since $\mathcal Au_i\in D(\mathcal A_0)\cap\mathcal U$,
the generator identification in
Proposition~\ref{prop:contrac}\,(iii) therefore yields
\[
\mathcal Au_{i+1}
=
\mathcal Au_i+\xi\mathcal A_0(\mathcal Au_i)
=
(I+\xi\mathcal A_0)(\mathcal Au_i)
\qquad\text{on }\bar D_T.
\]
The right-hand side belongs to
$D(\mathcal A_0^{n-1-i})$, and
$\mathcal Au_{i+1}\in\mathcal C_0$ was already established.
This completes the induction and justifies the operator
powers used below. In particular,
\[
\mathcal Au_i
=
(I+\xi\mathcal A_0)^i(\mathcal Au_0),
\qquad i=0,\ldots,n-1.
\]
	Iterating \eqref{eqinhom-euler-sup} $n$ times from $u_0 = u_{\theta^{(k)}}$ yields
 \begin{equation}\label{exp}
u_{n}
=
u_0+\xi\sum_{m=0}^{n-1}
(I+\xi\mathcal A_0)^m(\mathcal Au_0).
\end{equation}
	By the hockey-stick identity,
\[
\xi\sum_{m=0}^{n-1}
(I+\xi\mathcal A_0)^m(\mathcal Au_0)
=
\sum_{j=1}^{n}
\binom{n}{j}\xi^j
\mathcal A_0^{j-1}(\mathcal Au_0).
\]
Substituting into \eqref{exp} gives
\[
u_{n}
=
u_0+\sum_{j=1}^{n}
\binom{n}{j}\xi^j
\mathcal A_0^{j-1}(\mathcal Au_0).
\]

Define
\[
p_{n,j}
:=
\prod_{m=0}^{j-1}
\left(1-\frac{m}{n}\right)
\in[0,1],
\qquad j=1,\ldots,n.
\]
Since $\xi=h/n$, we have
\[
\binom{n}{j}\xi^j
=
\binom{n}{j}\left(\frac{h}{n}\right)^j
=
\frac{h^j}{j!}p_{n,j}.
\]
Consequently, we can rewrite $u_{n}$ as
\begin{align}\label{eq:proof}
u_{n}
&=
u_0+\sum_{j=1}^{n}
\frac{h^j}{j!}\mathcal A_0^{j-1}(\mathcal Au_0)
+
\sum_{j=1}^{n}
\frac{h^j}{j!}(p_{n,j}-1)
\mathcal A_0^{j-1}(\mathcal Au_0)
\nonumber\\
&=
u_0+\sum_{j=1}^{n}
\frac{h^j}{j!}\mathcal A_0^{j-1}(\mathcal Au_0)
+
E_{n}(h),
\end{align}
where 
\begin{equation}\label{eq:Enh}
E_{n}(h)
:=
\sum_{j=1}^{n}\frac{h^j}{j!}
(p_{n,j}-1)\,
\mathcal A_0^{j-1}(\mathcal Au_0) \, .
\end{equation}
The domain assumption on $\mathcal Au_0$ implies that every
term in this finite sum belongs to $\mathcal{C}$, so
$E_{n}(h)\in\mathcal{C}$.

It remains to estimate $E_{n}(h)$.
Since $p_{n,1}=1$, the term $j=1$ vanishes.
Moreover, for any finite family $(x_m)_m\subset[0,1]$,
\[
1-\prod_m(1-x_m)\leq\sum_m x_m.
\]
Applying this with $x_m=m/n$ gives, for
$j=1,\ldots,n$,
\[
1-p_{n,j}
=
1-\prod_{m=0}^{j-1}
\left(1-\frac{m}{n}\right)
\leq
\sum_{m=0}^{j-1}\frac{m}{n}
=
\frac{j(j-1)}{2n}.
\]
Therefore,
\begin{align*}
\|E_{n}(h)\|_\infty
&\leq
\sum_{j=1}^{n}
\frac{h^j}{j!}(1-p_{n,j})
\big\|\mathcal A_0^{j-1}(\mathcal Au_0)\big\|_\infty
\\
&\leq
\frac{1}{2n}
\sum_{j=2}^{n}
\frac{h^j}{j!}j(j-1)
\big\|\mathcal A_0^{j-1}(\mathcal Au_0)\big\|_\infty
\\
&=
\frac{h^2}{2n}
\sum_{j=2}^{n}
\frac{h^{j-2}}{(j-2)!}
\big\|\mathcal A_0^{j-1}(\mathcal Au_0)\big\|_\infty.
\end{align*}
This proves \eqref{eq:global-binomial-bound}.
When $n=1$, the sum is empty and $E_1(h)=0$.
Since $n$ is fixed and all the norms in the sum are finite,
\[
\|E_{n}(h)\|_\infty
=
\mathcal O\!\left(\frac{h^2}{n}\right),
\qquad h\downarrow0.
\]

We now insert \eqref{eq:proof} into the relaxed update
\eqref{eq:relaxed-update-clean}.
Since $u_0=u_{\theta^{(k)}}$, we obtain
\begin{align*}
u^{(k+1)}
&=
\left(1-\frac{\alpha}{h}\right)u_0
+
\frac{\alpha}{h}u_{n}
\\
&=
\left(1-\frac{\alpha}{h}\right)u_0
+
\frac{\alpha}{h}
\left[
u_0+\sum_{j=1}^{n}\frac{h^j}{j!}
\mathcal A_0^{j-1}(\mathcal Au_0)
+E_{n}(h)
\right]
\\
&=
u_0+
\alpha\sum_{j=1}^{n}\frac{h^{j-1}}{j!}
\mathcal A_0^{j-1}(\mathcal Au_0)
+
\frac{\alpha}{h}E_{n}(h),
\end{align*}
which is exactly \eqref{eq:conv-relaxed} on $\bar D_T$.

Finally, the expansion \eqref{fin_approx} from
Proposition~\ref{ass:series}, evaluated at $\psi=u_0$, gives
\[
\mathcal{T}_{h,\alpha}u_0
=
u_0+
\alpha\sum_{j=1}^{n}\frac{h^{j-1}}{j!}
\mathcal A_0^{j-1}(\mathcal Au_0)
+
\alpha\widetilde R_{n}(h),
\]
where
\[
\widetilde R_{n}(h)=\frac{R_{n+1}(h)}{h},
\qquad
\|\widetilde R_{n}(h)\|_\infty
=
\mathcal O(h^{n}).
\]
Comparing with \eqref{eq:conv-relaxed} yields
\begin{align*}
\big\|
u^{(k+1)}
-\mathcal{T}_{h,\alpha}u_{\theta^{(k)}}
\big\|_\infty
&\leq
\frac{\alpha}{h}\|E_{n}(h)\|_\infty
+
\alpha\|\widetilde R_{n}(h)\|_\infty
\\
&=
\mathcal O\!\left(\frac{\alpha h}{n}\right)
+
\mathcal O(\alpha h^{n}),
\end{align*}
as $h\downarrow0$ with $n$ fixed, which proves
\eqref{eq:conv-relaxed-Ttilde}.

If $\alpha=h$, then \eqref{eq:relaxed-update-clean}
reduces to
\[
u^{(k+1)}=u_{n},
\]
and $\mathcal{T}_{h,\alpha}=\mathcal T_h$.
In this case, comparing \eqref{eq:proof} with
the expansion \eqref{eq_mod_op} from
Proposition~\ref{ass:series}, namely
\[
\mathcal T_hu_0
=
u_0+\sum_{j=1}^{n}\frac{h^j}{j!}
\mathcal A_0^{j-1}(\mathcal Au_0)
+
R_{n+1}(h),
\]
gives
\begin{align*}
\big\|u^{(k+1)}
-\mathcal T_hu_{\theta^{(k)}}\big\|_\infty
&\leq
\|E_{n}(h)\|_\infty+\|R_{n+1}(h)\|_\infty
\\
&=
\mathcal O\!\left(\frac{h^2}{n}\right)
+\mathcal O(h^{n+1}),
\end{align*}
as $h\downarrow0$ with $n$ fixed.
This proves \eqref{eq:conv-inner-Th} and completes the proof.
\end{proof}

\section{Additional theoretical results} 
\subsection{Jump-induced variance and population minimizers} \label{app:jump_induced_var}
We quantify the jump-induced variance of the INEUS target and examine
how Monte Carlo noise affects the population objectives of INEUS and
residual-based PINNs. For INEUS, replacing the jump expectation by a
sample average preserves the population minimizers of each regression
problem, while for a standard  PINN loss,
squaring the sampled residual introduces an additional variance term
that alters the population minimizers.
\begin{Proposition}[Variance of the single-jump INEUS target] \label{prop:variance_estim}
Fix $u \in \mathcal{U}$, $y=(t,x)\in D_T$ and $m\geq 1$. Assume that
\begin{equation*}
\ell\!\left(
u(t,x+\gamma(t,x,E))
-u(t,x)
\right)
\end{equation*}
has finite variance under $E\sim \nu$. Let
\[
\widehat I_m[u](y)
:=
\lambda(y) \frac{1}{m}\sum_{j=1}^m \ell\!\left(
u(t,x+\gamma(t,x,E_j))
-u(t,x) \right),
\qquad
E_1,\ldots,E_m \overset{\mathrm{i.i.d.}}{\sim}\nu,
\]
be the $m$-sample Monte Carlo estimator of the nonlocal term $I[u](y)$, and let
\begin{equation} \label{singeljump2}
\mathcal{G}_\xi(y,E_1,u)
=
u(y)
+
\xi\left[
\partial_tu(y)
+
\mathcal{F}(y,u,\nabla_xu,\nabla_x^2u)
+
\lambda(y)  \ell\!\left(
u(t,x+\gamma(t,x,E_1))
-u(t,x) \right)
\right]
\end{equation}
be the single-jump INEUS target. Then, conditionally on $y$,
\[
\operatorname{Var}_\nu\!\left(\mathcal{G}_\xi(y,E_1,u)\right)
=
\xi^2 m\,
\operatorname{Var}_\nu\!\left(\widehat I_m[u](y)\right).
\]
\end{Proposition}

\begin{proof}
For fixed $y$, the only randomness in both estimators comes from the jump
samples. Since $E_1,\ldots,E_m$ are independent and identically distributed,
\[ \begin{aligned}
\operatorname{Var}_\nu\!\left(\widehat I_m[u](y)\right)
&=
\lambda^2(y)
\operatorname{Var}_\nu\!\left(
\frac{1}{m}\sum_{j=1}^m \ell\!\left(
u(t,x+\gamma(t,x,E_j))
-u(t,x) \right)
\right)
\\
&=
\frac{\lambda^2(y)}{m}
\operatorname{Var}_\nu\!\left( \ell\!\left(
u(t,x+\gamma(t,x,E))
-u(t,x) \right)\right).
\end{aligned}
\]
Moreover, all other terms in $\mathcal G_\xi(y,E_1,u)$ are deterministic
for fixed $y$, and hence
\[
\operatorname{Var}_\nu\!\left(\mathcal{G}_\xi(y,E_1,u)\right)
=
\xi^2\lambda^2(y)
\operatorname{Var}_\nu\!\left(\ell\!\left(
u(t,x+\gamma(t,x,E))
-u(t,x) \right)\right).
\]
The result follows immediately.
\end{proof}
In particular, for \(\xi = h/n=0.5/20\), we obtain
\[
\xi = 0.025,
\qquad
\xi^2 = 6.25\times 10^{-4}.
\]
Hence, when \(m=64\), $ 
\xi^2 m = 0.04 < 1,$ 
and therefore
\[
\operatorname{Var}_\nu\!\bigl(\mathcal{G}_\xi(y,E_1,u)\bigr)
<
\operatorname{Var}_\nu\!\bigl(\widehat I_m[u](y)\bigr).
\]
Thus, despite relying on a single jump sample, the INEUS target has substantially lower jump-induced variance than the  \(64\)-point Monte Carlo estimator of the nonlocal term appearing in a standard PINN residual, thanks to the small scaling factor \(\xi\). This comparison alone, however, does not imply greater
statistical efficiency since the scaling parameter $\xi$ reduces both
the jump-induced variance and the size of the corresponding INEUS update.

A more fundamental distinction concerns the effect of jump sampling
on the respective population objective. Fix $u\in\mathcal U$ and let $Y\sim\mu$
be independent of $\mathbf E=(E_1,\ldots,E_m) \sim \nu^{\otimes m}$.
Define the $m$-sample INEUS target by
\begin{equation} \label{eq:multi}
\mathcal G_{\xi,m}(y,\mathbf E,u)
:=
\frac{1}{m}\sum_{j=1}^m\mathcal G_\xi(y,E_j,u).
\end{equation}
This amounts to replacing the single-jump term in
\eqref{singeljump2} by $\widehat I_m[u](y)$ and recovers the original single-jump
target when $m=1$. Lemma \ref{lemma 1} directly implies that its conditional mean is again
\[
\mathbb E\!\left[
\mathcal G_{\xi,m}(Y,\mathbf E,u)\mid Y
\right]
=
u(Y)+\xi\mathcal A[u](Y).
\]
Consequently,  the conditional
bias-variance decomposition yields the following INEUS population objective over  $v\in L^2(\mu)$,
\begin{align}
&\mathbb E\!\left[
\bigl(
v(Y)-\mathcal G_{\xi,m}(Y,\mathbf E,u)
\bigr)^2
\right] \label{eq:INEUS_var}
\\
&\ =
\mathbb{E}_{\mu}\!\left[ \left(v(Y)-u(Y)-\xi\mathcal A[u](Y) \right)^2 \right]
+
\frac{\xi^2}{m}
\mathbb E_{\mu}\!\left[
\lambda^2(Y)
\operatorname{Var}_\nu\!\bigl(\ell\!\left(
u(Y+\gamma(Y,E))
-u(Y) \right)\bigr)
\right] , \nonumber
\end{align}
where $\E_{\mu} = \E_{Y\sim \mu}$ and $\operatorname{Var}_\nu = \operatorname{Var}_{E \sim \nu}$. In particular, since $u$ is fixed, the second term is independent of $v$
and therefore, for every $m \geq 1$, does not affect the population minimizer of the INEUS
regression problem. By contrast, for $v\in \mathcal{U}$, consider the $m$-sample PINN residual
\[
\widehat{\mathcal A}_m[v](y)
:=
\partial_tv(y)
+
\mathcal F(y,v,\nabla_xv,\nabla_x^2v)
+
\widehat I_m[v](y).
\]
Since
\(
\mathbb E_\nu\!\big[
\widehat{\mathcal A}_m[v](y)
\big]
=
\mathcal A[v](y),
\)
the corresponding PINN population loss over $v\in\mathcal{U}$ satisfies
\begin{align}
\mathbb E\!\left[
\widehat{\mathcal A}_m[v](Y)^2
\right]
&=
\mathbb E_\mu\!\left[
\mathcal A[v](Y)^2
\right] +
\frac{1}{m}
\mathbb E_\mu\!\left[
\lambda^2(Y)
\operatorname{Var}_\nu\!\left(
\ell\!\left(
v(Y +\gamma(Y,E))-v(Y)
\right)
\right)
\right].
\label{eq:pinn_variance_decomp}
\end{align}
Hence, unlike in the INEUS frozen-target regression problem, the additional Monte Carlo
variance term generally depends on the function $v$ being optimized. Consequently, although
$\widehat{\mathcal A}_m[v]$ is an unbiased estimator of
$\mathcal A[v]$, its square is generally an upward-biased estimator of
$\mathcal A[v]^2$. The  resulting additional variance term 
therefore modifies the finite-$m$ population objective and can alter the PINN population minimizer $v^*$.

\paragraph{Regression formulation vs. jump sampling.}
To assess the contribution of the regression-based formulation, we
compare INEUS with a residual-based PINN using the same number $m$
of jump samples per space-time point $y$. For INEUS, we replace the
single-jump term in \eqref{singeljump2} by $\widehat I_m[u](y)$,
giving the $m$-sample target $\mathcal G_{\xi,m}$ defined in
\eqref{eq:multi}; the case $m=1$ recovers the original INEUS target. Using the same number of jump samples $m$ in INEUS and PINN targets  indeed allows us to control for the accuracy of the  Monte Carlo approximation of the jump integral and therefore to isolate the effect of the regression-based INEUS formulation from that of the residual-based PINN. The new table below reports results for the linear PIDE \eqref{linear_PDE} in $d=50$ dimensions.

The results show that the regression-based INEUS formulation is particularly advantageous when the nonlocal term in $\mathcal G_{\xi,m}$ is approximated using only a small number of jump samples ($m=1,8$). This suggests that the INEUS regression target is substantially more robust than the PINN residual to the Monte Carlo noise induced by a coarse approximation of the jump integral. This behavior is consistent with \eqref{eq:INEUS_var} and \eqref{eq:pinn_variance_decomp}:
jump sampling leaves the population minimizers of each INEUS
regression problem unchanged, whereas it introduces a
function-dependent variance penalty into the PINN objective. For larger $m$, the PINN accuracy gradually improves and approaches that of INEUS, but at the expense of a much higher computational cost. Overall, INEUS remains both more accurate and considerably faster across all values of $m$. The table also shows that using between 1 and 16 jump samples provides a good compromise for INEUS between Monte Carlo variance reduction, accuracy and runtime.

\begin{table}[htbp]
    \centering
    \small
    \setlength{\tabcolsep}{4pt}
    \caption{Comparison of INEUS and PINN for the linear
    PIDE~\eqref{linear_PDE} in dimension $d=50$, using the same number
    of jump samples $m$ per space--time point $y$. We report the mean
    and standard deviation of the MAE, together with the runtime
    in seconds.\vspace{2mm}}
    \label{tab:jump_samples}
    \begin{tabular}{c c c r c c r} \vspace{2mm}
        & \multicolumn{3}{c}{INEUS}
        & \multicolumn{3}{c}{PINN} \\
        \cmidrule(lr){2-4} \cmidrule(lr){5-7}
        \shortstack{Jump\\samples $m$}
        & Mean MAE & Std.\ dev. & Time (s)
        & Mean MAE & Std.\ dev. & Time (s) \\
        \midrule
         1 & 0.0046 & 0.0009 & 5,867
           & 0.0323 & 0.0029 & 9,067 \\
         8 & 0.0036 & 0.0008 & 6,053
           & 0.0065 & 0.0013 & 9,693 \\
        16 & 0.0040 & 0.0006 & 6,107
           & 0.0048 & 0.0009 & 10,211 \\
        32 & 0.0034 & 0.0004 & 6,451
           & 0.0047 & 0.0012 & 11,184 \\
        64 & 0.0033 & 0.0003 & 7,529
           & 0.0035 & 0.0007 & 13,739 \\
        \bottomrule
    \end{tabular}
\end{table}

\subsection{Scalable implementation via directional differentiation} \label{section: directional diff}
	The next result  from \cite{cheridito2025deep} further reduces the cost of evaluating the differential part of the operator $\mathcal{G}_\xi$ by replacing the explicit computation of gradients and Hessians of \(u_\theta\) with a single directional second derivative. Indeed, the operator \( \mathcal{G}_\xi\) requires the evaluation of the second-order differential operator \(\mathcal L[u_\theta]\), which in principle involves \(\partial_t u_\theta\), \(\nabla_x u_\theta\), and \(\nabla_x^2 u_\theta\) at each sample point $(t,x)$.
	\begin{Proposition}\label{Proposition: efficent}
		Let  $u \in \mathcal{U}$ and $(t,x) \in D_T$ be given. Define the function 
		$\Psi: \mathbb{R} \to \mathbb{R}$ \vspace{-2mm} by
		\begin{equation*}
			\Psi(l) := \sum_{i=1}^{q} u \hspace{-0.3mm}
			\left(t + \frac{l^2}{2q}, x+\frac{l}{\sqrt{2}}\sigma_i(t,x)+\frac{l^2}{2q}b(t,x)\right) , \vspace{-1mm}
		\end{equation*} where $\sigma_i(t,x)$ is the $i^{th}$ column of the $d\times q$ matrix $\sigma(t,x)$. Then,
		\begin{align*}
			\Psi''(0) =  \partial_t u(t,x) &+ b^\top\hspace{-0.5mm}(t,x)  \, \nabla_x u(t,x) + \frac{1}{2} \normalfont \text{Tr} \left[ \sigma \sigma^\top\hspace{-0.5mm}(t,x) \nabla^2_x u(t,x) \right] 
		\end{align*}   \vspace{-4mm}
	\end{Proposition} 
    \begin{proof}
        See Proposition 3.1  of \cite{cheridito2025deep}.
    \end{proof}
	Proposition \ref{Proposition: efficent} makes it possible to replace the explicit computation of gradients and Hessians of 
	$u$ by evaluating the univariate function $\Psi''(0)$, the cost of which, using automatic differentiation, is a small multiple 
	of $q \cdot cost(u)$ and hence avoids the \(\mathcal{O}(d^2)\) scaling associated with explicitly
forming the full Hessian $\nabla_x^2 u$.
    \\ \\
    The next lemma shows how the same idea can be exploited in the nonlinear PIDE \eqref{nonlinear} through the logarithmic transformation \(v=e^{-\eta u}\).

\begin{Lemma}\label{lem:directional-log-transform}
Assume that $u(t,x) =
- \log (v(t,x))/\eta$.
Then
\[
\partial_t u(t,x)+\Delta u(t,x)-\eta\|\nabla_xu(t,x)\|^2
=
-\frac1\eta\,\frac{\partial_t v(t,x)+\Delta v(t,x)}{v(t,x)},
\]
where \(\Delta=\mathrm{Tr}\,\nabla_x^2\). Moreover, if
\[
\Psi_v(l)
:=
\sum_{i=1}^d
v\!\left(t+\frac{l^2}{2d},x+le_i\right)
=
\sum_{i=1}^d
\exp\!\left(
-\eta\,u\!\left(t+\frac{l^2}{2d},x+le_i\right)
\right),
\]
then
\[
\Psi_v''(0)=\partial_t v(t,x)+\Delta v(t,x),
\]
and therefore
\[
\partial_t u(t,x)+\Delta u(t,x)-\eta\|\nabla_xu(t,x)\|^2
=
-\frac{1}{\eta}\,\frac{\Psi_v''(0)}{v(t,x)}
=
-\frac{e^{\eta u(t,x)}}{\eta}\,\Psi_v''(0).
\]
\end{Lemma}

\begin{proof}
Since \(v=e^{-\eta u}\), we have
\[
\partial_t v=-\eta v\,\partial_t u,
\qquad
\nabla_x v=-\eta v\,\nabla_x u,
\]
and hence
\[
\Delta v
=
-\eta v\,\Delta u+\eta^2 v\,\|\nabla_xu\|^2.
\]
Dividing the identity
\[
\partial_t v+\Delta v
=
-\eta v\Big(\partial_t u+\Delta u-\eta\|\nabla_xu\|^2\Big)
\]
by \(-\eta v\) yields
\[
\partial_t u+\Delta u-\eta\|\nabla_xu\|^2
=
-\frac1\eta\,\frac{\partial_t v+\Delta v}{v}.
\]

Applying Proposition~\ref{Proposition: efficent} to \(v\) with \(b\equiv 0\) and \(\sigma= \sqrt{2}\,I_d\) gives
\[
\Psi_v''(0)=\partial_t v+\Delta v.
\]
Substituting this into the previous identity yields
\[
\partial_t u+\Delta u-\eta\|\nabla_xu\|^2
=
-\frac{1}{\eta}\,\frac{\Psi_v''(0)}{v}
=
-\frac{e^{\eta u}}{\eta}\,\Psi_v''(0),
\]
which proves the claim.
\end{proof}

\subsection{Setting generalizations} \label{app: generalization}
\paragraph{Polynomial-growth extension of the linear theory.}
The bounded-data assumption in Section~3 can be relaxed by replacing the
ordinary supremum norm with a polynomially weighted supremum norm. For
\(p>0\), define
\[
V_p(x):=(1+\lVert x\rVert^2)^{p/2}
\]
and the spatial polynomial-growth space
\[
\mathcal C^p
:=
\left\{
\psi\in C(\mathbb R^d):
\lVert\psi\rVert_p
:=
\sup_{x\in\mathbb R^d}
\frac{|\psi(x)|}{V_p(x)}
<\infty
\right\}.
\]
For \(\beta>0\), define the corresponding space--time space
\[
\bar{\mathcal C}^{\beta,p}
:=
\left\{
\psi\in C(\bar D_T):
\lVert\psi\rVert_{\beta,p}
:=
\sup_{(t,x)\in\bar D_T}
\frac{|\psi(t,x)|}
     {e^{\beta(T-t)}V_p(x)}
<\infty
\right\}.
\]
Since
\[
1\leq e^{\beta(T-t)}\leq e^{\beta T},
\]
this norm is equivalent, on the finite time interval, to the usual
polynomially weighted supremum norm. We assume that
\[
\varphi\in\mathcal C^p,
\qquad
f\in\bar{\mathcal C}^{\beta,p}.
\]

We also assume the following property:
\[
\mathcal S_h\bar{\mathcal C}^{\beta,p}
\subseteq
\bar{\mathcal C}^{\beta,p},
\qquad 0\leq h\leq T-t.
\]
More explicitly, this means that, whenever \(\psi\) is continuous and
has growth controlled by \(V_p\), the function
\[
(t,x)\longmapsto\mathcal S_h\psi(t,x)
\]
is also continuous and has growth controlled by \(V_p\). The moment
estimate below provides the growth bound, whereas preservation of
continuity can follow e.g.\ from standard Lipschitz/growth condition on the jump diffusion coefficients or locally uniform moment bound of order $q>p$.
%The moment estimate below provides the growth bound. Preservation of
%continuity requires an additional assumption on the transition
%operators. For example, suppose that \(\mathcal S_h\Phi\) is continuous
%for every bounded continuous function \(\Phi\), and that, for some
%\(q>p\),
%\[
%\sup_{(t,x)\in K}
%\mathbb E_{t,x}\!\left[V_q(X_{t+h})\right]<\infty
%\]
%for every compact set \(K\) of admissible starting points. These
%conditions also imply continuity of \(\mathcal S_h\psi\) for
%\(\psi\in\bar{\mathcal C}^{\beta,p}\): approximate \(\psi\) by bounded
%continuous functions obtained using spatial cutoffs. Their images
%under \(\mathcal S_h\) are continuous, while the higher-moment bound
%controls the truncation error uniformly on \(K\). Consequently,
%\(\mathcal S_h\psi\) is a locally uniform limit of continuous functions
%and is therefore continuous. Suitable Lipschitz and growth conditions
%on the jump-diffusion coefficients, together with appropriate
%assumptions on the jump intensity and jump moments, provide one way
%to establish these sufficient conditions.
Let
\[
\mathcal K_t:=\mathcal L_t+\mathcal I_t
\]
denote the diffusion--jump generator without the killing term. We assume
that \(\mathcal K_tV_p\) is well defined and for some \(\kappa_p\geq0\),
\[
\mathcal K_tV_p(x)\leq\kappa_pV_p(x).
\]
By the standard Lyapunov estimate, obtained from the localized Dynkin formula and Gronwall’s inequality,
\[
\mathbb E_{t,x}\left[V_p(X_{t+s})\right]
\leq
e^{\kappa_ps}V_p(x),
\qquad
0\leq s\leq T-t.
\]

Choose \(\beta>0\) such that
\[
\delta_p:=c_0+\beta-\kappa_p>0 .
\]
We now define
\[
\mathcal C_\varphi^{\beta,p}
:=
\{\psi\in\bar{\mathcal C}^{\beta,p}:
\psi(T,\cdot)=\varphi\}.
\]
%\[
%\mathcal C_0^{\beta,p}
%:=
%\{\psi\in\bar{\mathcal C}^{\beta,p}:
%\psi(T,\cdot)=0\}.
%\]
Let $\tau_h(t):=\min\{h,T-t\}$, as in Section~3.
For $\psi_1,\psi_2\in\mathcal C_\varphi^{\beta,p}$,
consider first $t+h<T$. Then
\[
\begin{aligned}
|(\mathcal S_h\psi_1-\mathcal S_h\psi_2)(t,x)|
&\leq
e^{-c_0h}
\lVert\psi_1-\psi_2\rVert_{\beta,p} \, 
e^{\beta(T-t-h)} \, 
\mathbb E_{t,x}[V_p(X_{t+h})]
\\
&\leq
e^{-\delta_ph} \, 
e^{\beta(T-t)}V_p(x)
\lVert\psi_1-\psi_2\rVert_{\beta,p}.
\end{aligned}
\]
Dividing by \(e^{\beta(T-t)}V_p(x)\) and taking the supremum yields
\[
\lVert
\mathcal S_h\psi_1-\mathcal S_h\psi_2
\rVert_{\beta,p}
\leq
e^{-\delta_ph} \,
\lVert\psi_1-\psi_2\rVert_{\beta,p}.
\]
If $t+h\geq T$, then $\tau_h(t)=T-t$ and
\[
(\mathcal S_h\psi_1-\mathcal S_h\psi_2)(t,x)=0,
\]
because $\psi_1(T,\cdot)=\psi_2(T,\cdot)=\varphi$.
The preceding bound therefore holds throughout $\bar D_T$.
The affine source term also has controlled polynomial growth.
Using \(c\geq c_0\), the weighted bound on \(f\), and
\(\mathbb E_{t,x}[V_p(X_{t+s})]\leq e^{\kappa_ps}V_p(x)\),
we obtain
\[
\begin{aligned}
|\mathcal R_h(t,x)|
&\leq
\int_0^{\tau_h(t)}
e^{-c_0s}
\mathbb E_{t,x}
\left[
|f(t+s,X_{t+s})|
\right]ds
\\
&\leq
\lVert f\rVert_{\beta,p} \,
e^{\beta(T-t)}V_p(x)
\int_0^h
e^{-(c_0+\beta-\kappa_p)s}\,ds.
\end{aligned}
\]
Dividing by \(e^{\beta(T-t)}V_p(x)\), taking the supremum,
and evaluating the integral gives
\[
\lVert\mathcal R_h\rVert_{\beta,p}
\leq
\frac{1-e^{-\delta_ph}}{\delta_p}
\lVert f\rVert_{\beta,p}.
\]
Thus,
\(\mathcal T_h=\mathcal R_h+\mathcal S_h\) is an affine self-map of
\(C_\varphi^{\beta,p}\), and its linear part has contraction
modulus \(e^{-\delta_ph}\). The relaxed operator
\(\mathcal{T}_{h,\alpha}\) therefore remains a strict contraction
whenever
\[
0<\alpha<\frac{2h}{1+e^{-\delta_ph}}.
\]

The preceding results admit the following weighted formulation.
Define
\[
D^{\beta,p}(\mathcal A_0)
:=
\left\{
\psi\in \bar{\mathcal C}^{\beta,p}:%\mathcal C_0^{\beta,p}:
\lim_{h\downarrow0}
\frac{\mathcal S_h\psi-\psi}{h}
\text{ exists in }\bar{\mathcal C}^{\beta,p}%\mathcal C_0^{\beta,p}
\right\}.
\]
In the generator characterization of
Proposition~\ref{prop:contrac}, take
\(\mathcal{C}=\bar{\mathcal C}^{\beta,p}\) and require
\[
\psi\in D^{\beta,p}(\mathcal A_0)\cap\mathcal U.
\]

For \(n\geq1\), the expansions and remainder estimates in
Proposition~\ref{ass:series}--Theorem~\ref{Prop: conv} retain their form,
with the assumptions interpreted as
\[
\psi\in\mathcal U\cap\bar{\mathcal C}^{\beta,p},
\qquad
\mathcal A\psi
\in D^{\beta,p}(\mathcal A_0^{n})
\cap\mathcal C_0^{\beta,p}
\]
in Proposition~\ref{ass:series}, and
\[
u_0\in\mathcal U\cap\mathcal C_\varphi^{\beta,p},
\qquad
\mathcal Au_0
\in D^{\beta,p}(\mathcal A_0^{n})
\cap\mathcal C_0^{\beta,p}
\]
in Theorem~\ref{Prop: conv}, where
\(
\mathcal C_0^{\beta,p}
:=
\left\{
\psi\in\bar{\mathcal C}^{\beta,p}:
\psi(T,\cdot)=0
\right\}
\).
 All supremum-norm estimates are then understood
in \(\lVert\cdot\rVert_{\beta,p}\), and the contraction and relaxation
estimates use \(e^{-\delta_ph}\) in place of \( e^{-c_0 h}\).
The minimization objective in Theorem~\ref{Prop: conv} is unchanged.
In particular, taking \(p=2\) accommodates the quadratic terminal
payoff of the linear experiment.

\paragraph{Use for nonlinear PIDEs.}
The convergence analysis in Section~\ref{sec:convergence} is deliberately
restricted to linear PIDEs. Extending our convergence result to nonlinear PIDEs requires additional nontrivial assumptions.
Our linear analysis relies crucially on the fact that $(\mathcal S_h)_{h\geq0}$ is a linear Feynman--Kac semigroup with linear infinitesimal generator $\mathcal A_0$. This allows us to expand $\mathcal S_h$ in powers of $\mathcal A_0$ in Proposition \ref{ass:series} and to show that the repeated INEUS minimizations $(I+\xi\mathcal A_0)^{n}$ reproduce this expansion in Theorem \ref{Prop: conv}.
Conversely, for a nonlinear operator $\mathcal A$ this composition cannot be expanded using the binomial theorem.
Even if the nonlinear PIDE generates a well-posed solution operator $(\mathcal P_h)_{h\geq0}$, its formal higher-order expansion generally contains terms such as 
\begin{equation*}
    D\mathcal A(\psi)\big[\mathcal A(\psi)\big]    
\end{equation*}
rather than powers \(\mathcal A^j\psi\). Stability of the exact nonlinear solution operator therefore does not imply stability or contractivity of the explicit INEUS update \(I+\xi\mathcal A\).
Nevertheless, the two nonlinear examples considered in Section~\ref{Section: numeric} possess problem-specific stability structures. 

For the HJB PIDE~\eqref{nonlinear}, the transformation \(v=e^{-\eta u}\) converts the exact nonlinear equation into a linear PIDE, yielding the  representation in Lemma~\ref{lem:directional-log-transform}.
This gives an order-preserving and stable exact solution operator, although INEUS is trained directly on $u$, and the transformation does not commute with the explicit INEUS recursion.

For the nonlinear Black--Scholes PIDE, let
\(\mathcal P_h\psi\) denote the exact solution obtained by evolving
the continuation value $\psi(t+h,\cdot)$ backward from time $t+h$ to time $t$.
The reaction function
\[
g(y)=-\bigl[(1-\delta)Q(y)+R\bigr]y
\]
is globally Lipschitz because $Q$ is continuous and piecewise affine and
is constant outside $[v^h,v^l]$. Consequently, under the standard
moment and well-posedness assumptions
\[
\|\mathcal P_h\psi_1-\mathcal P_h\psi_2\|_{\beta,p}
\leq
e^{-(\beta-\kappa_p-L_g)h}
\|\psi_1-\psi_2\|_{\beta,p}
\]
where $L_g$ is the Lipschitz constant.
Thus the exact evolution is
contractive when $\beta>\kappa_p+L_g$. Its terminal payoff
$\min_i x_i$ has linear growth and is covered by $p=1$. These structures
help explain the observed numerical stability, but they do not establish
convergence of the explicit INEUS recursion.

\section{Additional numerical experiments}
\paragraph{Linear PIDE with quadratic terminal condition.} We use as parameters in the linear PIDE experiments: $T = 0.5$, $b = 1$, $c=2$, $\Sigma = 0.28\, \mathbf{1}_{d \times q}$, $\lambda=0.25$, $\Sigma_J = 0.4 \,I_{d\times d}$, and a test set uniformly sampled from $[0,0.5] \times [-1.5,1.5]^d$.
Table \ref{tab:highdim_results} reports the performance of INEUS for higher-dimensional linear PIDEs of the form \eqref{linear_PDE}. The results are obtained after \(k_* = 1000\) training epochs and averaged over 10 independent runs of Algorithm \ref{Algo1}. As the state dimension \(d\) increases from 10 to 150 (with fixed $q=10$), the mean MAE remains small, while the computational time grows only moderately. These results provide empirical evidence of the computational scalability of INEUS. This is consistent with Proposition~\ref{Proposition: efficent} which allows the differential operator to be evaluated at a cost proportional to $q \operatorname{cost}(u_\theta)$, without explicitly forming the full spatial Hessian.
\begin{table}[h!]
\caption{\small Performance of INEUS for linear PIDEs of the form \eqref{linear_PDE} as a function of the state dimension \(d\), with fixed  $q=10$. Results are computed after \(k_* = 1000\) training epochs and reported as means and standard deviations over 10 independent runs of Algorithm \ref{Algo1}.}
\vspace{2mm}
\label{tab:highdim_results}
	\centering
	\begin{tabular}{c|c|c|c|c|c}
		 $d$ & Mean ${\rm MAE}$ & Std.\ Dev.\  $ {\rm MAE}$ & Mean Loss $\widehat{\mathscr{L}}^{(k_*)} $ & Std.\ Dev.\ Loss $\widehat{\mathscr{L}}^{(k_*)} $ & Time (sec) \\[1mm]
		\hline \\[-3mm]
		10   & 0.00043 & \(3.7 \times 10^{-5}\)   & 0.10523 & 0.00334 & 4120 \\
		25  & 0.00106 & \(1.4 \times 10^{-4}\)  & 0.17669 & 0.00217 & 4225 \\
		50  & 0.00291 & \(5.54 \times 10^{-4}\)  & 0.22776 & 0.00315 & 4303 \\
		100 & 0.01529 & \(6.63 \times 10^{-3}\) & 0.34239 & 0.01347 & 4687 \\
		150 & 0.04091 & \(1.67 \times 10^{-2}\)  & 0.56882 & 0.08724 & 5028 \\
		\hline 
	\end{tabular}
    \end{table}
    \\
The following figure then confirms that the relaxed operator $\mathcal{T}_{h,\alpha}$ in \eqref{mod_operator} helps improve the stability properties of our recursive approach based on  single-jump point sampling, compared with the standard Feynman--Kac operator $\mathcal{T}_h$ (i.e., $\alpha =h$). We also observe that this relaxed operator yields faster convergence in terms of epochs, even if performing the regression step on the target $u^{(k+K)}$ in \eqref{eq: target} once every block of $K=2$ epochs results in a higher overall computational cost. To improve stability without increasing runtime excessively, we then set $K=10$ in the experiments of Section \ref{Section: numeric}. \vspace{-1.5mm}
\label{Appendix: additional}
\begin{figure}[H]
	\begin{center}		\centerline{\includegraphics[scale=0.34]{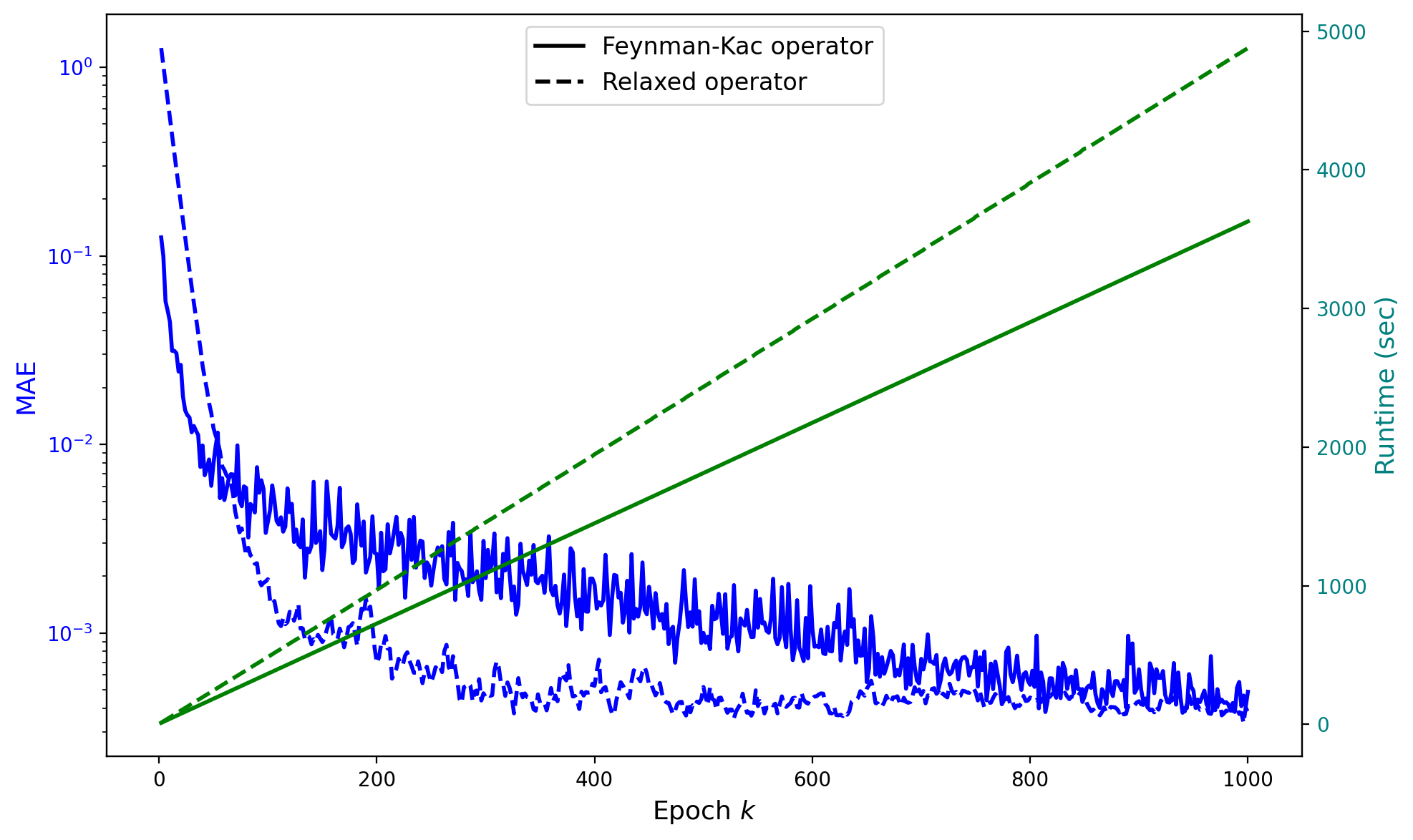}} \vspace{-1mm}
\caption{\small Comparison of INEUS using the Feynman--Kac operator $\mathcal{T}_h$ (solid line) and the relaxed operator $\mathcal{T}_{h,\alpha}$ (dotted line, $K=2$) in terms of $\rm{MAE}$ (blue) and runtime in seconds (green), as functions of the number of epochs $k$ for the linear PIDE \eqref{linear_PDE}. }
		\label{Fig: FKvsmod}
	\end{center}
	\vskip -0.3in
\end{figure}
Finally, we consider the linear PDE \eqref{linear_PDE} without jumps ($\lambda=0$). As shown in Figure~\ref{Fig: pde_linear}, INEUS remains more efficient even in this purely local setting. The gain in accuracy is consistent with the contractive structure of the INEUS fixed-point update in the linear case, while the reduction in computational cost comes from the recursive formulation: the training step avoids backpropagation through the full PDE residual and therefore does not require the higher-order derivatives that arise in residual-based \vspace{-1mm}PINNs.
\begin{figure}[H]
	\begin{center}		\centerline{\includegraphics[scale=0.34]{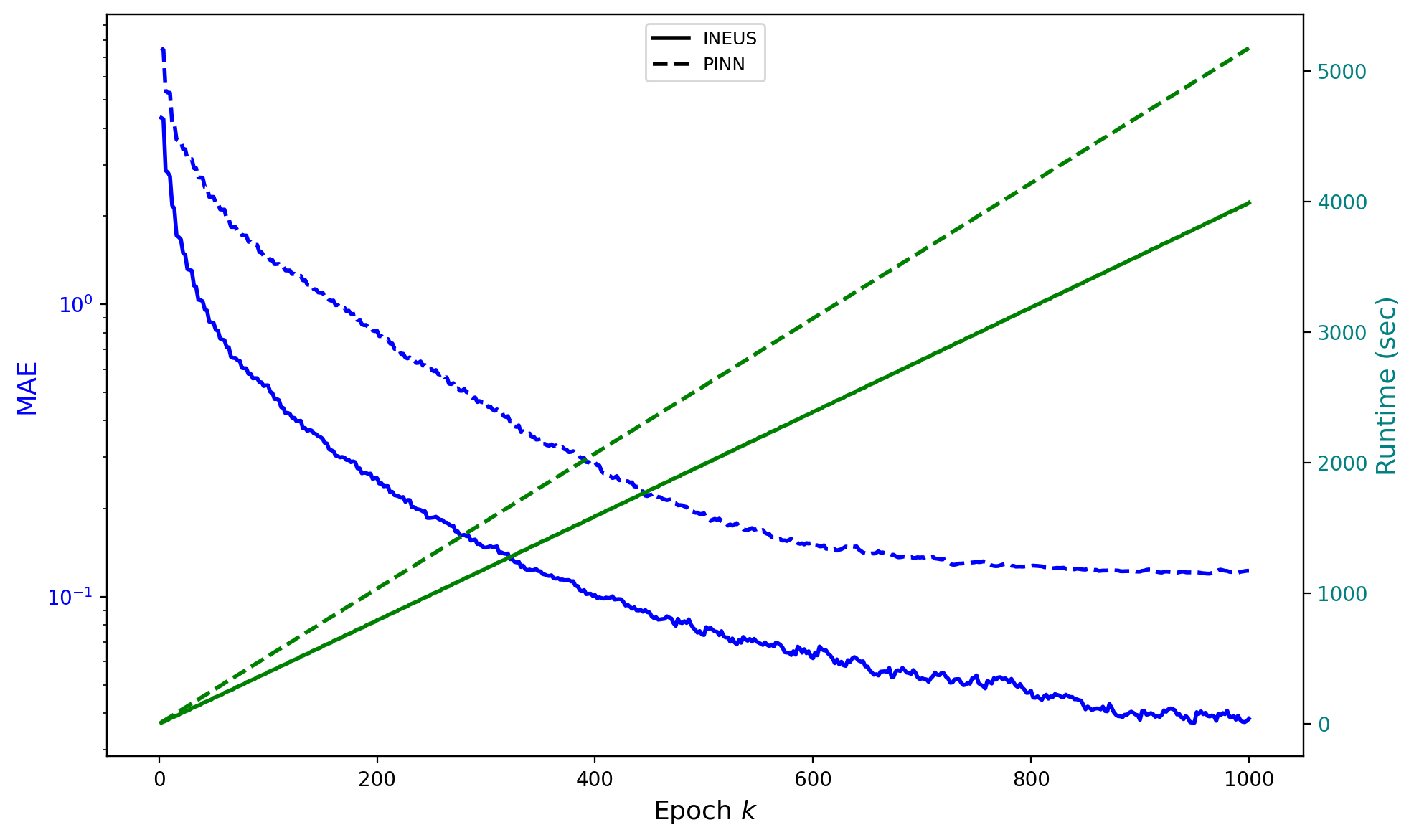}} \vspace{-1mm}
\caption{\small Comparison of $\rm{MAE}$ (blue) and runtime in seconds (green) of INEUS (solid lines) and PINNs (dashed lines) for a 100-dimensional linear PDE \eqref{linear_PDE} without jumps $(\lambda=0)$. }
		\label{Fig: pde_linear}
	\end{center}
	\vskip -0.3in
\end{figure}

\paragraph{Nonlinear Hamilton--Jacobi-Bellman PIDE.} The parameters for the PIDE \eqref{nonlinear} are: $T=1$, $\eta =1$, $f(\cdot) \equiv 2$, $\lambda = 0.5$, $\mu_J = \mathbf{0}_d$, $\Sigma_J = 0.2 I_d$, and a test set uniformly sampled from $[0,1] \times [-1.5,1.5]^d$. Table~\ref{tab:highdim_resultsnonlinear} further confirms the scalability of INEUS as the dimension $d$ increases, while Figure~\ref{Fig: pde_nonlinear} illustrates its computational advantage over PINNs even for nonlinear PDEs \eqref{nonlinear} without nonlocal term $(\lambda=0)$. The reference solution is evaluated semi-analytically from the Feynman--Kac representation \eqref{repres1} by conditioning on the Poisson jump count. For the quadratic terminal condition and Gaussian jump sizes considered here, the resulting conditional Gaussian expectation admits a closed-form expression. We evaluate the remaining Poisson series up to $40$ jumps using a numerically stable log-sum-exp formulation. \vspace{-5mm}\begin{table}[h!]
\vspace{-1mm}
\caption{\small Performance of INEUS for the nonlinear PIDE \eqref{nonlinear} as a function of the state dimension \(d\). Results are computed after \(k_* = 1000\) epochs and reported as means and standard deviations over 10 independent runs of Algorithm \ref{Algo1}.}
\label{tab:highdim_resultsnonlinear}
	\centering
	\begin{tabular}{c|c|c|c|c|c}
		 $d$ & Mean ${\rm MAE}$ & Std.\ Dev.\  $ {\rm MAE}$ & Mean Loss $\widehat{\mathscr{L}}^{(k_*)} $ & Std.\ Dev.\ Loss $\widehat{\mathscr{L}}^{(k_*)} $ & Time (sec) \\[1mm]
		\hline \\[-3mm]
		5  & 0.0049 & 0.0007  & 0.0513 & 0.0029 & 3278 \\
		10  & 0.0158 & 0.0015  & 0.1299 & 0.0068 & 3289\\
		25  & 0.0599 & 0.0185  & 0.3040 & 0.0853 & 4737 \\
		50 & 0.1284 &0.0201  &  0.7561& 0.1243 & 5772 \\
		100 & 0.2526 & 0.0277  & 0.8640 & 0.1053 & 10,422 \\
		\hline 
	\end{tabular}
    \end{table}
\begin{figure}[H]
	\begin{center}		\centerline{\includegraphics[scale=0.355]{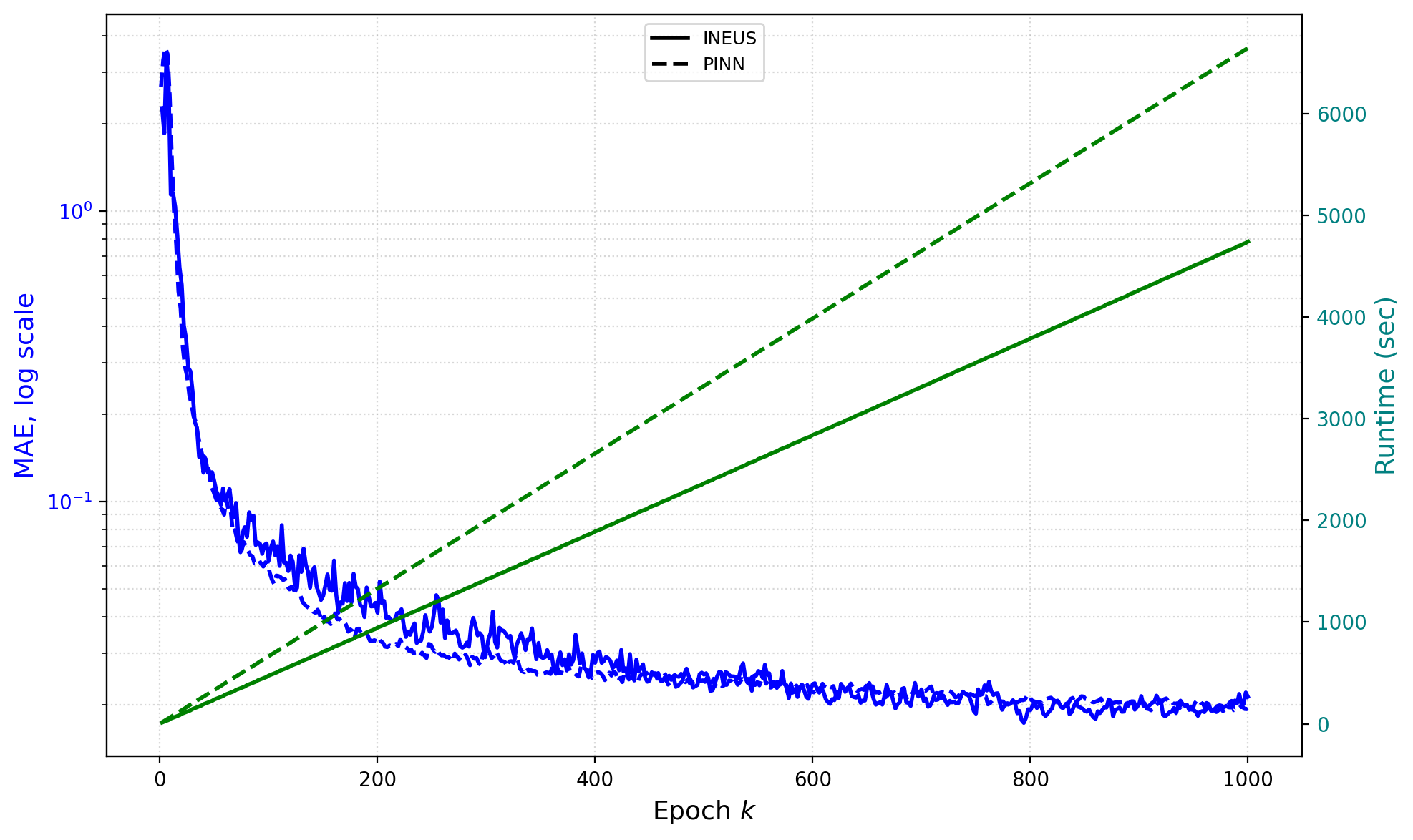}} \vspace{-1mm}
\caption{\small Comparison of $\rm{MAE}$ (blue) and runtime in seconds (green) of INEUS (solid lines) and PINNs (dashed lines) for a 100-dimensional nonlinear PDE \eqref{nonlinear} without jumps $(\lambda=0)$. }
		\label{Fig: pde_nonlinear}
	\end{center}
	\vskip -0.2in
\end{figure}
 \paragraph{Linear Black--Scholes PIDE.}
 We consider  a  financial derivative on a basket of  $d$ underlying stocks, whose Black--Scholes pricing PIDE is given on $D_T = [0,T) \times \R^d_+ $ \vspace{-1mm}by
\begin{align} 
 		\partial_t u(t,x) + (r - \kappa \odot \lambda) \odot x \cdot \nabla_x u(t,x) + \frac{1}{2}&\text{Tr}[\operatorname{diag}(x) \sigma \Sigma \sigma^\top \operatorname{diag}(x) \nabla^2_x u(t,x)] -r u(t,x)\nonumber \\[-2mm]
 		&\hspace{-5mm}+ \sum_{i=1}^d \lambda_i \mathbb{E}_\nu\left[u(t,x \odot e^{E_i e_i} ) - u(t,x) \right] = 0 \, ,
        \label{Black_Scholes}
 	\end{align}
 with terminal condition $u(T,x) = (\sum_{i=1}^d w_i x_i - K)_+$, for $w_i \in[0,1]$ and $\sum_{i=1}^d w_i =1$. We consider here i.i.d.\ normal random variables $E_i \sim \mathcal{N}(\mu^i_J, (\sigma^i_J)^2)$ with $\mu^i_J \in \mathbb{R}$, $\sigma^i_J \in \mathbb{R}_+$, and denote $\kappa=(\kappa_1, \ldots, \kappa_d)$ where $ \kappa_i := \mathbb{E}[e^{E_i}-1] $, %= \exp \big(\mu_Z^i +\frac{1}{2} (\sigma^i_Z)^2 \big) - 1,$
 $\sigma =\text{diag}(\sigma_1,\ldots, \sigma_d),$ \,$\odot$  the component-wise multiplication,  and $e_i$ the  $i$-th unit vector. The parameters in this setting  are: $T=1$, $r=0.05$, $K=30$, $\Sigma
=
0.1\,\mathbf{1}_{d \times d}
+
0.9\,I_d$ with $w_i = 1/d$, $\sigma_i = 0.15$, $\lambda_i =0.5$, $\mu_J^i = 0.2$, $\sigma_J^i = 0.3$, $\kappa_i = \exp (\mu_J^i +\frac{1}{2} (\sigma^i_J)^2) -1 = 0.278$ for $i=1,\dots, d$, and a test set uniformly sampled from $[0,1] \times [0,100]^d$. The reference solution in Figures \ref{fig: BS_linear}--\ref{Fig: pde_bs} and Table \ref{tab:highdim_resultsbs}  is then approximated by Monte Carlo simulation based on the Feynman--Kac representation \eqref{Feynman}. %with prices processes $S^i$ evolving under the risk-neutral measure $\mathbb{Q}$ as
% \begin{equation}
% 	\frac{dS^i_t}{S^i_{t-}} = (r - \lambda^i \kappa^i) dt + \sigma^i dW^i_t 
 %	+ \brak{e^{Z^i} - 1} dN^i_t , \quad \ \  i=1,\ldots, d, \label{S}
% \end{equation}
 %Let $\Phi : \R^d_+ \to \R_+$ denote the payoff of a European claim maturing at maturity $T$. Its time-$t$ price is given by
% \begin{equation}
% 	V(t,S) = \mathbb{E}^{\mathbb{Q}}[e^{-r(T-t)}\varphi(S_T) | S_t = S] \, , \quad \ (t,S) \in [0,T]\times \R^d_+ \, .
% \end{equation}
% The faire price $V$ is known to satisfy the following PIDE
  \begin{figure}[h]
	\centering
	\begin{minipage}{.5\textwidth}
		\includegraphics[width=0.95\columnwidth, height=4.2cm]{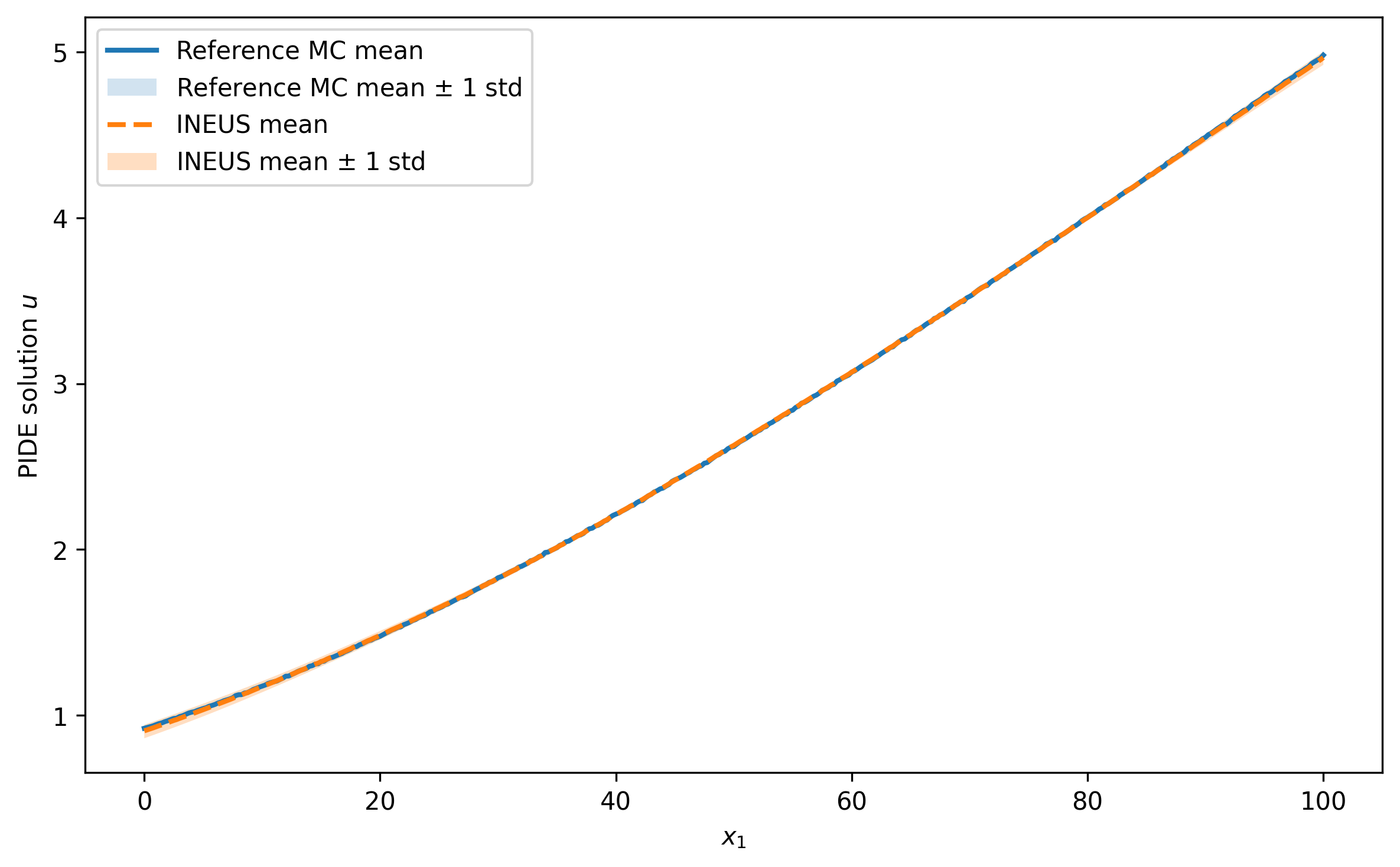}
		\label{fig:BS_t0} \vspace{6mm}
	\end{minipage} \hspace{-4mm}
	\begin{minipage}{.5\textwidth} 	\vspace{-5mm}
		\includegraphics[width=0.96\columnwidth, height=4.2cm]{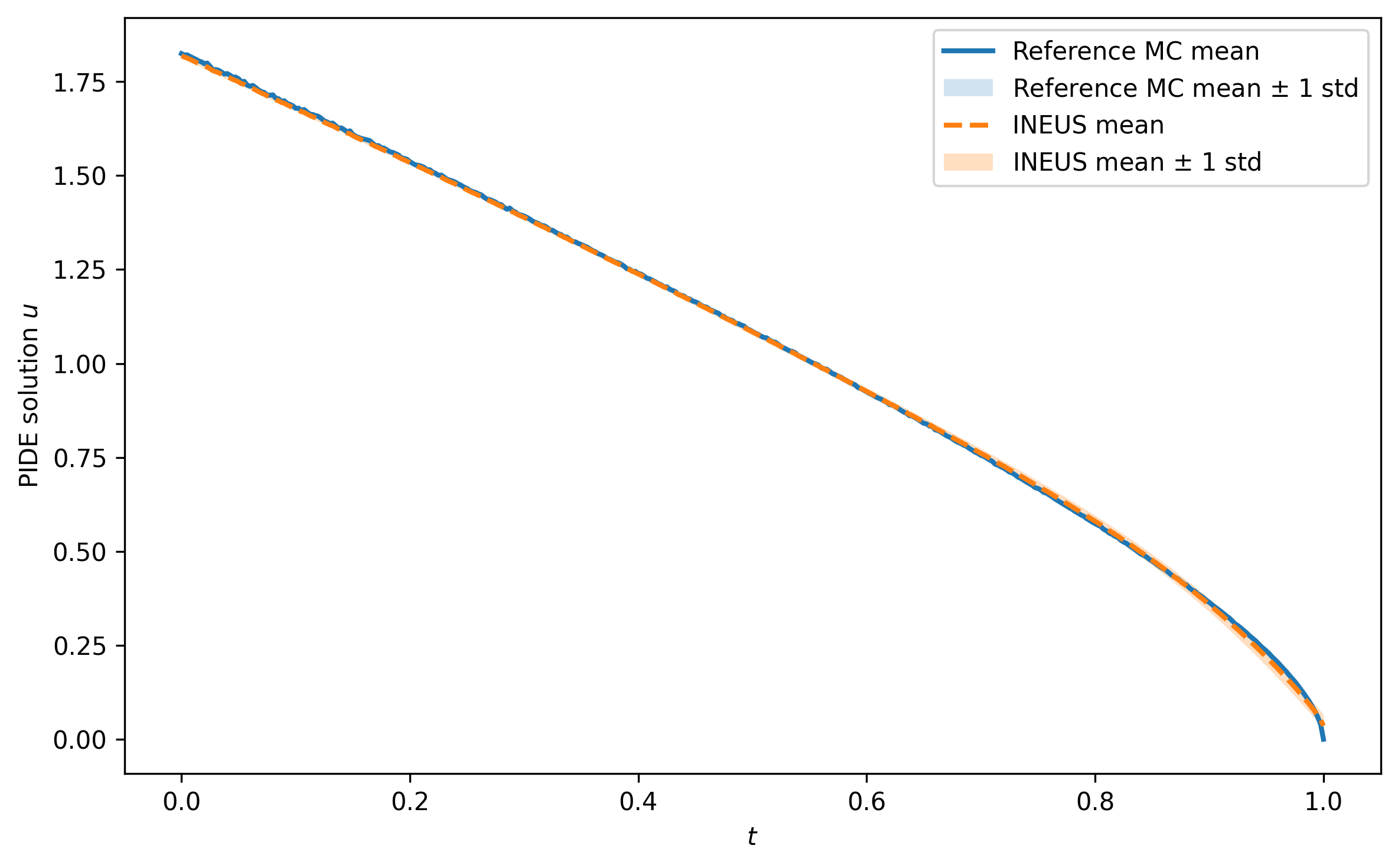}
		\label{fig:BS_K}
	\end{minipage}\vspace*{-6mm}
\caption{\small PIDE solution $u(0,x)$  for $x=(x_1,K, \dots, K)$ with $x_1  \in [0,100]$ (left) and  $u(t,x)$ for  $t \in [0,1]$ and $x= K \mathbf{1}_{20}$ (right) for a 20-dimensional linear Black--Scholes PIDE \eqref{Black_Scholes}. Orange dotted lines: numerical results of INEUS with $\pm 1$ standard deviation given by orange shaded area. Blue lines: reference MC solution obtained from the Feynman--Kac formula \eqref{Feynman} with $\pm 1$ standard deviation given by blue shaded area.}
\label{fig: BS_linear} \vspace{-5mm}
\end{figure} 
\begin{table}[h!]
\caption{\small Performance of INEUS for linear Black--Scholes PIDEs of the form \eqref{Black_Scholes} as a function of the state dimension \(d\). Results are computed after \(k_* = 1000\) training epochs and reported as means and standard deviations over 10 independent runs of Algorithm \ref{Algo1}.}
\vspace{2mm}
\label{tab:highdim_resultsbs}
	\centering
	\begin{tabular}{c|c|c|c|c|c}
		 $d$ & Mean ${\rm MAE}$ & Std.\ Dev.\  $ {\rm MAE}$ & Mean Loss $\widehat{\mathscr{L}}^{(k_*)} $ & Std.\ Dev.\ Loss $\widehat{\mathscr{L}}^{(k_*)} $ & Time (sec) \\[1mm]
		\hline \\[-3mm]
		5  & 0.0202 & 0.0031  & 0.3045 & 0.0167 & 3259 \\
		10  & 0.0224 & 0.0045  & 0.2749 & 0.0355 & 4049\\
		25  & 0.0246 & 0.0036  & 0.1933 & 0.0200 & 8248 \\
		50 & 0.0257 &0.0054 & 0.2305 & 0.0879& 10,984 \\
		100 & 0.0294 & 0.0064  & 0.2277 & 0.0950 & 15,799 \\
		\hline 
	\end{tabular}
    \end{table}

Table \ref{tab:highdim_resultsbs} again confirms the scalability of INEUS across dimensions.    As the state dimension \(d\) increases from 5 to 100 (with the Brownian dimension increasing accordingly as $q=d$ in \eqref{Black_Scholes}), the mean MAE remains small, while the runtime increase remains moderate. This is consistent with the directional-differentiation trick \ref{Proposition: efficent}, whose differential-operator cost scales with \(q\) rather than
with the \(\mathcal{O}(d^2)\) complexity of the full Hessian.

Figure \ref{fig:bsvs} then again compares INEUS performance with PINN and deep BSDE methods for the 10-dimensional linear Black--Scholes PIDE (where PINNs remain feasible). The left panel shows that INEUS reaches lower MAE values within significantly less runtime than the PINN approach, reflecting the efficiency gained by avoiding direct numerical integration of the jump term and the computation of higher-order derivatives. The right panel further shows that, when errors are compared as functions of runtime, INEUS attains the smallest MAE on the test set of randomly sampled space--time points. This suggests that INEUS provides a more accurate global approximation of the PIDE solution than the deep BSDE method with jumps, which is typically designed to approximate the solution along simulated trajectories with fixed initial point $(t_0, x_0)$. However, despite the use of Polyak averaging with $K=10$ in \eqref{eq: target}, the INEUS error curve displays larger fluctuations, indicating that its training dynamics remain less stable than those of the competing methods. Figure \ref{Fig: pde_bs} shows the corresponding results for a Black--Scholes PDE \eqref{Black_Scholes} without jumps $(\lambda=0)$. \\
\begin{figure}[h]
	\centering
	\subfigure{ \hspace{-1mm}
		\includegraphics[width=0.485\columnwidth, height=4.2cm]{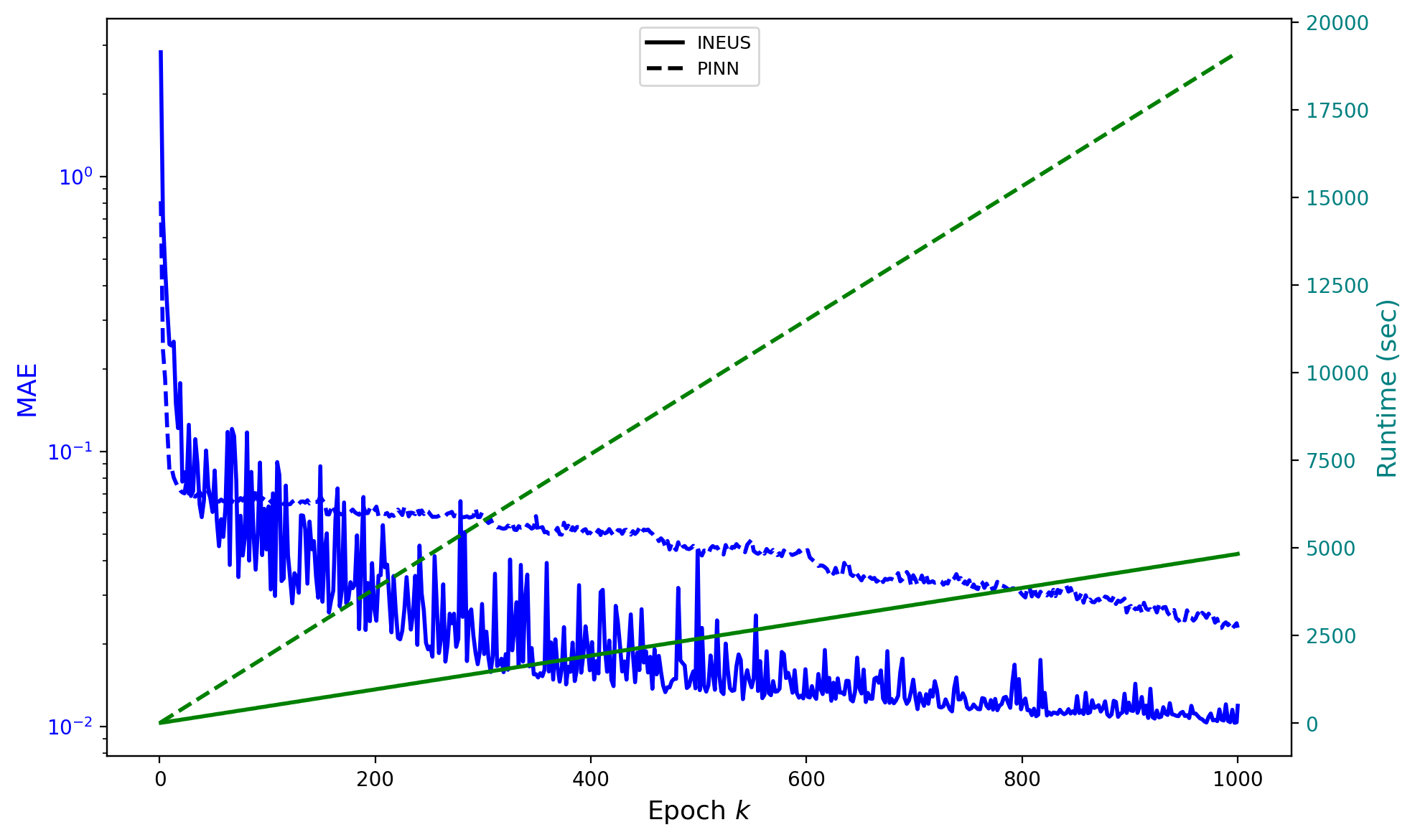}
		\label{fig:bsvspinn}
	}
	 \hspace{-3mm}
	\subfigure{
	\includegraphics[width=0.49\columnwidth, height=4.2cm]{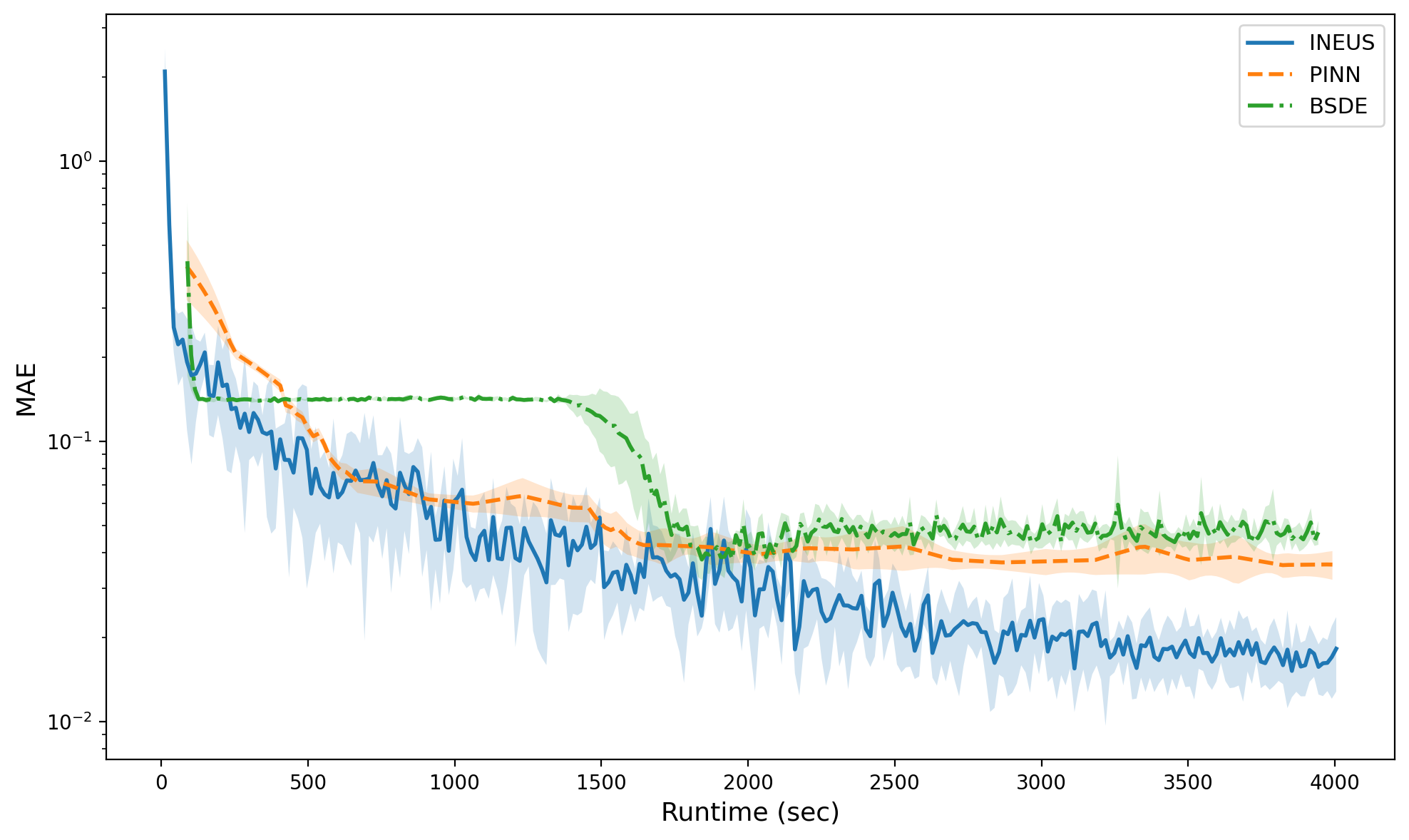}
	\label{fig:bsvsall}
	} \vspace{-1.95mm}
	\caption{{\small Left panel: comparison of ${\rm MAE}$ (blue) and runtime in seconds (green) of INEUS (solid lines) and  PINN (dashed lines) as functions of the epoch $k$. Right panel: comparison of INEUS, PINN, and deep BSDE with jumps as functions of runtime.\ Both panels correspond to a 10-dimensional linear Black--Scholes PIDE \eqref{Black_Scholes}.}} \vspace{-4mm}
	\label{fig:bsvs} 
\end{figure} 
\begin{figure}[H]
	\begin{center}		\centerline{\includegraphics[scale=0.37]{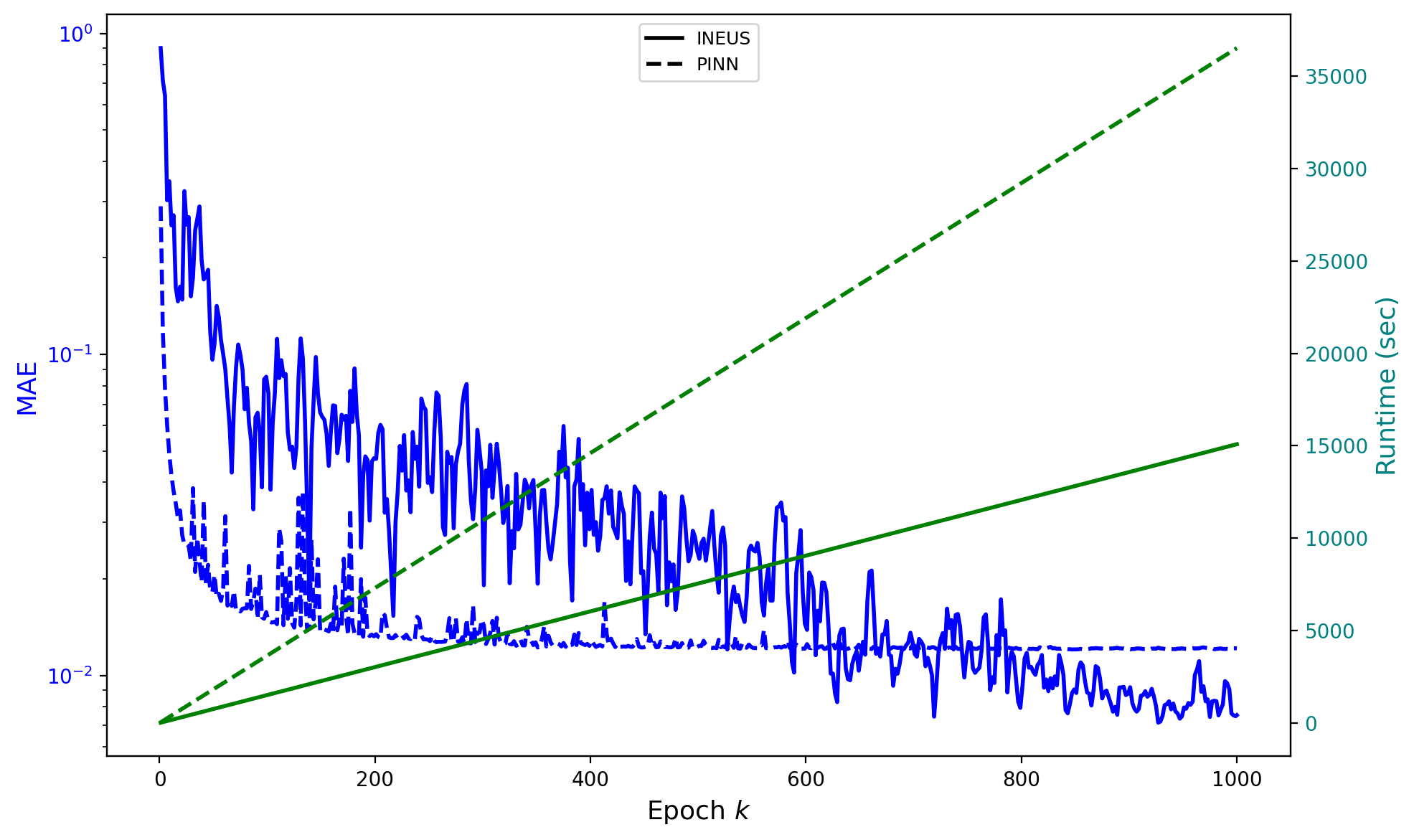}} \vspace{-1mm}
\caption{\small Comparison of $\rm{MAE}$ (blue) and runtime in seconds (green) of INEUS (solid lines) and PINNs (dashed lines) for a 100-dimensional linear Black--Scholes PDE \eqref{Black_Scholes} without jumps $(\lambda=0)$. }
		\label{Fig: pde_bs}
	\end{center}
	\vskip -0.2in
\end{figure}
\paragraph{Nonlinear Black--Scholes PIDE with default risk.}
The piecewise-linear default intensity $Q$ in the nonlinear Black--Scholes PIDE \eqref{eq:default_risk} is given by 
\begin{equation} \label{eq:intensity}
Q(u)=\mathds{1}_{\left(-\infty, v^h\right)}(u) \gamma^h+\mathds{1}_{\left[v^l, \infty\right)}(u) \gamma^l+\mathds{1}_{\left[v^h, v^l\right)}(u)\left[\frac{\left(\gamma^h-\gamma^l\right)}{\left(v^h-v^l\right)}\left(u-v^h\right)+\gamma^h\right] ,
\end{equation} with $v^h = 50, v^l = 70, \gamma^h =0.2$, $\gamma^l =0.02$. The parameters in \eqref{eq:default_risk} are set to $d=100$, $T=1$, $\beta = 0.02$, $\sigma = 0.20$, recovery rate $\delta = 2/3$, discount rate $R = 0.02$,  intensity $\lambda = 0.1$, log-jump mean $\mu_J = -0.2$, log-jump volatility $\sigma_J = 0.15$, correlation $\rho_J = 0.5$, and initial condition $x_0 = 100 \cdot \mathbf{1}_{100}$.  In the no-jump case, the reference value $u(0,x_0) = 57.300$ is obtained using the deep BSDE method of \cite{han2018solving}.  This method is designed to approximate the solution at a prescribed initial point and is thus known to give a reliable pointwise benchmark at $(0,x_0)$,  but not a global reference solution over the full space-time domain (see Introduction and Appendix \ref{app:baselines}). 
Similarly, in the jump case, the reference value $u(0,x_0) = 55.81$ is  computed using the  forward Picard scheme~\cite{bender2007forward}, 
based on the Feynman--Kac representation with $4\times10^5$ antithetic 
paths and $30$ Picard iterations\footnote{This value was also verified using the backward regression method of~\citet{gobet2005regression}.}.
In Figure~\ref{fig:min_option}, we then assess convergence to the benchmark value $u(0,x_0)$, rather than reporting a global MAE since obtaining the reference solutions for the full test set on the whole domain would be computationally very intensive.

\begin{figure}[h]
	\centering
	\subfigure{ \hspace{-1mm}
		\includegraphics[scale=0.44]{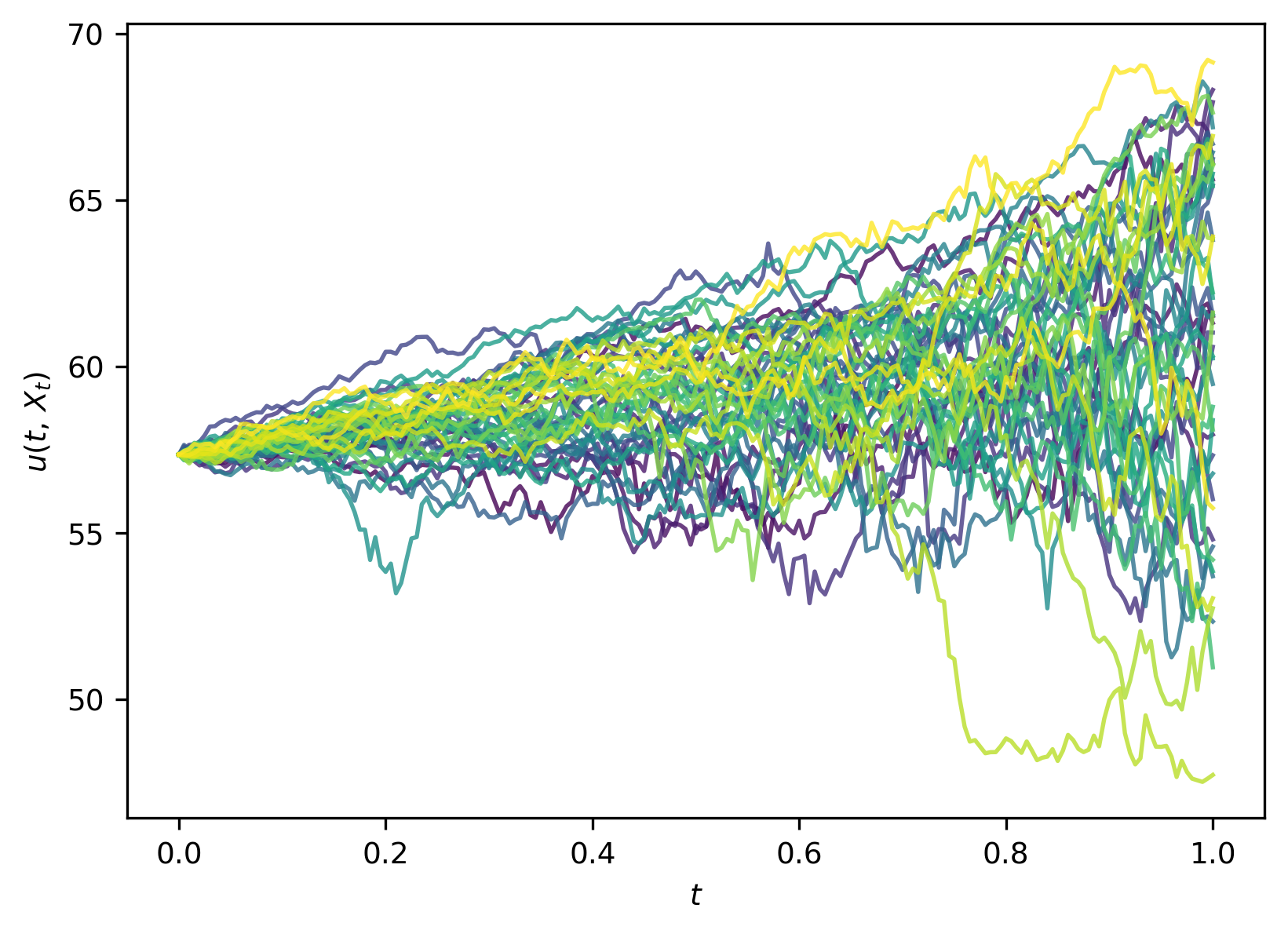}
	}
	 \hspace{-3mm}
	\subfigure{
	\includegraphics[scale=0.44]{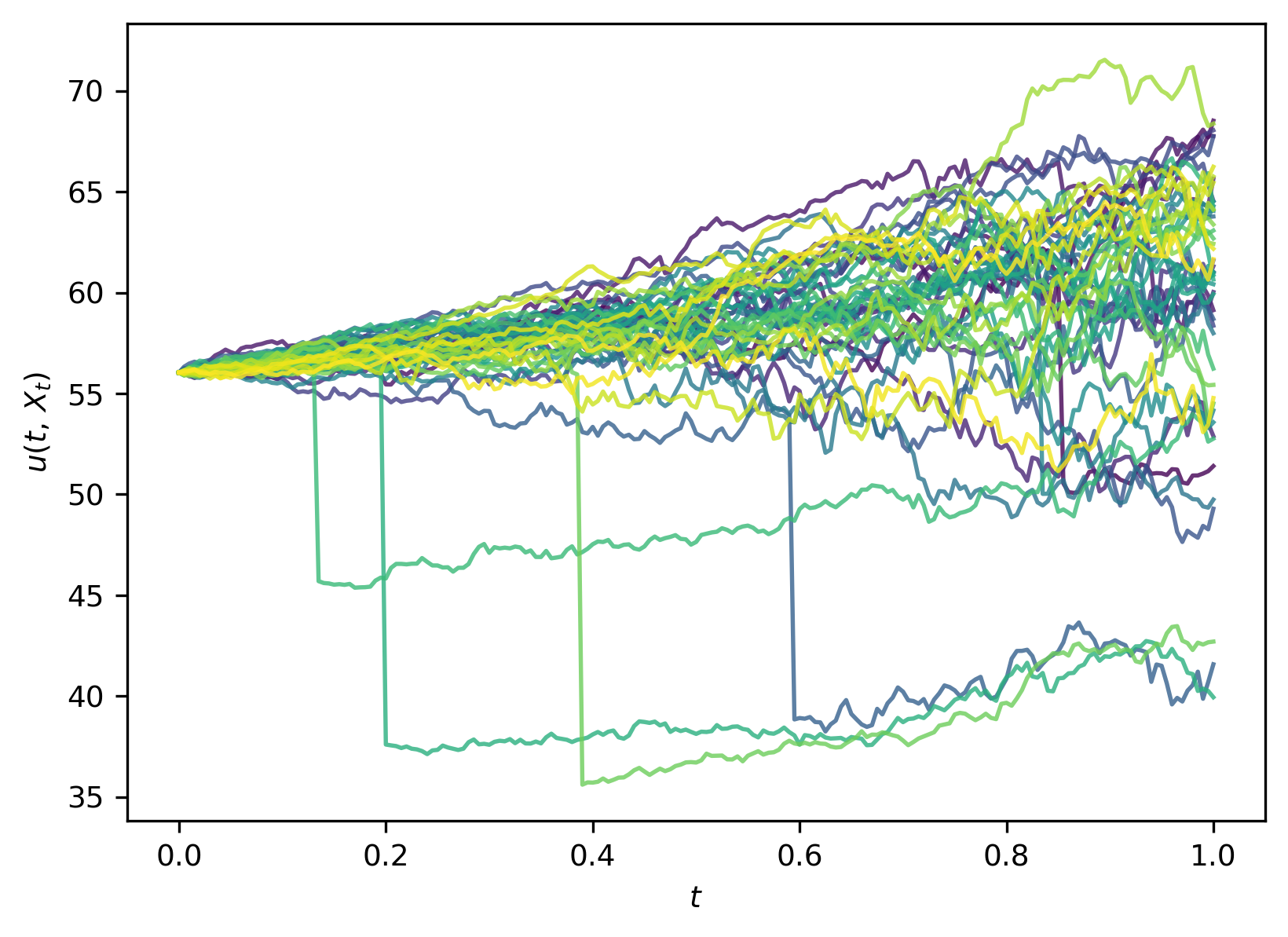}
	} \vspace{-1.95mm}
	\caption{\small Learned solution $u_\theta(t,X_t)$ of the Black--Scholes nonlinear equation \eqref{eq:default_risk} without jumps (left, $\lambda=0$) and with jumps (right, $\lambda=0.1$).    
    In both panels, the solution is evaluated along $50$ independent (jump-)diffusion sample paths starting from $x_0 = 100 \cdot \mathbf{1}_{100}$. In the no-jump setting (left), trajectories start from $u_\theta(0,x_0) \approx 57.3$ and spread as the diffusion variance grows, consistent with the expected behavior of the min-payoff pricing function under default risk. 
    In the jump setting (right), trajectories start from $u_\theta(0,x_0) \approx 55.81$ and exhibit additional discontinuities induced by the jump component ($\lambda=0.1$), leading to increased variability.} 
	\label{fig:path} 
\end{figure}

\end{document}